\documentclass[sn-mathphys-num]{sn-jnl}

\usepackage{lipsum}

\usepackage{epstopdf}

\usepackage{subfigure}
\usepackage{color}
\usepackage{colortbl}

\usepackage{graphicx}%
\usepackage{multirow}%
\usepackage{amsmath,amssymb,amsfonts}%
\usepackage{amsthm}%
\usepackage{mathrsfs}%
\usepackage[title]{appendix}%
\usepackage{xcolor}%
\usepackage{textcomp}%
\usepackage{manyfoot}%
\usepackage{booktabs}%
\usepackage{algorithm}%
\usepackage{algorithmicx}%
\usepackage{algpseudocode}%
\usepackage{listings}%

\theoremstyle{thmstyleone}%
\newtheorem{theorem}{Theorem}
\theoremstyle{thmstyletwo}%

\theoremstyle{thmstylethree}%
\begin{document}



\title[Article Title]{Re-initialization-free Level Set Method via Molecular Beam Epitaxy Equation Regularization for Image Segmentation}
\author{\fnm{Fanghui} \sur{Song}}\email{21b912024@stu.hit.edu.cn}
\author{\fnm{Jiebao} \sur{Sun}}\email{sunjiebao@hit.edu.cn}
\author*{\fnm{Shengzhu} \sur{Shi*}}\email{mathssz@hit.edu.cn}

\author{\fnm{Zhichang} \sur{Guo}}\email{mathgzc@hit.edu.cn}
\author{\fnm{Dazhi} \sur{Zhang}}\email{zhangdazhi@hit.edu.cn}

\affil{\orgdiv{ School of Mathematics}, \orgname{ Harbin Institute of Technology}, \orgaddress{\street{No. 92 Dazhi Street}, \city{Harbin}, \postcode{150001}, \country{China}}}


\abstract{
Variational level set method has become a powerful tool in image segmentation due to its ability to handle complex topological changes and maintain continuity and smoothness in the process of evolution.
However its evolution process can be unstable, which results in over flatted or over sharpened contours and segmentation failure. To improve the accuracy and stability of evolution, we propose a high-order level set variational segmentation method integrated with molecular beam epitaxy (MBE) equation regularization. This method uses the crystal growth in the MBE process to limit the evolution of the level set function. Thus can avoid the re-initialization in the evolution process and regulate the smoothness of the segmented curve and keep the segmentation results independent of the initial curve selection. It also works for noisy images with intensity inhomogeneity, which is a challenge in image segmentation.
To solve the variational model, we derive the gradient flow and design a scalar auxiliary variable (SAV) scheme, which can significantly improve the computational efficiency compared with the traditional semi-implicit and semi-explicit scheme. Numerical experiments show that the proposed method can generate smooth segmentation curves, preserve segmentation details and obtain robust segmentation results of small objects. Compared to existing level set methods, this model is state-of-the-art in both accuracy and efficiency.
}


\maketitle
\section{Introduction}

Image segmentation is an important part of computer vision where the goal is to partition a given image into several regions that usually represent objects of interest. It has attracted significant attention for its wide applications in areas such as medical imaging, remote sensing, automatic driving robotics and object detection \cite{guo2021effective,lambert2024incorporation,li2010edge,falcone2020high,bowden2021active}.

Numerous image segmentation methods have been developed including thresholding-based models \cite{otsu1979threshold}, clustering-based models \cite{shi2000normalized}, watershed-based models \cite{vincent1991watersheds}, methods with graph theory \cite{tai2021multigrid}, level set theory \cite{CVmodel} and so on \cite{yang2019image,cardelino2013contrario}.
The level set method was initially introduced by Osher and Sethian and has developed into one of the most competitive segmentation methods for its ability to handle complex topological changes \cite{luo2022convex,gao2011level,liu2022two}. It uses a level set function to describe the difference in image properties between the inside and outside of the curve and evolves an initial curve into the boundary of the object of interest. 
Compared to other segmentation methods, it has advantages such as not requiring parameterization, increased robustness to noise and enhanced flexibility \cite{MR4014923,osher1988levelset}.

One of the most representative level set methods is the geodesic active contour (GAC) model \cite{GACmodel}. The GAC model is utilizes the information of the image gradient and distance transformation. It can deal with complex gray distributions and textures. However, it is sensitive to the initial position of the curve and may result in poor segmentation for objects with weak boundaries.
Another typical model is the Chan-Vese (CV) model \cite{CVmodel}. It is based on the grayscale values and curvature. The CV model  is effective in addressing object segmentation problems with complex shapes, but it is only applicable to images with uniform regions.
To overcome the challenges posed by inhomogeneous regions, Li et al. proposed a region-scalable fitting (RSF) energy model \cite{li2008RSF}. The fidelity terms of this model are derived from local intensity information, allowing it to effectively handle segmentation in inhomogeneous regions.
However, similar to other level set models, the RSF model is sensitive to the initial curve, often resulting in local minimization and the segmentation results can be significantly affected by image noise.

The above methods heavily rely on the initial curve selection for segmentation and lack smooth control over the level set function. As a result, in recent years, several fourth-order segmentation models have been developed to overcome these limitations and address the shortcomings of the second-order models. For instance, Gao and Bertozzi proposed a new active contour model for segmentation based on the Chan-Vese model \cite{gao2011level}. It effectively captures sharp features, like object corners, often smoothed by regularization terms in conventional approaches. Yang et al. explored a model for image segmentation using the Cahn-Hilliard equation, which stands out for its ability to interpolate missing contours along wide gaps, resulting in meaningful object boundaries \cite{yang2019image}. While Falcone suggested a high-order accurate scheme for image segmentation using the level-set method \cite{falcone2020high}. The scheme improves stability and achieves a more precise approximation by modifying the curve evolution velocity. It combines a monotone scheme and a high-order scheme using a filter function for automatic adaptation based on solution regularity.

In summary, the fourth-order segmentation model has the advantages of reducing the sensitivity of initial curve selection, improving the control of level set function, enhancing contour regularization and superior performance in complex segmentation scenarios. These improvements make them a promising choice for various image segmentation applications.

To obtain more accurate and stable segmentation method, it is crucial to keep the level set function remains as a directed distance function relative to the desired evolving interface. This property is known as the signed distance function \cite{MR3001147}.
However,  during the evolution process, the level set function often fails to maintain the property of the signed distance function. This can result in the degeneration of the evolution equation or the unbounded gradients of the level set function, leading to segmentation failure \cite{osher1988levelset}.
To overcome this challenge, one possible approach is to carefully select the initial contour, such that it corresponds to the signed distance function to an initial curve encompassing the objects.  This can be achieved by using the signed distance function as the initial condition at the starting time \cite{osher2004levelset}. Additionally, re-initialization techniques can be employed to ensure the curve remains smooth and close to the signed distance function  \cite{chopp1991computing,peng1999pde,sussman1999efficient}. Unfortunately, re-initialization is complex to implement in practice and may generate undesirable side effects. These can impact the subsequent evolution process and have the risk of preventing new zero curves from emerging, which can lead to undesirable results for image segmentation, such as failures in detecting the interior boundary.

The variational level set method allows for controlling the evolution process by adjusting the fidelity and regularization terms.  This approach can maintain the signed function property without the need for explicit re-initialization steps, which is known as the re-initialization-free method.

Estellers et al. \cite{estellers2012efficient} proposed the $l^1$ optimization model, which minimizing the sum of the absolute values of the differences between the current and updated distance functions at each time step to preserve distance functions. It is computationally efficient due to the utilization of linear programming techniques.
Li et al. developed a series work on re-initialization free method by using the distance regularization term which constrain the gradient of the geometric active curves.
They first proposed DR1 regularization term \cite{li2005DR1} which limits the gradient of the curves to approach $1$. It ensures the level set function remains smooth and well-behaved during the evolution process, thereby improving the accuracy and stability of the level set method. This is especially important for handling complex and irregular shapes. But, this method relies on the selection of parameters, such as the regularization parameter and the time step. Furthermore, it is mainly designed for binary image segmentation and may not be suitable for other types of problems, such as multi-phase segmentation or shape reconstruction. To address these problems, Li et al. introduced the double-well potential distance regularization term (DR2) \cite{li2010distance}. This regularization term offers a solution that can be easily adapted to handle various types of image segmentation tasks, including binary, multi-phase, and object tracking.
In addition, Xie  et al. \cite{xie2009diffusion} proposed a regularization term to minimize steep surfaces and smooth the zero level set. This method can deal with the initialization dependency problem that commonly appears in variational level methods. but the segmentation results may not be satisfactory for images with a large amount of details.
Zhang et al. summarized the relevant methods of distance regularization and proposed a variation model derived from a reaction-diffusion model \cite{zhang2012rd}, which can smooth the evolving curve and prevent it from becoming too jagged or too irregular. Specially, the segmentation process can start from an arbitrary initial curve or surface that is placed inside or outside the object of interest.%

In conclusion, the above methods and their variants \cite{li2011level,li2022image} can avoid re-initialization and improve the accuracy of the segmentation results. However, they may still face challenges in handling complex images with a large amount of details.

Moreover, the artificial distance regularization term can keep the signed distance property but lose the advantage of controlling the evolution process by real physical principles. Can we improve the segmentation effect by constructing the distance regularization term using the equation with real physical meaning? For this purpose, we propose a high order level set variational segmentation method \cite{slope2000} by using the molecular beam epitaxial (MBE) equation to limit the evolution of the level set function.
The MBE is a technique used to deposit thin films of atoms or molecules onto a substrate to create semiconductor devices and other advanced materials.
This MBE regularization term we proposed has the following two advantages. First, the non-equilibrium term of the MBE equation forces the gradient of level set function to be  $1$, which avoids the need for re-initialization. Secondly, the surface diffusion current term, which represents the smoothness constraint of the crystal growth in the MBE process, regulates the smoothness of the segmented curve.

The major contributions of our work are threefold:

\begin{itemize}
  \item We propose a high order segmentation method based on the MBE equation. To the best of our knowledge, we are the first to apply MBE to image processing tasks.  This idea is reasonable and interesting as MBE is a typical phase-field model with complete physical meaning, which ensures that the evolution process is smooth and controllable, avoiding the problem of insufficiently smooth evolution caused by artificial distance regularization such as DR1 and DR2. It is less sensitive to the initial curve selection, provides better control over the level set function, improves contour regularization, and performs well in complex segmentation scenarios. These improvements make it a promising option for various image segmentation applications.
  \item 
     We derive the scalar auxiliary variable (SAV) scheme \cite{shen2019new,cheng2019highly,yao2020total} for the proposed model and solve it by fast Fourier transform (FFT) and inverse fast Fourier transform(IFFT). Compared with the traditional semi-implicit semi-explicit finite difference scheme \cite{li2003withoutslope}, SAV scheme is unconditional first-order energy stable, which significantly improves the computational efficiency.
  \item We design various numerical experiments to show the superiorities of the proposed model. The results show that the MBE regularization term can control the range of the gradient of the level set function and obtain more stable evolution process. The segmentation accuracy is independence of initial curve and the MBE model tends to generate smooth segmentation boundaries and have good anti-noise performance. In addition, the model also works for multi-target and small target image segmentation.
\end{itemize}

The paper is structured as follows. Section \ref{sec:motivation} presents the motivation and preliminary knowledge. In section \ref{sec:method}, we introduce the MBE regularization term and propose two related models: MBE-GAC and MBE-RSF. Section \ref{sec:numerical} outlines the SAV scheme coupled with the FFT method for our model and presents a first-order unconditional energy stability theorem. Section \ref{sec:exresults} presents numerical experiments to demonstrate the advantages of the proposed MBE model. Finally, section \ref{sec:conclusions} provides the conclusion.
\section{Motivation}
\label{sec:motivation}
Although the level set method has the advantages such as high accuracy, local adaptivities and the ability to handle complex boundaries, its evolution process may be unstable. How to improve the accuracy and stability of the evolution is one of the major challenge in the level set method, which is also the main contribution of this paper. We will demonstrate the detailed motivation of this paper from the aspect of the level set method, the distance regularization term and the property of the molecular beam epitaxial (MBE) equation.
\subsection{Level set segmentation method and re-initialization}
 Let $\Omega\subset\mathbb{R}^2$ be the image domain, $C(t)=(x(t),y(t))$ be a closed curve evolved as follows
\begin{equation}\label{eq:ev-c}
\begin{cases}
\dfrac{{\partial C}}{{\partial t}} = \alpha (x,y,t)\mathbf{T} + \beta (x,y,t)\mathbf{N}, & (x,y,t)\in \Omega \times (0,T],
\\
C(0)=C_{0}, & (x,y) \in \Omega,
\end{cases}
\end{equation}
where $\alpha$ is tangential velocity, $\beta$ is normal velocity, $\mathbf{T}$ and $\mathbf{N}$ represent the unit tangent vector and unit normal vector, respectively, and $C_0$ is the initial value.

On a local scale, by representing $y=\gamma(x(t),t)$, one has $C(t)=(x(t),\gamma(x,t))$ and $ C_x=(1,\gamma_x)$,
\begin{equation*}
\mathbf{T}=\frac{(1,\gamma_x)}{\sqrt {1 + {\gamma _x}^2}},\quad
\mathbf{N}=\frac{(-\gamma_x,1)}{\sqrt {1 + {\gamma _x}^2}}.
\end{equation*}
When $C$ evolves according to \eqref{eq:ev-c}, $x$ and $y$ move according to the following equations
\begin{align*}
& \frac{{dy}}{{dt}} = \alpha \frac{{{\gamma _x}}}{{\sqrt {1 + {\gamma _x}^2} }} + \beta \frac{1}{{\sqrt {1 + {\gamma _x}^2} }}
\\
&\frac{{dx}}{{dt}} = \alpha \frac{1}{{\sqrt {1 + {\gamma _x}^2} }} + \beta \frac{{ - {\gamma _x}}}{{\sqrt {1 + {\gamma _x}^2} }}.
\end{align*}
Considering $$\frac{{dy}}{{dt}} = {\gamma _x}\frac{{dx}}{{dt}} + {\gamma _t},$$
then
\begin{equation*}
{\gamma _t}= \frac{dy}{dt} - {\gamma _x}\frac{dx}{dt}=\beta \sqrt {1 + {\gamma _x}^2}.
\end{equation*}
So the evolution of \eqref{eq:ev-c} can be simplified as
\begin{equation}\label{ev-c-final}
\frac{\partial{C}}{\partial{t}}=\beta{\mathbf{N}},
\end{equation}
which shows the change in the geometric shape of a curve only depends on its normal component \cite{aubert2006mathematical}.
Solving \eqref{ev-c-final} directly is difficult. To address this issue, Osher and Sethian proposed the level set method \cite{osher1988levelset}. It represent $C$ as a zero level set of a higher dimensional function $\phi:\Omega\rightarrow \mathbb{R}$, which is
\begin{equation}
\label{230711-1}
C(t) = \{(x,y)|\phi(x,y,t)=0\}.
\end{equation}
In \eqref{230711-1}, $\phi(x,y,t)$ is determined by solving the nonlinear equation,
\begin{equation}\label{eq:level set}
\begin{cases}
\dfrac{\partial \phi }{\partial t} =F\left | \nabla\phi \right |,& (x,y,t)\in \Omega \times (0,T],
\\
\phi(x,y,0)=\phi_{0}(x,y), & (x,y) \in \Omega,
\end{cases}
\end{equation}
where $\nabla \phi$ denotes the spatial gradient of $\phi$, $F$ is related to normal component of the velocity, which does not change sign during the evolution and the orientation depends on the type of evolution (outward for an expansion and inward for a shrinking), $\phi_0$ satisfies
 \begin{equation*}
\left\{\begin{array}{ll}
\phi_{0}(x, y)<0, & (x, y) \in \Omega_{0}, \\
\phi_{0}(x, y)=0, & (x, y) \in C_{0}, \\
\phi_{0}(x, y)>0, & (x, y) \in \mathbb{R}^{2} \backslash \bar{\Omega}_{0},
\end{array}\right.
\end{equation*}
and $\Omega_0$ is the region delimited by $C_0$.

Typically, the evolution governed by equation \eqref{eq:level set} can be unstable. The gradient of $\phi$ can become unbounded or zero, which prevents $\phi$ from being preserved as a signed distance function and leads segmentation failure. To address this issue, one popular method is re-initialization \cite{chopp1991computing}, which consists the following two parts.
\begin{itemize}
\item[(i)]{Choose a good initial value:
A good choice of $\phi_0$ is the signed distance function to an initial given curve $C_0$ surrounding the objects, which is given by
\begin{equation*}
\left\{\begin{array}{ll}
\phi_{0}(x,y)=+d((x,y),C_0), &  (x, y) \in \Omega_{0}, \\
\phi_{0}(x,y)=-d((x,y),C_0), & (x, y) \in \mathbb{R}^{2} \backslash \bar{\Omega}_{0},
\end{array}\right.
\end{equation*}
where $d((x,y),C_0)$ is the Euclidean distance between the point $(x,y)$ and the curve $C_{0}$. }
\item[(ii)]{Re-initialization: When the gradient of $\phi$ tends to become irregular, one can restart the evolution with a new initial value $\phi_0$, which is the steady solution of the Eikonal equation
\begin{equation}\label{eq:re-init-pde}
\begin{cases}
	 \dfrac{\partial \phi}{\partial t} = \mbox{sign}(\phi(x,y,\tilde{t}_{0}))(1 - |\nabla \phi|),& (x,y,t)\in \Omega \times (0,T],
\\  \phi(x,y, 0)=\phi(x,y, \tilde{t}_{0}),& (x,y)\in \Omega,
 \end{cases}
\end{equation}
where $\tilde{t_{0}}$ is the moment when re-initialization needs to be applied, $\phi(x,y,\tilde{t}_{0})$ is the currently updated level set function. The steady solution of \eqref{eq:re-init-pde} satisfies $|\nabla \phi| = 1$, which can ensure $\phi$ satisfies the property of the signed distance function over a period of time.}
\end{itemize}

However, in real implementation, the re-initialization process faces some difficulties: (i) The re-initialization process needs to be done on a regular basis. The rule of when and how to use it is unknown. (ii) The equation \eqref{eq:re-init-pde} is a nonlinear hyperbolic equation that presents challenges in dealing with discontinuity and nondifferentiability. When the equation is transformed into a time-dependent problem, the CFL condition \cite{kamga2006cfl} for finite propagation velocity and time stability may require many time steps for the solution to converge to the entire domain. When considered as a stable boundary value problem, the existence of discontinuity can seriously affect the design of upwind difference schemes and some fast algorithms. Although the emergence of fast sweeping method \cite{zhao2005fast} provides a good solution for this kind of problem, it still faces development difficulties in practical applications.

\subsection{ Distance regularization term for re-initialization-free} It can be observed that the intention of re-initialization method is to control $|\nabla \phi|$ within a certain range during the evolution process and thus guarantee the stability of the evolution. Another natural idea is to constrain $|\nabla \phi|$ by the regularization term under the framework of the variational level set method, which formulates the segmentation problem as an energy minimization problem and incorporates prior knowledge \cite{wali2023level} or constraints to improve the accuracy and stability \cite{li2008RSF,chen2002using,falcone2020high}.

To illustrate our motivation, we briefly recall some representative distance regularization methods. For simplicity, we assume that $x \in \Omega \subset \mathbb{R}^2$ in the following.
In \cite{li2005DR1}, Li et al. proposed DR1 regularization term
\begin{equation}\label{eq:dr1-f}
	\mathcal{R}_1(\phi) = \int_\Omega r_1(|\nabla\phi|)\mbox{d}x,
\end{equation}
where $r_1(s) =  \dfrac{1}{2}(s-1)^2$. The gradient flow of \eqref{eq:dr1-f} is
\begin{equation}\label{eq:drlse1-e}
	\frac{\partial \phi}{\partial t} = \mbox{div}(d_1(|\nabla\phi|)\nabla\phi),
\end{equation}
where $d_1(s)= 1-1/s.$ The evolution equation (2.7) has an undesirable side effect. When $|\nabla \phi| \rightarrow 0$, the diffusion rate $(1-\frac{1}{|\nabla\phi|}) \rightarrow -\infty$ (Fig.\ref{fig:drlse-plot}). This causes strong backward diffusion, resulting in oscillations of $\phi$. To avoid this, Li et al. proposed the double-well regularization term (DR2)
\begin{equation}\label{eq:drlse2-f}
	\mathcal{R}_2(\phi) = \int_\Omega r_2(|\nabla\phi|)\mbox{d}x,
\end{equation}
where
\begin{equation*}
r_2(s) = \left\{\begin{array}{ll}
	(1 - \mbox{cos}(2\pi s))/{(2\pi)^2}, & \mbox{if} \hspace{0.1cm} s\le1,\\
	(s-1)^2/2, & \mbox{if} \hspace{0.1cm}  s>1.
\end{array}\right.
\end{equation*}
The corresponding gradient flow is
\begin{equation}\label{eq:drlse2-e}
	\frac{\partial\phi}{\partial t} = \mbox{div}(d_2(|\nabla\phi|)\nabla\phi),
\end{equation}
where
\begin{equation*}
d_2(s) = \left\{\begin{array}{ll}
	\sin(2\pi s)/(2\pi s), & \mbox{if} \hspace{0.1cm} s\le1,\\
	1-1/s, & \mbox{if} \hspace{0.1cm}  s>1.
\end{array}\right.
\end{equation*}
Different from the DR1, the DR2 constrains $|\nabla \phi|$ to $0$ or $1$, the coefficient of \eqref{eq:drlse2-e} approaches $1$ when $|\nabla\phi|$ approaches $0$ (Fig. \ref{fig:drlse-plot}), which modify the oscillation of the DR1 when $|\nabla \phi|$ is very small. At the same time, DR2 could cause the desired sign distance function around zero level set because of the design of function $r_2$ at $|\nabla\phi|>1/2$.
\begin{figure}
	\centering
	\includegraphics[scale=0.6]{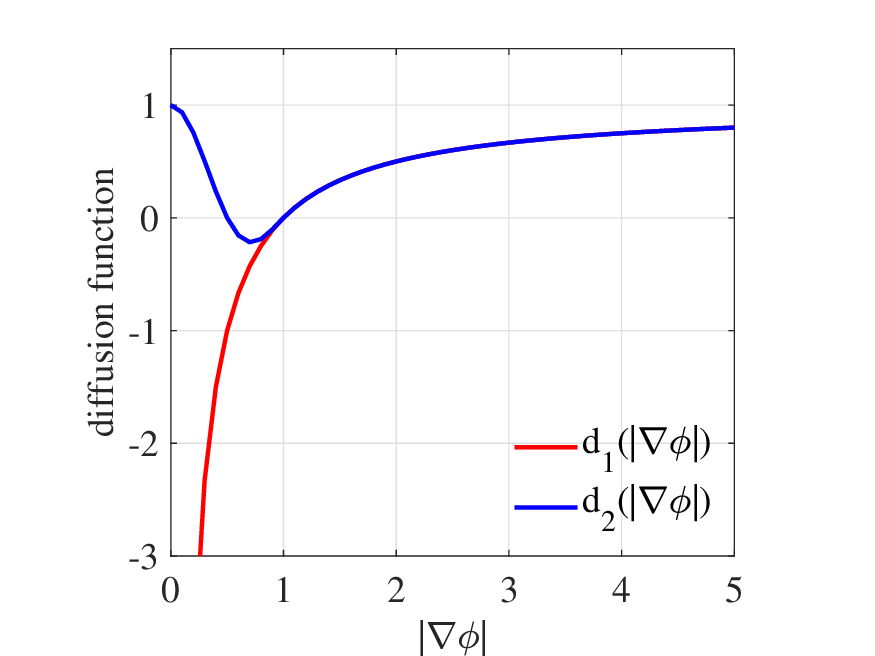}
	\caption{The plots of the diffusion coefficient of the DR1, DR2 of $|\nabla\phi|$.}
    \label{fig:drlse-plot}
\end{figure}

Despite the success of the re-initialization-free DR method, there are still some issues that need to be addressed. According to the experimental results, the gradient amplitude of the level set function $\phi$ does not approach a value close enough to $1$. Additionally, both regularization term rely on first-order derivatives. The arc length term constrains $|\nabla \phi|$ to approach $0$ near the zero level set region, while the distance regularization term enforces $|\nabla \phi|$ to approach $1$. Balancing these two terms can be a challenging task.
To address these issues, we propose higher order regularization term based on MBE equation. This approach can provide a theoretical and physical foundation for the development of the model.
\subsection{Molecular Beam Epitaxy regularization term}
The molecular beam epitaxy (MBE) technique is renowned for its precision in growing thin solid films, enabling the formation of monolayer-thin interfaces and atomically flat surfaces \cite{slope2000}.
The MBE equation is a fourth-order partial differential equation that ensures the level set function's gradient magnitude approaches $1$ while maintaining the desired signed distance function around the zero level set.

Consider the MBE equation with slop selection, which represents the epitaxial growth of thin films, which can be derived from the following functional
\begin{equation}\label{eq:MBE-functional}
	\mathcal{E}_{MBE}(\phi) = \int_{\Omega} \frac{\alpha}{2}|\Delta\phi|^2  + \frac{1}{4} (|\nabla \phi|^2 - 1)^2 \mbox{d}x,
\end{equation}
where $\phi=\phi(x,t)$ is a scaled height function of a thin film in a co-moving frame.

The MBE model usually has the physical characteristics of unconditional energy stability and unique solution, and the slope can be selected according to the actual research needs. $\mathcal{E}_{MBE}(\phi)$ as the effect free energy, is an example of energy functional of $\phi$ in the strain gradient theory for structural phase transitions in solids. The second term of $\mathcal{E}_{MBE}(\phi)$ clearly shows the non-equilibrium term of \eqref{eq:MBE-functional} forces $|\nabla \phi|$ approach $1$.

We compare the MBE regularization term with different regularization terms under the same conditions. Experiments show that the range of the gradient of level set function of MBE regularization tends to be small than that of other regularization terms. See for instance, the MBE regularization term is steadily approaching $1$ within a range from $0$ to $5$, while the range of DR1 becomes from $0$ to $16$, see Fig. \ref{fig:text-alpha}.

\begin{figure}[htbp]
	\centerline{
		\subfigcapskip=-3pt
		\subfigure[]{\includegraphics[height=0.25\textwidth, width=0.35\textwidth]{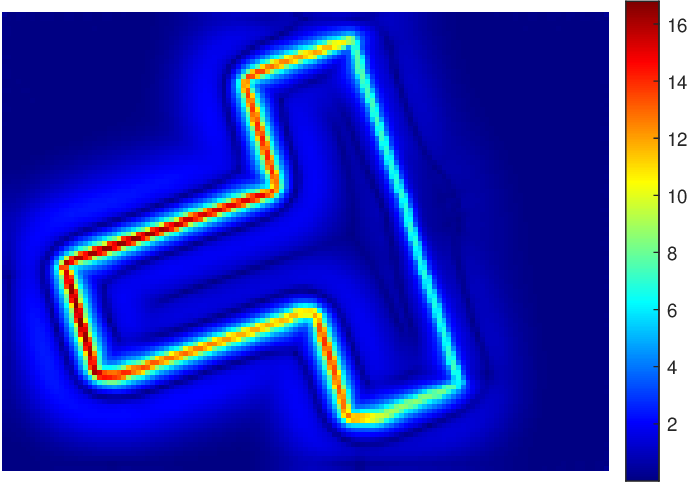}}\quad
		\subfigure[]{\includegraphics[trim={0cm 0cm 0cm 0cm},clip, width=0.35\textwidth]{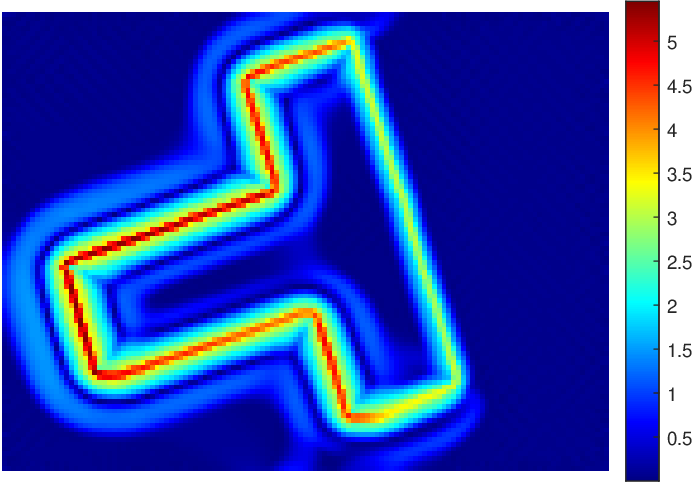}}
	}
	\caption{The $|\nabla \phi|$ of segmentation result of different regularization term: (a)DR1 (b)MBE}\label{fig:text-alpha}
\end{figure}

Furthermore, the MBE equation possesses properties that correspond to various image features.  For instance, the biharmonic term in the equation contributes to smoothness, reducing noise and irregularities on the thin film surface.  This leads to smoother images that are more resistant to noise.  Additionally, the non-equilibrium diffusion term treats the thin film surface as a distance function, which preserves its shape and position, similar to preserving the shape and position of image edges.  This makes it useful for segmenting objects with complex shapes while retaining image information.  Finally, the parameters of the equation can be adjusted to achieve the best segmentation effect that meets practical needs.

In summary, the physical principle of the MBE equation is aligned with the objective of re-initializing the level set function.  The embedding function (i.e., the film growth height function) is neither too flat nor too steep.  Therefore, we apply the MBE equation to the level set method as the regularization term of the level set segmentation model to avoid the need for re-initialization and regulate the smooth of the segmented curve.
\section{Main method}
\label{sec:method}
In this section, we present a general framework for variational segmentation level set methods integrated with the MBE distance regularization term.
Consider the general form of variational level set segmentation model
\begin{equation}\label{eq:seg-model}
	\mathcal{E}(\phi) = \mathcal{S}(\phi,I) + \mu\mathcal{R}(\phi) + \nu \mathcal{L}(\phi),
\end{equation}
where $I:\Omega\rightarrow \mathbb{R}$ is a given vector valued image, $\mathcal{S}(\phi,I)$ is called segmentation term or data fidelity term, $\mathcal{R}(\phi)$ is the regularization term, $\mathcal{L}(\phi)$ is the arc length term and $\nu$, $\mu$ are the parameters to adjust the length and smoothness of the curve.

Next, we propose the novel MBE regularization term and present two relevant segmentation models. The regularization term we propose in \eqref{eq:seg-model} is designed to avoid re-initialization of the segmentation procedure and keep the evolution process accurate and stable. To demonstrate this, we apply the MBE regularization term to the geodesic active contour (GAC) model \cite{GACmodel} and the region-scalable fitting (RSF) model \cite{li2008RSF} and illustrate the superiority of using the MBE regularization term.

\subsection{The MBE regularization term}
Denote the MBE functional as
\begin{equation}
\label{230711-2}
\mathcal{R}_{MBE} =  \int_{\Omega} \frac{\alpha}{2}|\Delta\phi|^2  + \frac{1}{4} (|\nabla \phi|^2 - 1)^2 \mbox{d}x.
\end{equation}
The $L^2$ gradient flow of \eqref{230711-2} is
\begin{equation}\label{eq:mbe-eq}
\begin{cases}
\dfrac{\partial \phi}{\partial t} = -\alpha \Delta^2\phi + \mbox{div}((|\nabla \phi|^2 - 1)\nabla \phi), & \quad  (x,t)\in\Omega\times(0,T],\\
\phi(x,0)=\phi_0(x), & \quad  x\in \Omega,
\end{cases}
\end{equation}
supplemented with suitable boundary conditions,
which is the MBE equation with slop selection.

The first term in the right side of equation \eqref{eq:mbe-eq} is the equilibrium term, which is the surface diffusion current that describes surface diffusion driven by surface tension. This term represents the equilibrium situation, $\alpha$ is the surface diffusion constant. The second term (a nonlinear second-order term) in the right side of equation \eqref{eq:mbe-eq} is  non-equilibrium diffusion current that depends on the local slope $|\nabla \phi|$.  This term is caused by Ehrlich-Schwoebel instability, which selects the slope of film surface. The diffusion term of MBE equation is nonlinear and can describe the nonlinear behaviors of substance diffusion, aggregation, and recrystallization on thin film surfaces. For small slopes, this current is positive, making initially flat interfaces unstable. However, the current vanishes where the slope is $1$, which is the preferred slope \cite{slope2000}. Actually, the physical principles of the MBE equation are consistent with the constraint that the level set function is a distance function. Essentially, both require level set function (thin film growth height function) $\phi$ is neither too flat nor too sharp, meaning $|\nabla \phi|$ should approach $1$ as much as possible.

To illustrate the necessity of fourth-order terms in the MBE regularization term, we analyze the following two cases.
On the one hand, when $\alpha=0$, the diffusion rate is proportional to $(|\nabla \phi|^2 - 1)$. If $|\nabla \phi|>1$, the diffusion rate is positive and the effect is similar to usual diffusion, which makes the level set function more even and reduces $|\nabla \phi|$. However, if $|\nabla \phi|<1$, the diffusion term has a reverse effect and may increase the gradient. This can result in an ill-posed equation, where $|\nabla \phi|$ cannot be well-controlled within a stable range. On the other hand, when $\alpha\neq0$, the term $-\alpha \Delta^2\phi$ acts as a type of viscous term that provides a well-controlled mechanism for diffusion. Specifically, the fourth-order regularization term can penalize changes in second-order curvature, resulting in a segmentation result with continuity and smoothness. This regularization term can reduce noise and discontinuity in the segmentation results, while preserving more details.

Therefore, the MBE equation imposes a stronger smoothness requirement than isotropic and anisotropic diffusion models, thus eliminating the staircase effect. It can also be combined with a $L^2$ fidelity term to generate piecewise linear solutions \cite{gao2011level}.
In addition, the MBE equation predicts that mound-like or pyramid structures in the surface profile will tend to have a uniform, constant mound slope \cite{slope2000,li2003withoutslope,ortiz1999continuum}, which ensuring the smooth of the evolution.

Overall, by incorporating this higher order regularization framework, the need for re-initialization procedures can be eliminated and the evolution will be more accurate and stable.
Meanwhile, the MBE equation satisfies the energy dissipation law, which facilitates the acquisition of numerical schemes. The unconditional stability of the energy dissipation maintenance scheme allows for flexibility in model processing.
\subsection{Two applications of MBE regularization method}
\subsubsection{The MBE-GAC model}
The GAC model \cite{GACmodel} is a popular image segmentation method that incorporates a geodesic distance term, which is a commonly used data fidelity term.

Let  $\delta(\cdot)$ be the Dirac function, $H(\cdot)$ be the Heaviside function \cite{weisstein2002heaviside},
\begin{equation}
\label{230713-1}
	H(\phi) = \left\{\begin{array}{ll}
		1, & \mbox{if} \hspace{0.1cm} \phi\ge0,\\
		0, & \mbox{if} \hspace{0.1cm}  \phi<0,
	\end{array}\right.
\end{equation}
and $g = \dfrac{1}{{1 + {{\left| {\nabla {G_\sigma }*I} \right|}^2}}}$ is the edge detection function, where ${G_\sigma }$ is Gaussian kernel.
Rewriting GAC model as
$$\mathcal{S}_{GAC}(\phi)=\lambda\int_{\Omega}g \delta(\phi)|\nabla \phi|\mbox{d}x+\gamma \int_{\Omega}g H(-\phi)\mbox{d}x,$$
we propose the MBE-GAC model
\begin{equation*}
\mathcal{E}_{MBE-GAC}(\phi) = \mathcal{S}_{GAC}(\phi) + \mu\mathcal{R}_{MBE}(\phi).
\end{equation*}
Denoting $F_{GAC}(\phi) = \lambda \delta(\phi) \operatorname{div}\left(g \frac{\nabla \phi}{|\nabla \phi|}\right)+\gamma g \delta(\phi)$, one derives the corresponding $L^2$ gradient flow as
\begin{equation} \label{230714-2}
\begin{cases}
\dfrac{\partial \phi}{\partial t} = \mu(-\alpha \Delta^2\phi + \mbox{div}((|\nabla \phi|^2 - 1)\nabla \phi)) + F_{GAC}(\phi), & \   (x,t)\in\Omega\times(0,T],\\
\phi(x,0)=\phi_0(x), & \quad  x\in \Omega,
\end{cases}
\end{equation}
supplemented with suitable boundary conditions. The energy functional corresponding to the GAC model can actually be regarded as the arc length with weights, so the MBE-GAC model contains only data fidelity term and regularization term.

The original GAC model is independent of the curve's parameters but is highly sensitive to the selection of the initial position. It often fails when applied to the segmentation of objects with non-convex boundaries. Additionally, the original GAC model is a second-order model that imposes constraints only on the first derivative of the level set function. By incorporating the MBE regularization term, the new model applies higher-order constraints to the level set function, allowing it to better capture and maintain complex shapes and boundaries.

\subsubsection{The MBE-RSF model}
The RSF model has achieved good performance in handling the image of heterogeneous regions. It is a two-phase segmentation method that uses two fitting functions, denoted by $f_1$ and $f_2$, to approximate the image intensity on both sides of the object boundary.  The data fidelity term is then formulated based on the difference between the image intensity and the fitting functions inside the boundary of object, forcing the evolution towards the object boundary.

Let $C$ is the closed curve separating the region $\Omega$, $\Omega_1 = \mbox{inside}(C)$, $\Omega_2 = \mbox{outside}(C)$, the locally fitting function can be represented as
\begin{equation*}
	\mathcal{S}_x(C, f_1, f_2) = \sum_{i=1}^{2} \lambda_i \int_{\Omega_i}K_\sigma({x}-{y})|I(y)-f_i(x)|^2\mbox{d}x,
\end{equation*}
where $\lambda_1$, $\lambda_2$ are positive constants and $K_\sigma(u) = \dfrac{1}{(2\pi)^{n/2}\sigma^2}\exp(-|u|^2/2\sigma^2)$ is the kernel function.
Integrating $\mathcal{S}_{x}(C, f_1, f_2)$ over the image domain $\Omega$ and by \cite{li2008RSF}, based on level set method, one derives RSF fidelity term
\begin{equation*}\label{eq:rsf-variational}
	\mathcal{S}_{RSF}(\phi, f_1, f_2) =
\sum_{i=1}^{2}\lambda_i\int_{\Omega}\int_{\Omega}K_\sigma(x-y)|I(y)-f_i(x)|^2M_i(\phi(y))\mbox{d}y\mbox{d}x,
\end{equation*}
where $M_1(\phi) = H(\phi)$ and $M_2(\phi) = 1 - H(\phi)$.

Different from MBE-GAC model, when using the RSF model, it is necessary to add the arc length term to eliminate the effect of noise.
Usually, $\mathcal{L}(\phi)$ measures the arc length of the zero level curve which is defined by
\begin{equation}\label{eq:arc-length}
	\mathcal{L}(\phi) = \int_{\Omega} |\nabla H(\phi)|\mbox{d}x = \int_{\Omega} \delta(\phi)|\nabla \phi| \mbox{d}x.
\end{equation}
The arc length term penalizes the curvature of the curve, resulting in a smoother and more natural curve shape while suppressing local details and noise.

Overall, we derive the MBE-RSF model as
\begin{equation}\label{eq:rsf-mbe}
\mathcal{E}_{MBE-RSF}(\phi, f_1, f_2) = \mathcal{S}_{RSF} + \mu\mathcal{R}_{MBE} + \nu \mathcal{L}.
\end{equation}
Consequently, set
$$F_{RSF}(\phi) =  -\delta(\phi) (\lambda_1e_1 - \lambda_2e_2)  + \nu \delta(\phi)\mbox{div}(\dfrac{\nabla\phi}{|\nabla\phi|}).$$
The corresponding $L^2$ gradient flow of \eqref{eq:rsf-mbe} is
\begin{equation} \label{eq:rsf-mbe-el}
\begin{cases}
\dfrac{\partial\phi}{\partial t} = \mu(-\alpha \Delta^2\phi + \mbox{div}((|\nabla \phi|^2 - 1)\nabla \phi)) + F_{RSF}(\phi), & \   (x,t)\in\Omega\times(0,T],\\
\phi(x,0)=\phi_0(x), & \quad  x\in \Omega,
\end{cases}
\end{equation}
supplemented with suitable boundary conditions, where
$$e_i = \int_{\Omega}K_\sigma(y-x)|I(x)-f_i(y)|^2\mbox{d}{y}, \quad i=1,2,$$
and
$$f_i(x) = \dfrac{K_\sigma(x)*[M_i(\phi(x))I(x)]}{K_\sigma*M_i(\phi(x))}, \quad i=1,2.$$

The original RSF model uses the DR1 regularization term, which often fails to segment images with noise accurately. Replacing the DR1 regularization term with the MBE regularization term makes the model more robust to noise, resulting in more accurate segmentation results.

\section{Numerical algorithms}\label{sec:numerical} The numerical implementation is carried out by solving the gradient descent flow of the model \eqref{eq:seg-model}. In general, for high-order nonlinear problems, the convergence condition of explicit scheme is harsh, resulting in low computational efficiency. The implicit scheme is stable, but the solution process is challenging as it involves solving a nonlinear steady equation at each time step. Usually, semi-implicit and semi-explicit finite difference scheme is widely used in such problems, see for instance \cite{li2022semi,ethier2008semi}.

Let $h=1$ be the space grid size, $\tau$ be the times step, and denote
$$x_{i} = ih, i=1,2,\cdots, M,  \quad y_{j}=jh, j=1,2,\cdots,N, \quad
t_{n}=n\tau, n=0,1,\cdots $$
where $M\times N$ is the size of the image support. For simplicity, we denote the segmentation term and arc length term in model \eqref{230714-2} and \eqref{eq:rsf-mbe-el} as $\mbox{T}_{seg}(\phi)$. We first generalize the semi-implicit and semi-explicit discretization given in \cite{li2003withoutslope} to the proposed MBE-based model as
\begin{equation}\label{eq:temporal-scheme}
\begin{aligned}
	\frac{\phi_{i,j}^{n+1} - \phi_{i,j}^n}{\tau} &+ \mu\left(\frac{3}{4}\alpha\Delta^2\phi_{i,j}^{n+1} + \Delta\phi_{i,j}^{n+1}\right)
	\\
=& \mu \left(-\frac{\alpha}{4}\Delta^2\phi_{i,j}^n + \mbox{div}(|\nabla\phi_{i,j}^n|^2\nabla\phi_{i,j}^n)\right) + \mbox{T}_{seg}(\phi_{i,j}^n),
\end{aligned}
\end{equation}
where $\phi_{i,j}^n = \phi(x_i, y_j, \tau_{n})$ denotes the discretization of $\phi$.
Experiments have revealed the accuracy of the scheme \eqref{eq:temporal-scheme}.
However, due to the high order of the linear term and the complexity of the nonlinear terms, the time step $\tau$ should be small enough to ensure the stability, which limits the applications of scheme \eqref{eq:temporal-scheme} in real implementation by its low efficiency.

The scalar auxiliary variable (SAV) method \cite{shen2019new} is unconditionally energy stable for solving gradient flows with no restrictions on the specific form of the nonlinear terms. By introducing an auxiliary variable that do not depend on spatial variables, the SAV scheme for the proposed MBE segmentation model only involves solving two linear equations with constant coefficients and the accuracy does not depend on the choice of the time step $\tau$. Besides, the resulting linear equations of our model can be easily solved by the fast Fourier transform (FFT) algorithm \cite{1984FFT}. Thus, in this section, we derive the scalar auxiliary variable (SAV) for the proposed segmentation model. Compared with the traditional semi-implicit and semi-explicit scheme \eqref{eq:temporal-scheme}, the SAV scheme for the MBE segmentation model is easier to implement and can significantly improve the computational efficiency.

\subsection{Scalar auxiliary variable (SAV) scheme}
As we know, the SAV is a new technique to construct time discretization schemes for a class of gradient flows driven by a bounded free energy functional $\mathcal{E}(\phi)$ from below. Obviously, in both MBE-GAC model and MBE-RSF model, there exists a positive constants ${C_{0}}$ such that $\mathcal{E}_{1}(\phi)\geq {C_{0}} >0$.
To illustrate the idea, we rewrite the variational level set model \eqref{eq:seg-model} in the  general form,
\begin{equation*}
\mathcal{E}(\phi)=\frac{1}{2}(\phi, \mathcal{L} \phi)+\mathcal{E}_{1}(\phi),
\end{equation*}
where $\mathcal {L}=\mu \alpha  \Delta^{2}$ is a symmetric non-negative linear operator independent of $\phi$, and $\mathcal{E}_{1}(\phi)$ contains other lower order operators in \eqref{eq:seg-model}.

Introduce a scalar auxiliary variable $r= \sqrt{\mathcal{E}_{1}(\phi)}$ and rewrite the $L^2$ gradient flow \eqref{230714-2} or \eqref{eq:rsf-mbe-el} as
\begin{subequations}
\begin{align}
& \frac{\partial \phi}{\partial t}= -\omega, \\
& \omega  =\mathcal{L}\phi+\frac{r}{\sqrt{\mathcal{E}_{1}(\phi)}}U(\phi), \\
& \frac{\mbox{d}r}{\mbox{d}t} =\frac{1}{2{\sqrt{\mathcal {E}_{1}(\phi)}}} \int_{\Omega}U\frac{\partial\phi}{\partial t}\mbox{d}x,
\end{align}
\end{subequations}

According to the standard procedure given in \cite{shen2019new, yao2020total},
taking the time step $\tau$, we derive the following first-order SAV scheme
\begin{subequations}\label{eq:sav}
\begin{align}
& \label{eq:sav1} \frac{\phi^{n+1}-\phi^{n}}{\tau}  =-\omega^{n+1}, \\
& \label{eq:sav2} \omega^{n+1}  =\mathcal{L}\phi^{n+1}+\frac{{r}^{n+1}}{\sqrt{\mathcal{E}_{1}(\phi^{n})}}U(\phi^{n}),\\
& \label{eq:sav3} \frac{{r}^{n+1}-r^{n}}{\tau}  = \frac{1}{2 \sqrt{\mathcal{E}_{1}(\phi^{n})}} \int_{\Omega}U(\phi^{n})  \frac{\phi^{n+1}-\phi^{n}}{\tau} \mbox{d} x.
\end{align}
\end{subequations}

A major advantage of SAV scheme is that it is easy to implement. In fact, subsisting \eqref{eq:sav2} and \eqref{eq:sav3} into \eqref{eq:sav1} and put the terms consisting $\phi^{n+1}$ on the left hand side, then
\begin{equation}\label{eq:savcom1}
\left(I+\tau \mathcal{L}\right) \phi^{n+1}+\frac{\tau}{2} b^{n} \left(b^{n},\phi^{n+1}\right) = c^n.
\end{equation}
where $b^{n}=U(\phi^{n})/ \sqrt{\mathcal{E}_{1}(\phi^{n})}$ and $c_{n} = \phi^{n}-\tau r^{n}  b^{n}+\frac{\tau}{2}(b^{n}, \phi^{n})b^{n}$ are know quantities.

To obtain $\phi^{n+1}$, denote $(I+\tau  \mathcal{L})$ by $A$. Multiplying \eqref{eq:savcom1} with $A^{-1}$ and taking the inner product with $b^n$, one finds
\begin{equation}\label{eq:savcom3}
(b^{n}, \phi^{n+1})=\frac{\left(b^{n},A^{-1} c^{n}\right)}{1+ \tau / 2 \left(b^{n}, A^{-1}b^{n} \right)}:=d^n.
\end{equation}
Combining with \eqref{eq:savcom1} and \eqref{eq:savcom3} gives
\begin{equation}\label{eq:savcom4}
\phi^{n+1}=A^{-1}\left(c^n-\frac{\tau}{2}b^{n}d^{n}\right).
\end{equation}

We sum up the above SAV scheme as \eqref{alg:SAV}.
\begin{algorithm}
\caption{SAV algorithm of MBE-RSF model}
\label{alg:SAV}
\begin{algorithmic}
\State{Input: $\phi_0$, $\tau$, iterNum, model parameters}
\State{For $n=1$ to iterNum }
\State{Update $\phi^{n+1}$ according the initial function $\phi_0$ }
\State{Calculate $b^n=U(\phi^{n})/ \sqrt{\mathcal{E}_{1}(\phi^{n})}$ and $c^n$(the righthand side of \eqref{eq:savcom1})}
\State{Let $A=I+\tau  \mathcal{L}$ and calculate $(b^n,\phi^{n+1})$ from \eqref{eq:savcom3}}
\State{Calculate $\phi^{n+1}$ from \eqref{eq:savcom4}}
\end{algorithmic}
\end{algorithm}

Then we have following theorem to guarantee the above scheme \eqref{eq:sav} is energy stable for the gradient flow \eqref{eq:rsf-mbe-el}.

\begin{theorem}
The SAV scheme \eqref{eq:sav} is first-order unconditionally energy stable in the sense that:
\begin{align*}
\frac{1}{\tau } & {\left(\tilde{\mathcal{E}}(\phi^{n+1}, r^{n+1})-\tilde{\mathcal{E}}(\phi^{n}, r^{n})\right) } \\
& +\frac{1}{\tau}(\frac{1}{2}\left(\phi^{n+1}-\phi^{n}, \mathcal{L}\left(\phi^{n+1}-\phi^{n}\right)\right)+\left(r^{n+1}-r^{n}\right)^{2})=-(\omega^{n+1},  \omega^{n+1}),
\end{align*}
where $\tilde{\mathcal{E}}(\eta, r)=\dfrac{1}{2}\left(\eta, \mathcal{L} \eta\right)+r^{2}$ is the modified energy.
\end{theorem}
\begin{proof}
Firstly, we multiply the equations \eqref{eq:sav} by $\omega^{n+1}$, $(\phi^{n+1}-\phi^{n}) / \tau$, $2r^{n+1}$, respectively. Next, integrating the first two equations, adding them together, and using the identity$$2 a^{n+1}\left(a^{n+1}-a^{n}\right)=\left(a^{n+1}\right)^{2}-\left(a^{n}\right)^{2}+\left(a^{n+1}-a^{n}\right)^{2},$$ dropping some unnecessary terms, we obtain the stability result.
\end{proof}

To illustrate the stability of the SAV scheme \eqref{eq:sav}, we applied it to real medical images, as shown in Fig. \ref{fig:testenergy}. We conducted a series of experiments using varying time steps, keeping the initial level set position same and performing $1000$ iterations to ensure reliable results. The experiments show that accurate image segmentation outcomes can be achieved with different time steps when using the SAV scheme. With smaller time steps, the energy decay rate during the initial evolution process is relatively slow (Fig. \ref{plot001}, Fig. \ref{plot01}). Moreover, as the time steps increase, the number of iterations required for the energy to stabilize significantly decreases (Fig. \ref{plot05}, Fig. \ref{plot1}), indicating potential computational efficiency gains at larger time steps. Furthermore, the modified energy consistently declines over iterations, as demonstrated in Fig. \ref{fig:plotenergy}, confirming the robustness and effectiveness of the SAV scheme under various conditions.
\begin{figure}[htbp]
	\centerline{
	\subfigcapskip=-3pt
		\subfigure[$\tau=0.01$]{\includegraphics[height=0.22\textwidth,width=0.2\textwidth]{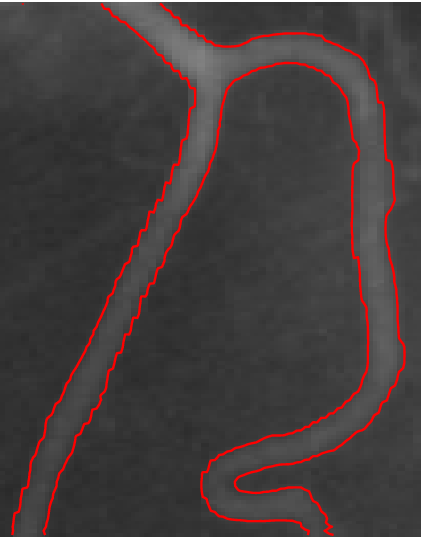}\label{energy001}}\quad
		\subfigure[$\tau=0.1$]{\includegraphics[height=0.22\textwidth,width=0.2\textwidth]{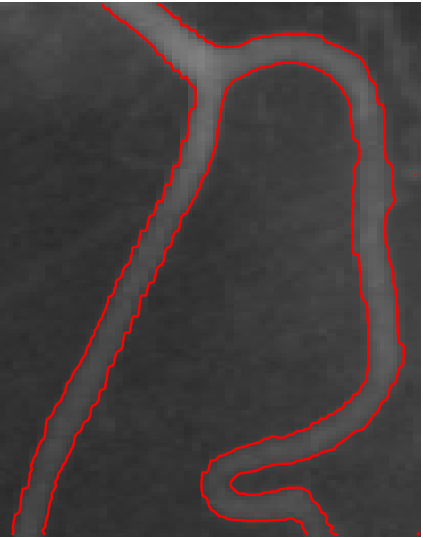}\label{energy01}}\quad
		\subfigure[$\tau=0.5$]{\includegraphics[height=0.22\textwidth,width=0.2\textwidth]{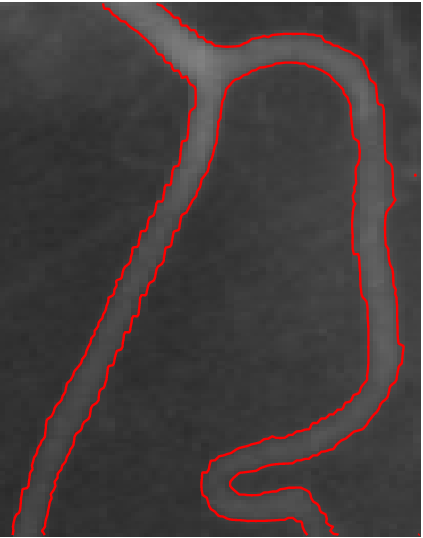}\label{energy05}}\quad
		\subfigure[$\tau=1$]{\includegraphics[height=0.22\textwidth,width=0.2\textwidth]{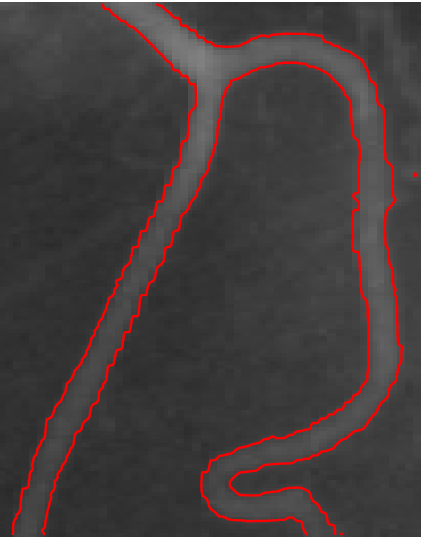}\label{energy1}}\quad

	}
	\caption{Segmentation results of different time steps: (a) $\tau=0.01$; (b) $\tau=0.1$; (c) $\tau=0.5$; (d) $\tau=1$.}\label{fig:testenergy}
\end{figure}

\begin{figure}[htbp]
	\centerline{
	\subfigcapskip=-3pt
		\subfigure[$\tau=0.01$]{\includegraphics[height=0.35\textwidth,width=0.40\textwidth]{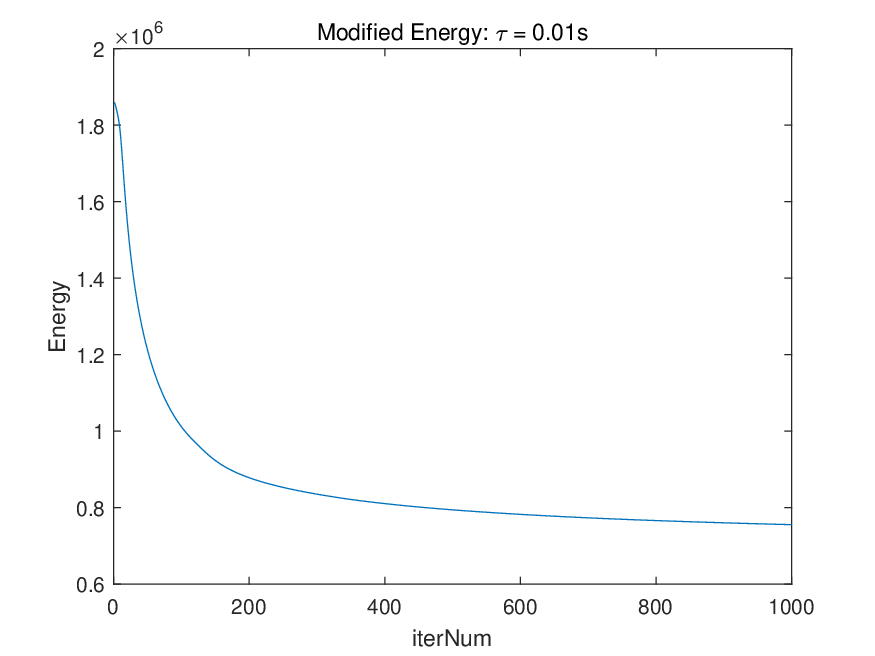}\label{plot001}}\quad
		\subfigure[$\tau=0.1$]{\includegraphics[height=0.35\textwidth,width=0.40\textwidth]{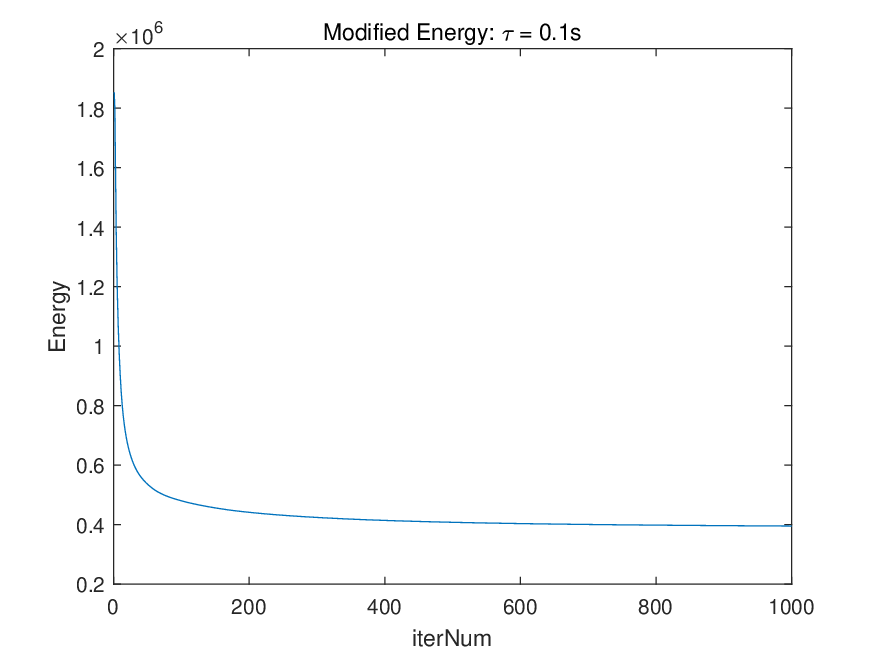}\label{plot01}}
		}\vspace{0.1cm}
\centerline{ \subfigure[$\tau=0.5$]{\includegraphics[height=0.35\textwidth,width=0.40\textwidth]{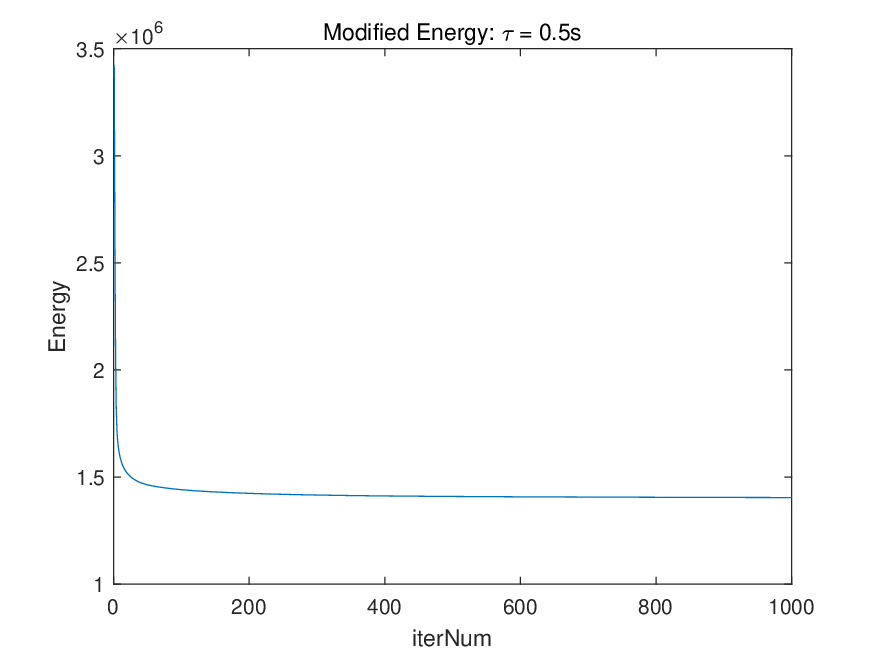}\label{plot05}}\quad
		\subfigure[$\tau=1$]{\includegraphics[height=0.35\textwidth,width=0.40\textwidth]{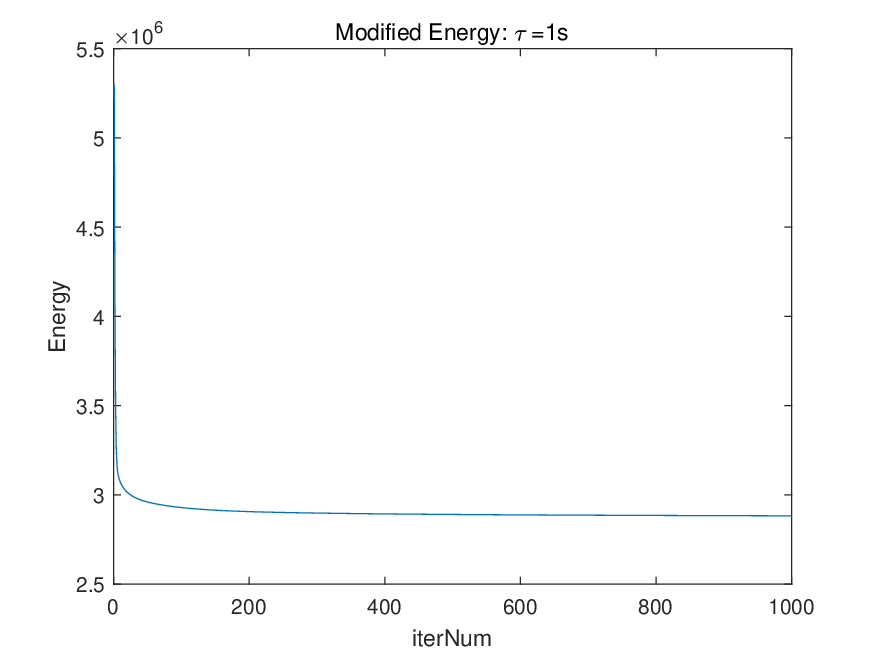}\label{plot1}}

	}
	\caption{Corresponding modified energy $\tilde{\mathcal{E}}$ of different time steps: (a) $\tau=0.01$; (b) $\tau=0.1$; (c)$\tau=0.5$; (d) $\tau=1$.}\label{fig:plotenergy}
\end{figure}

\subsection{The discretization of space}

We employ the $n+1$ level data only some of the linear terms, and the others we employ the $n$ level data, which confirms we could use discrete Fourier transform to solve the $\phi^{n+1}$.
Define the discrete backward and froward differential operators with periodic boundary condition as follows:
\begin{eqnarray*}
 &\partial_1^-\phi_{i,j} = \left\{\begin{array}{ll}
   \phi_{i,j} - \phi_{i-1,j}, & 1<i\le M,\\
   \phi_{1,j}-\phi_{M,j}, & i=1,
  \end{array}\right.\\
 &\partial_1^+\phi_{i,j} = \left\{\begin{array}{ll}
  \phi_{i+1,j} - \phi_{i,j}, & 1\le i< M,\\
  \phi_{1,j}-\phi_{M,j}, & i=M,
 \end{array}\right.\\
 &\partial_2^-\phi_{i,j} = \left\{\begin{array}{ll}
  \phi_{i,j} - \phi_{i,j-1}, & 1< j\le N,\\
  \phi_{i,1}-\phi_{i,N}, & j=1,
 \end{array}\right.\\
   &\partial_2^+\phi_{i,j} = \left\{\begin{array}{ll}
 \phi_{i,j+1} - \phi_{i,j}, & 1\le j< N,\\
 \phi_{i,1}-\phi_{i,N}, & j=N,
   \end{array}\right.
\end{eqnarray*}
and the central difference operators, the gradient, divergence, Laplace, and biharmonic operators are defined accordingly as
\begin{equation*}
\begin{aligned}
	& \partial_1^c \phi_{i,j} = (\partial_1^-\phi_{i,j} + \partial_1^+\phi_{i,j})/2,\\
	& \partial_2^c \phi_{i,j} = (\partial_2^-\phi_{i,j} + \partial_2^+\phi_{i,j})/2,\\
	 & \nabla^c\phi_{i,j} = \langle \partial_1^c\phi_{i,j}, \partial_2^c\phi_{i,j}\rangle,\\
	 & \mbox{div}^c(\langle\phi_{i,j}, \psi_{i,j} \rangle) = \partial_1^c\phi_{i,j} +  \partial_2^c\psi_{i,j},\\
	& \Delta^c\phi_{i,j} = \partial_1^+\partial_1^-\phi_{i,j} + \partial_2^+\partial_2^-\phi_{i,j},\\
	 & \Delta^{2c}\phi_{i,j} =  \Delta^c\Delta^c\phi_{i,j}.
\end{aligned}
\end{equation*}

As the low computational efficiency of finite difference schemes for fourth-order equations, we apply the discrete Fourier transform $\mathcal{F}$ to both sides \cite{1984FFT},
\begin{equation}\label{eq:fourier-operator}
	\begin{array}{l}
		\mathcal{F}\partial_1^{\pm} \phi_{i,j} = \pm(e^{\pm\sqrt{-1}z_i^1}-1)\mathcal{F}\phi_{i,j},\\
		\mathcal{F}\partial_2^{\pm} \phi_{i,j} = \pm(e^{\pm\sqrt{-1}z_j^2}-1)\mathcal{F}\phi_{i,j},
	\end{array}
\end{equation}
where $z_i^1 = 2\pi(i-1)/M, i=1,2,\cdots,M$, $z_j^2 = 2\pi(j-1)/N, j=1,2,\cdots,N$.

Therefore, once $\mathcal{F}\phi_{i,j}^{n+1}$ is calculated, we can obtain $\phi_{i,j}^{n+1}$ by the discrete inverse Fourier transform(IFFT).
Because of the utilize of the FFT, it is efficient to solve the designed numerical scheme.

In addition, the Dirac function $\delta(\phi)$ is approximated by the following forms,
\begin{equation}\label{eq:delta1}
	\delta_{1,\epsilon}(p) = \left\{\begin{array}{ll}
		0, & p\in\mathbb{R},\quad |p|>\epsilon,\\
		\dfrac{1}{2\epsilon}\left(1+\cos(\dfrac{\pi p}{\epsilon})\right), & |p|\leq\epsilon,
	\end{array}\right.
\end{equation}
and
\begin{equation}\label{eq:delta2}
\delta_{2,\epsilon}(p)={\frac{1}{\pi}}\cdot \frac{\epsilon}{{\epsilon^2}+{p^2}}, \quad p\in\mathbb{R}.
\end{equation}

The support of equation \eqref{eq:delta1} is limited, which can cause the level set to become trapped in local minima during evolution. On the other hand, equation \eqref{eq:delta2} acts on all curves, generating new curves and leading to improved global minimum results. In this paper, we choose the equation \eqref{eq:delta2} as the smooth approximation.

\section{Experimental results} \label{sec:exresults}
In this section, we compare the experimental results of the MBE-GAC and MBE-RSF models for image segmentation.  Our findings show that the MBE-GAC model performs well in segmenting images with clear edges, while the MBE-RSF model is more effective for images with fuzzy edges.  Moreover, we present the segmentation results obtained using the MBE-RSF model.  Comparative experiments indicate that the MBE regularization term is more robust to noise compared to other regularization terms, and eliminates the need for level set re-initialization.  Furthermore, the MBE-RSF model is capable of handling a wide range of composite and real images, including medical images and those with complex boundaries and intensity inhomogeneity distributions.  These results demonstrate the effectiveness and versatility of the MBE-based models for various image segmentation tasks.

Now we provide the meanings of each parameter in the experimental conclusions. $\mu$ is the coefficient of the regularization term and $\nu$ is the coefficient of the length term in model \eqref{eq:seg-model}, $\alpha$ is the coefficient of the biharmonic term in MBE equation, $\lambda_{i}$ and $\sigma$ are parameters in the MBE-RSF model\eqref{eq:rsf-mbe-el}, $\tau$ is the time step, and $iter$ represents the number of iterations.

\subsection{Independence of initial curve}\label{sec:5-1}
To verify the independence of the MBE regularization term on the initial curve, we perform an experiment using two composite images with different initial curves. The first image, depicted in Fig. \ref{fig:text-flower}, is a uniform intensity image. We apply the MBE-GAC model and classical regularization terms with the same parameters to the image. The experimental results show that the MBE regularization term performs similarly to the classical regularization term, but with a smaller range of variation in $|\nabla \phi|$. Moreover, the segmentation results are accurate for different initial curves and can effectively identify the target boundary.

The second image, shown in Fig. \ref{fig:text-2}, has heterogeneous regions. We apply the MBE-RSF model to the image and compare its segmentation performance with DR1 regularization term. The experimental results demonstrated that the segmentation curves generated by the MBE-GAC and MBE-RSF models are smoother than those produced by DR1 regularization term. For instance, the boundary curve of the quadrilateral in Fig. \ref{fig:text-2-i} is smoother than that in Fig. \ref{fig:text-2-g}. We also enlarged details in Fig. \ref{figdetail1} and Fig. \ref{figdetail2} to highlight the effectiveness of MBE regularization in achieving smoother contours compared to DR1 regularization Overall, our findings indicate that the MBE regularization term is effective in producing smooth and accurate segmentation results that are independent of the initial curve.
\begin{figure}[htbp]
	\centerline{
		\subfigcapskip=-3pt
	\subfigure[Initial]{\includegraphics[height=0.18\textwidth, width=0.18\textwidth]{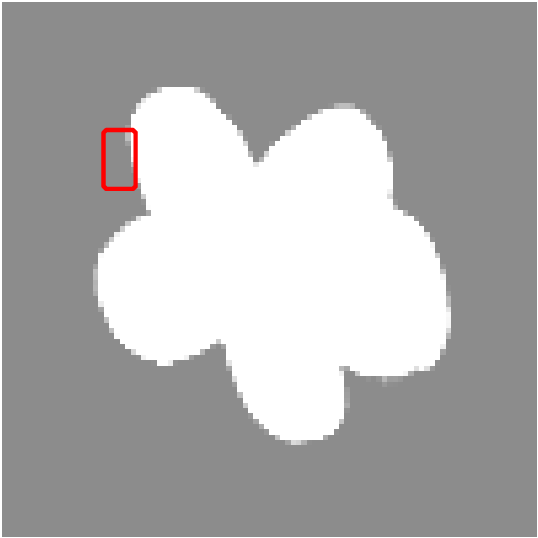}}
	\subfigure[DR1-GAC]{\includegraphics[height=0.18\textwidth, width=0.18\textwidth]{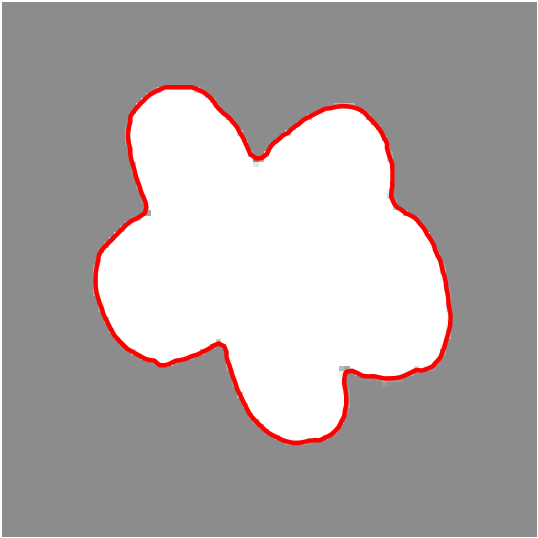}\label{fig:111}}
	\subfigure[$|\nabla \phi|$ of (b)]{\includegraphics[height=0.18\textwidth, width=0.19\textwidth]{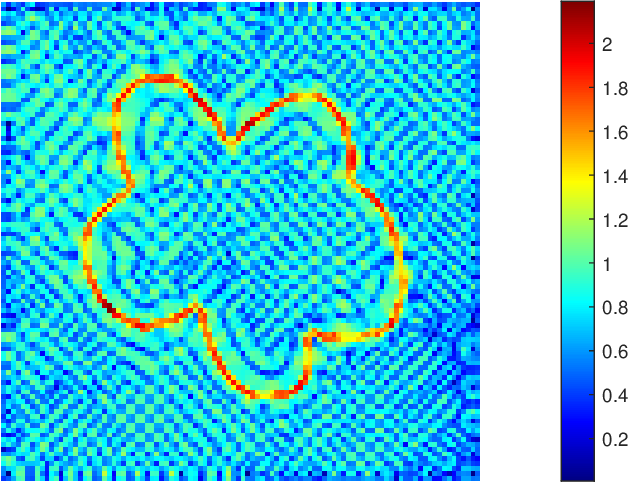}}
\subfigure[MBE-GAC]{\includegraphics[height=0.18\textwidth, width=0.18\textwidth]{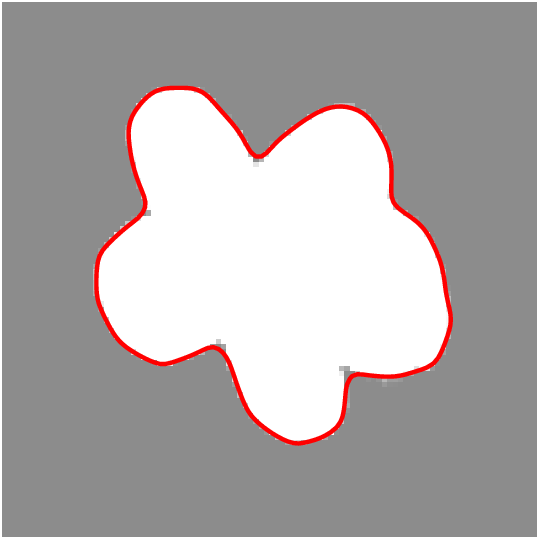}\label{fig:112}}
\subfigure[$|\nabla \phi|$ of (d)]{\includegraphics[height=0.18\textwidth, width=0.19\textwidth]{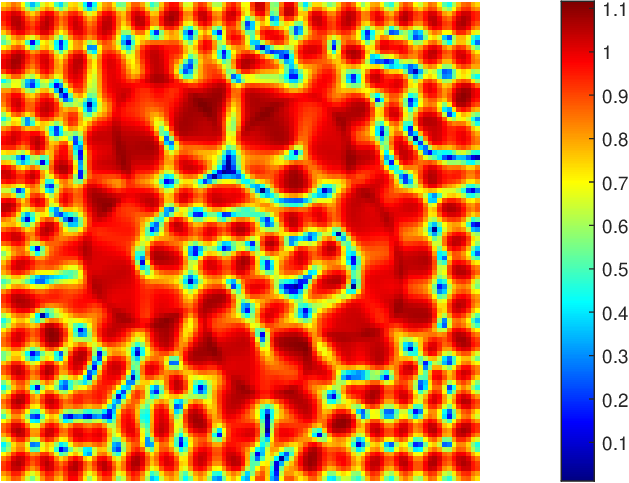}}
}\vspace{0.1cm}
\centerline{
	\subfigcapskip=-3pt
	\subfigure[Initial]{\includegraphics[height=0.18\textwidth, width=0.18\textwidth]{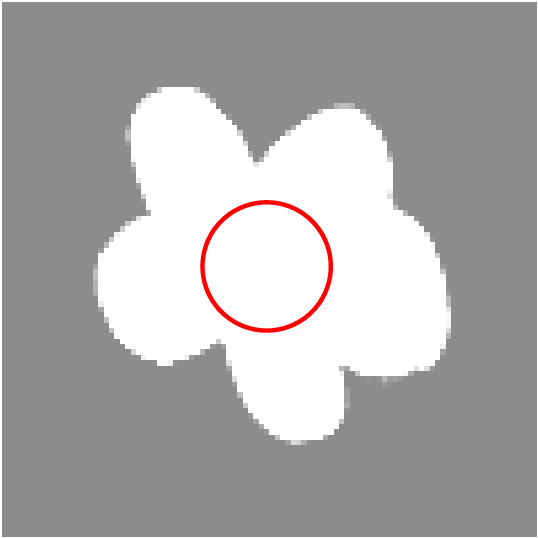}}
\subfigure[DR1-GAC]{\includegraphics[height=0.18\textwidth, width=0.18\textwidth]{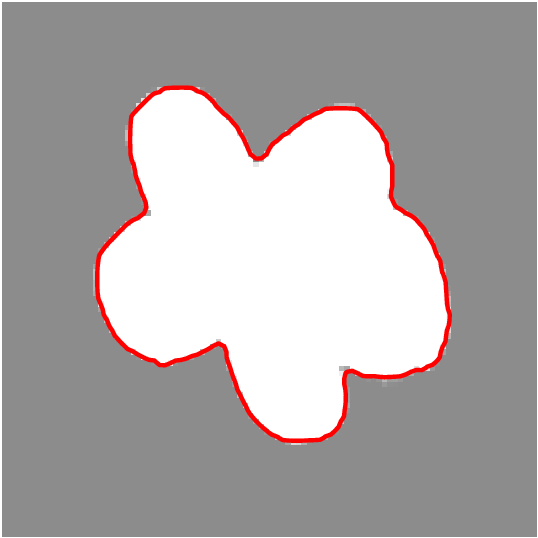}\label{fig:11}}
	\subfigure[$|\nabla \phi|$ of (g)]{\includegraphics[height=0.18\textwidth, width=0.19\textwidth]{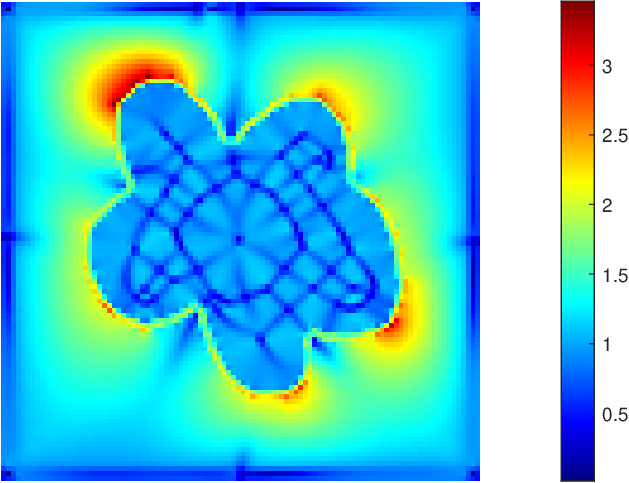}}
	\subfigure[MBE-GAC]{\includegraphics[height=0.18\textwidth, width=0.18\textwidth]{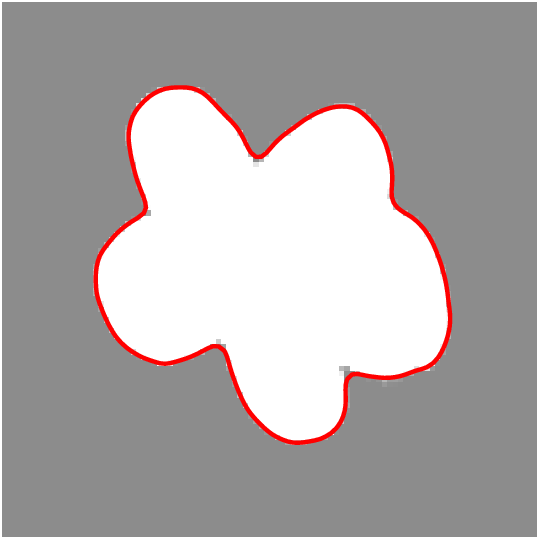}\label{fig:12}}
	\subfigure[$|\nabla \phi|$ of (h)]{\includegraphics[height=0.18\textwidth, width=0.19\textwidth]{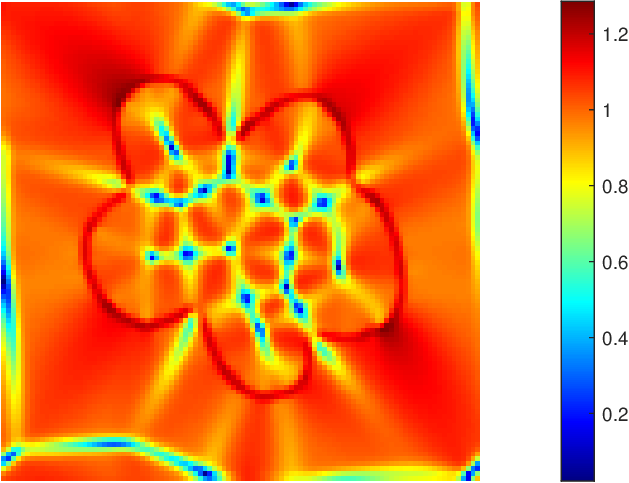}}
}
	\caption{Segmentation results of the DR1-GAC model and MBE-GAC model with same initial level set functions. Upper row: binary function
as the initial level set function; lower row: signed distance function as the initial level set function.}\label{fig:text-flower}
\end{figure}

\begin{figure}[htbp]
\begin{center}
\includegraphics[height=.5\textwidth,width=.65\textwidth ]{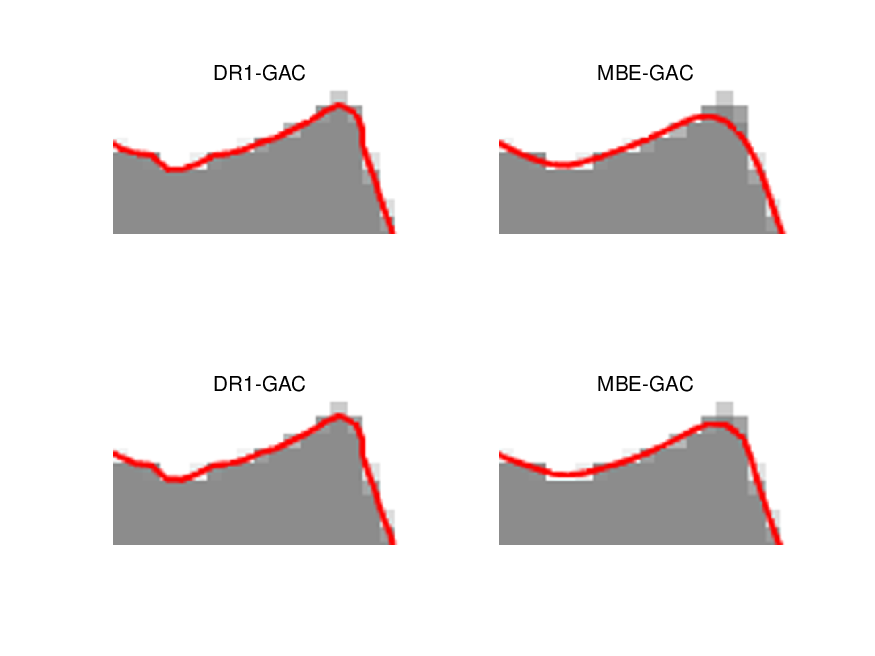}
\end{center}
\caption{Local details of DR1-GAC model and MBE-GAC model with same initial level set functions. Upper row: binary function
as the initial level set function; lower row: signed distance function as the initial level set function}\label{figdetail1}
\end{figure}

\begin{figure}[htbp]
	\centerline{
		\subfigcapskip=-2pt
	\subfigure[Initial]{\includegraphics[height=0.18\textwidth,width=0.18\textwidth]{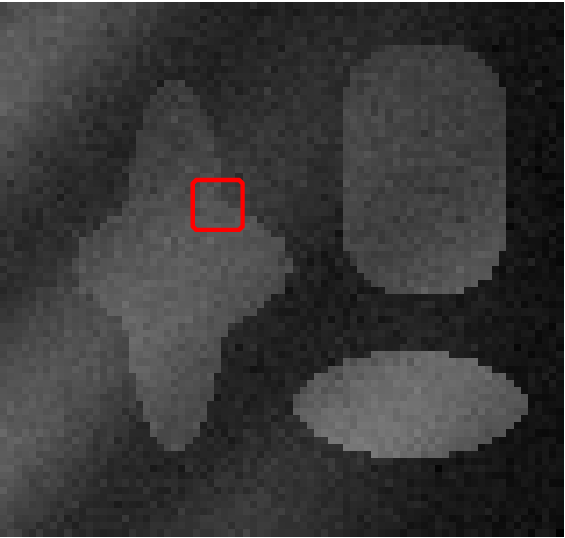}} \subfigure[DR1-RSF]{\includegraphics[height=0.18\textwidth,width=0.18\textwidth]{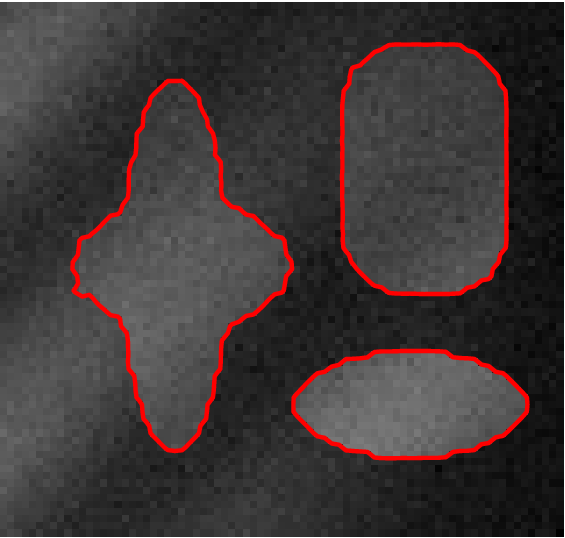}\label{fig:21}}
	\subfigure[$|\nabla \phi|$ of (b)]{\includegraphics[height=0.18\textwidth,width=0.19\textwidth]{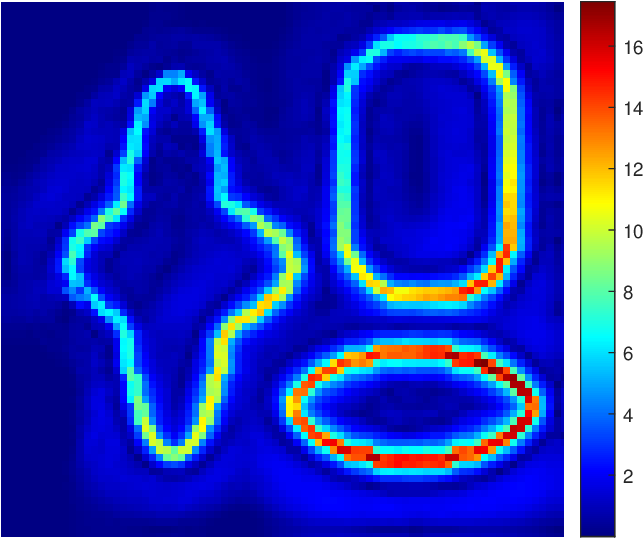}}
\subfigure[MBE-RSF]{\includegraphics[height=0.18\textwidth,width=0.18\textwidth]{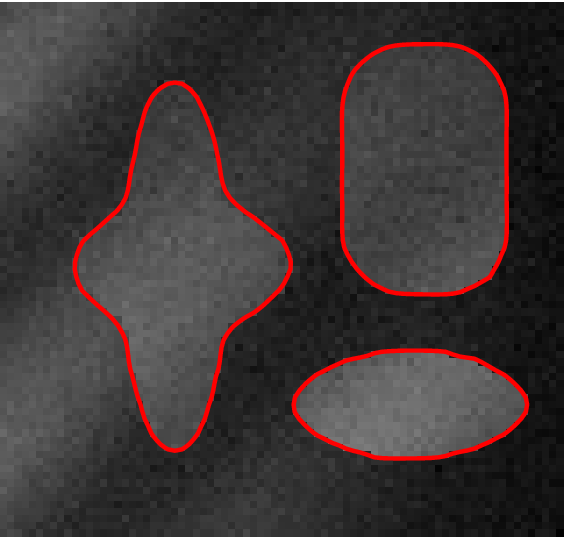}\label{fig:22}}
\subfigure[$|\nabla \phi|$ of (d)]{\includegraphics[height=0.18\textwidth,width=0.19\textwidth]{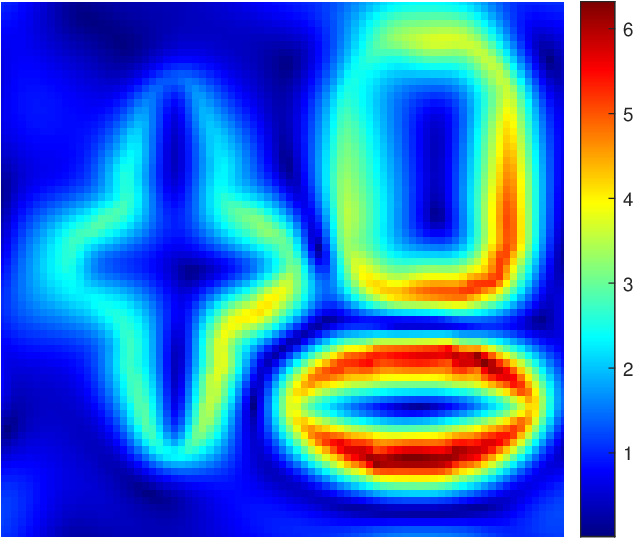}}
}\vspace{0.1cm}
\centerline{
\subfigbottomskip=5pt
	\subfigcapskip=-2pt
	\subfigure[Initial]{\includegraphics[trim={3cm 1cm 3cm 1cm},clip,height=0.18\textwidth, width=0.18\textwidth]{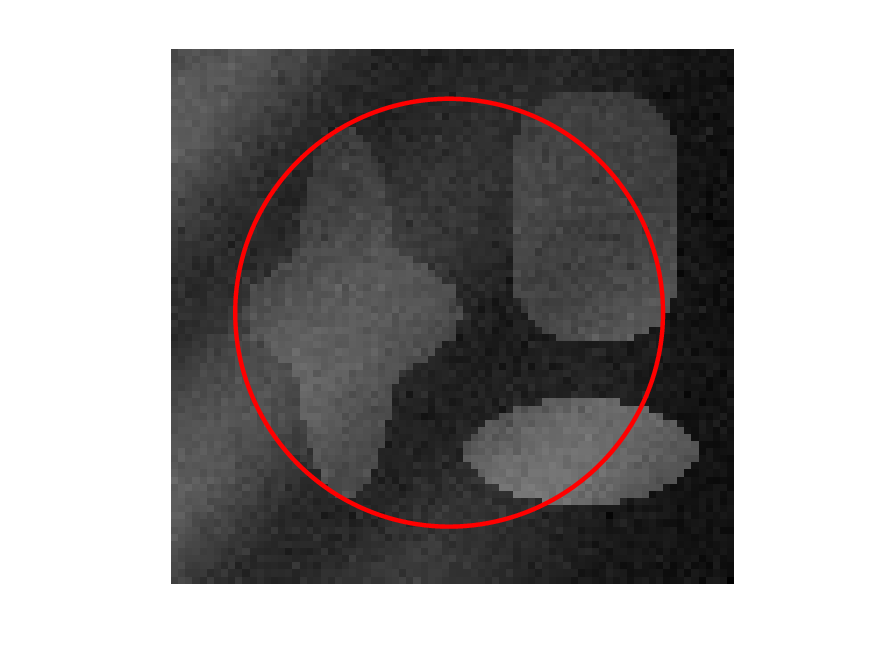}}
\subfigure[DR1-RSF]{\includegraphics[height=0.18\textwidth, width=0.18\textwidth]{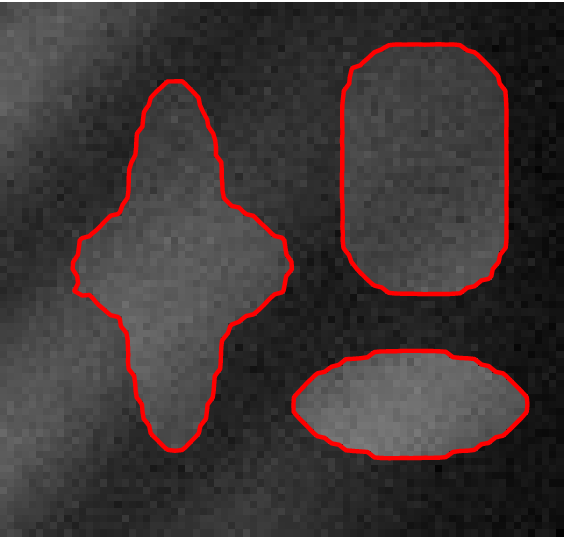}}\label{fig:text-2-g}
	\subfigure[$|\nabla \phi|$ of (g)]{\includegraphics[height=0.18\textwidth, width=0.19\textwidth]{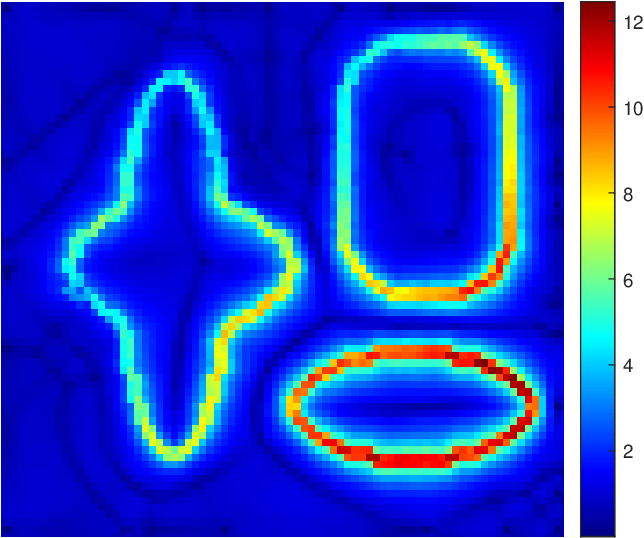}}
	\subfigure[MBE-RSF]{\includegraphics[height=0.18\textwidth,  width=0.18\textwidth]{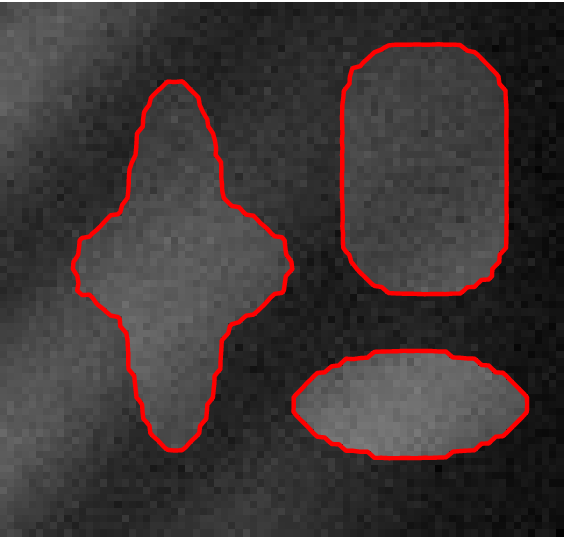}}\label{fig:text-2-i}
	\subfigure[$|\nabla \phi|$ of (i)]{\includegraphics[height=0.18\textwidth, width=0.19\textwidth]{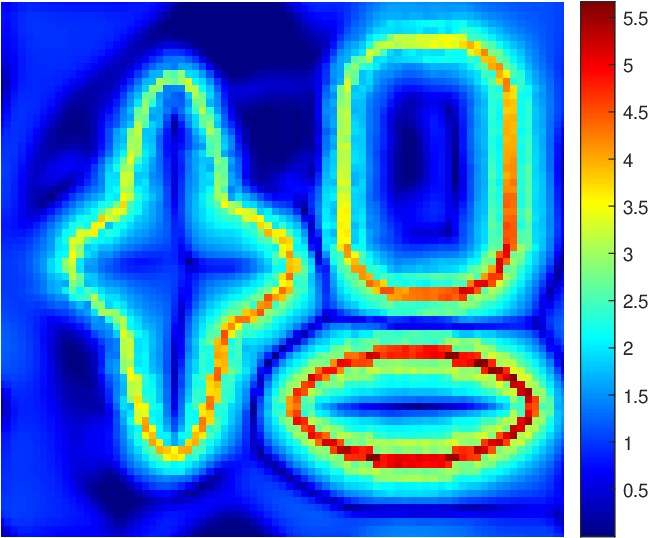}}
}
	\caption{Segmentation results of DR1-RSF model and MBE-RSF model with same initial level set functions. Upper row: binary function
as the initial level set function; lower row: signed distance function as the initial level set function.}\label{fig:text-2}
\end{figure}

\begin{figure}[htbp]
\begin{center}
\includegraphics[height=.5\textwidth,width=.65\textwidth ]{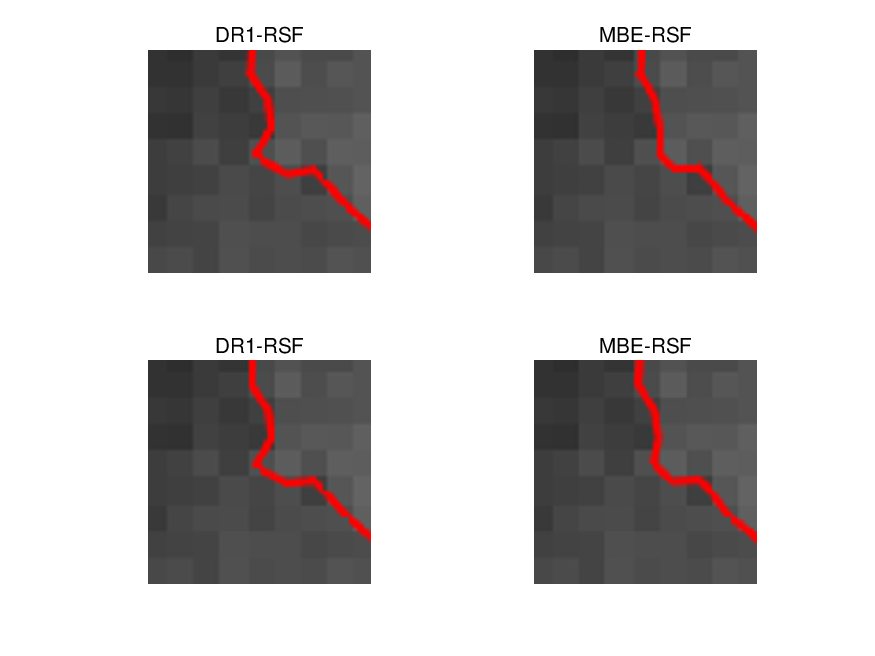}
\end{center}
\caption{Local details of DR1-RSF model and MBE-RSF model with same initial level set functions. Upper row: binary function
as the initial level set function; lower row: signed distance function as the initial level set function}\label{figdetail2}
\end{figure}
\subsection{Property of smooth control}\label{sec:5-2}


In the variational level set method, the coefficient $\nu$ of the arc length term controls the degree of smoothing, with larger values leading to higher curve smoothness and smoother shapes. By adjusting the value of $\nu$, the curve shape and smoothness can be flexibly controlled.

We test the MBE regularization term on a boundary-blurred image shown in Fig. \ref{fig:length} and set $\lambda_1=0.33$, $\lambda_2=0.67$, $\epsilon=1$, $\sigma=3$, $\tau = 0.01$, $iter=2000$. Our experimental results show that even when $\nu=0$, the biharmonic term can still ensure the smoothness of the segmentation curve and requires fewer iterations. However, lowering the bi-harmonic term parameter $\mu$ can lead to some erroneous segmentation results, as seen in images Fig. \ref{fig:length7} and Fig. \ref{fig:length8}. This suggests that the fourth-order term in the MBE equation has a greater control over the segmentation results than the biharmonic term when $\mu$ is low.

Based on these findings, we plan to further explore the smooth control characteristics of the MBE regularization term. Specifically, we will investigate how the different terms in the MBE equation affect the smoothness and accuracy of the segmentation results under different parameter settings. This will enable us to better understand the behavior of the MBE regularization term and improve its performance in various segmentation tasks.
\begin{figure}[htbp]
	\centerline{
	\subfigcapskip=-3pt
		\subfigure[$\alpha$=50]{\includegraphics[height=0.20\textwidth,width=0.25\textwidth]{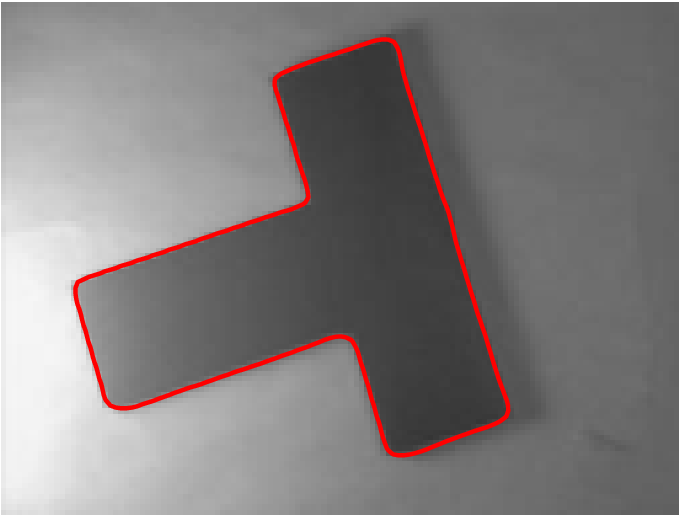}\label{fig:length3}}\quad
		\subfigure[$\alpha$=100]{\includegraphics[height=0.20\textwidth,width=0.25\textwidth]{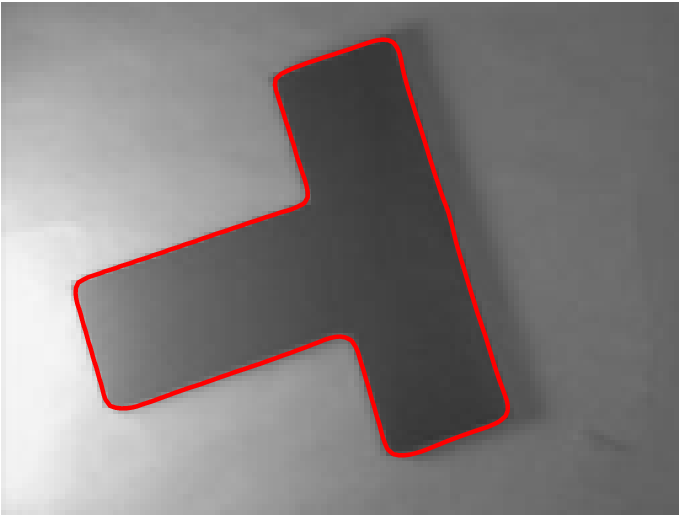}\label{fig:length4}}\quad
		\subfigure[$\alpha$=200]{\includegraphics[height=0.20\textwidth,width=0.25\textwidth]{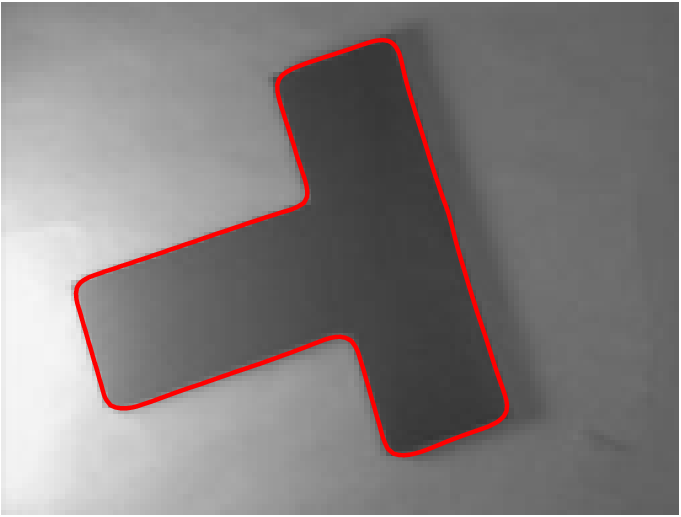}\label{fig:length5}}
	}\vspace{0.1cm}
      \centerline{
		\subfigure[$\alpha$=1]{\includegraphics[height=0.20\textwidth,width=0.25\textwidth]{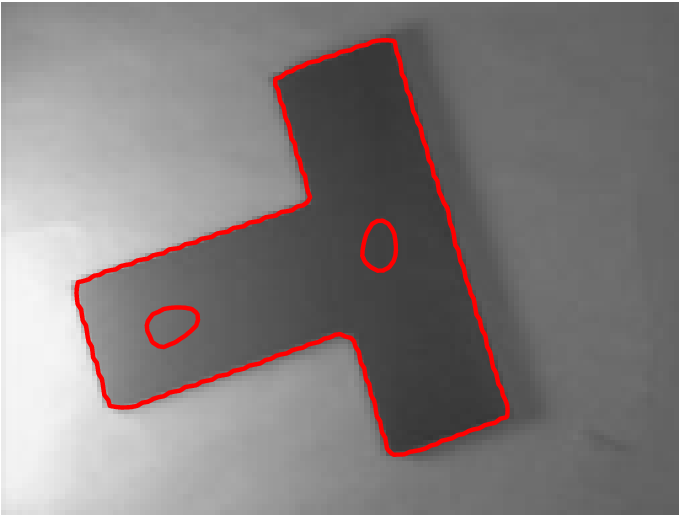}\label{fig:length7}}\quad
		\subfigure[$\alpha$=5]{\includegraphics[height=0.20\textwidth,width=0.25\textwidth]{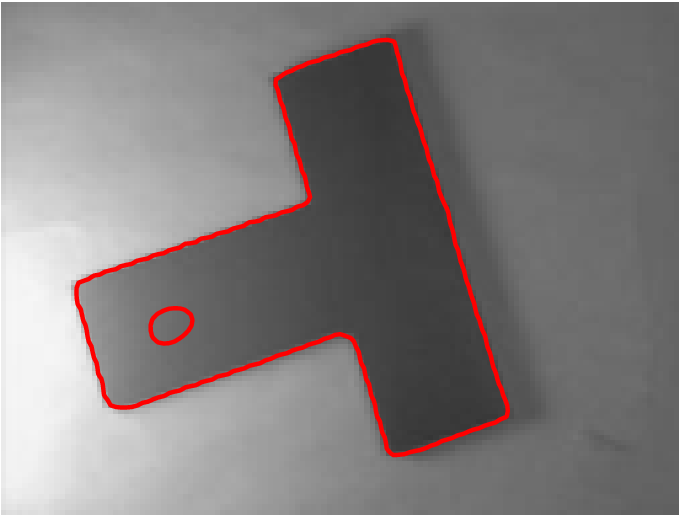}\label{fig:length8}}\quad
		\subfigure[$\alpha$=20]{\includegraphics[height=0.20\textwidth,width=0.25\textwidth]{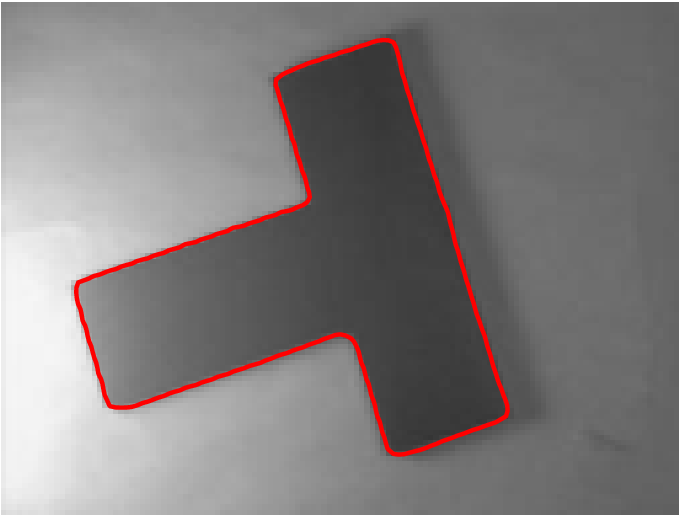}\label{fig:length9}}
	}
	\caption{Segmentation results of the MBE-RSF model. Upper row: results without arc length term; lower row: results with arc length term when $\nu=10$.}\label{fig:length}
\end{figure}
\begin{table}[htbp]
	\centering
 \setlength{\tabcolsep}{5mm}
	\caption{Parameters of Fig. \ref{fig:length}.}
	\begin{tabular}{lccccccc}
		\hline
		\multicolumn{1}{r}{} & $\lambda_1$ &$\lambda_2$ & $\mu$ &$\alpha$ & $\nu$ &$\tau$ & iter \\
		\hline
		Fig. \ref{fig:length3} & 1     & 2     & 1     & 50    & 0     & 0.01  & 3000 \\
		
		Fig. \ref{fig:length4} & 1     & 2     & 1     & 100   & 0     & 0.01  & 3000 \\
		
		Fig. \ref{fig:length5} & 1     & 2     & 1     & 200   & 0     & 0.01  & 3000 \\
		
		Fig. \ref{fig:length7} & 1     & 3     & 1     & 1     & 10    & 0.01  & 10000 \\
		
		Fig. \ref{fig:length8} & 1     & 3     & 1     & 5     & 10    & 0.01  & 10000 \\
		
		Fig. \ref{fig:length9} & 1     & 3     & 1     & 20    & 10    & 0.01  & 10000 \\
		\hline
	\end{tabular}%
	\label{tab:test_length}%
\end{table}

\begin{figure}[htbp]
	\centerline{
		\subfigcapskip=-3pt
		\subfigure[]{\includegraphics[height=0.25\textwidth,width=0.25\textwidth]{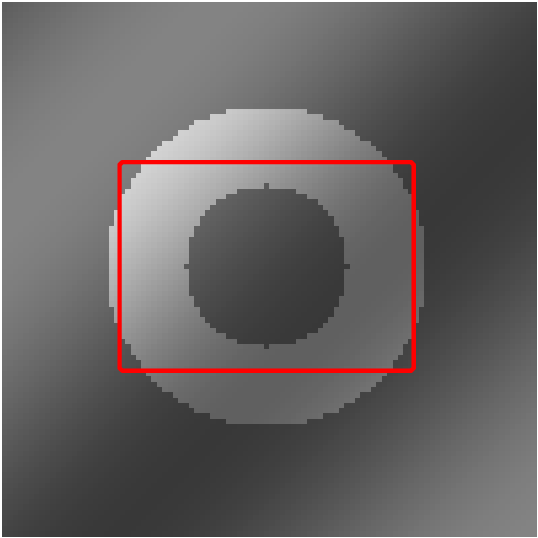}\label{fig:lengh-9-init1}}\quad
\subfigure[]{\includegraphics[height=0.25\textwidth,width=0.25\textwidth]{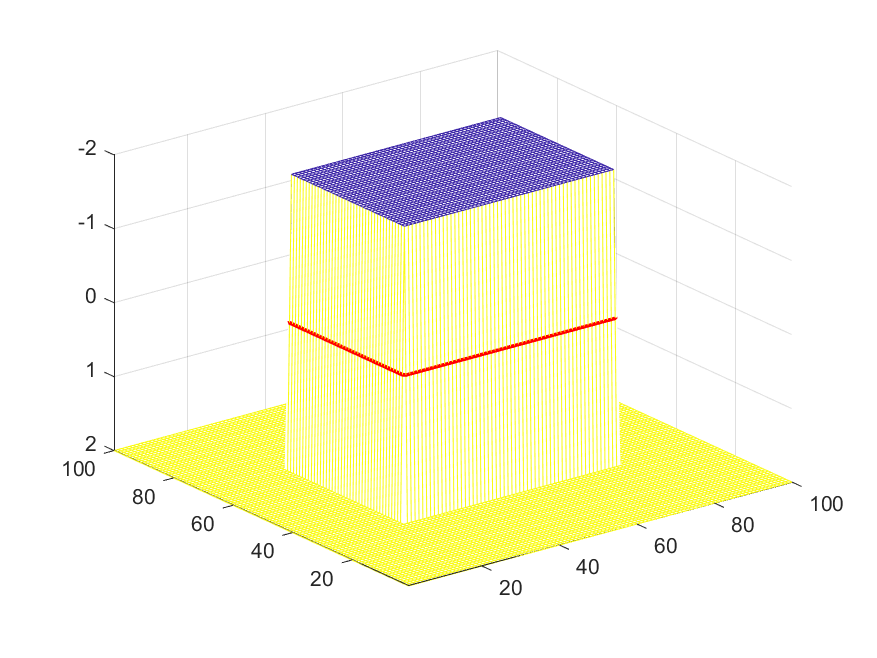}}
	}
	\caption{The initial curve and level set function.}\label{fig:length-9-init}
\end{figure}

We utilize the MBE-RSF model to distinguish the ring from the background. The initial curve and its corresponding level set function are depicted in Fig. \ref{fig:length-9-init}. With parameters set to $\mu=1$, $\tau=0.01$, we observe that increasing the coefficient of the bi-harmonic term results in smoother curves. For instance, in Fig. \ref{fig:length_9-6}, Fig. \ref{fig:length_9-2}, and Fig. \ref{fig:length_9-4}, all parameters except $\alpha$ are kept the same, and the segmentation result of Fig. \ref{fig:length_9-4} is notably smoother, demonstrating that the bi-harmonic term effectively regulates the smoothness of the curve. Additionally, we set $\alpha=1$, $\nu=0, 10, 500$ for the results shown in Fig. \ref{fig:length_9-1}, Fig. \ref{fig:length_9-6}, and Fig. \ref{fig:length_9-7}, respectively, while keeping other parameters unchanged. We find that even with $\nu=500$ in Fig. \ref{fig:length_9-7}, the local smoothness of the curves remains the same. This is because the arc-length term, used as a regularization term, only controls the macroscopic length of the curve without improving its local smoothness.

\begin{figure}[htbp]
	\centerline{
		\subfigcapskip=-3pt

	\subfigure[$\alpha=1$,$\nu=0$]{\includegraphics[height=0.22\textwidth,width=0.22\textwidth]{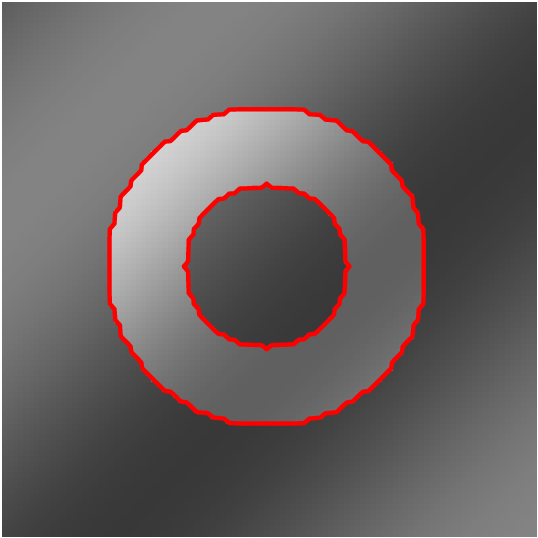}\label{fig:length_9-1}}\quad
	\subfigure[$\alpha=1$,$\nu=10$]{\includegraphics[height=0.22\textwidth,width=0.22\textwidth]{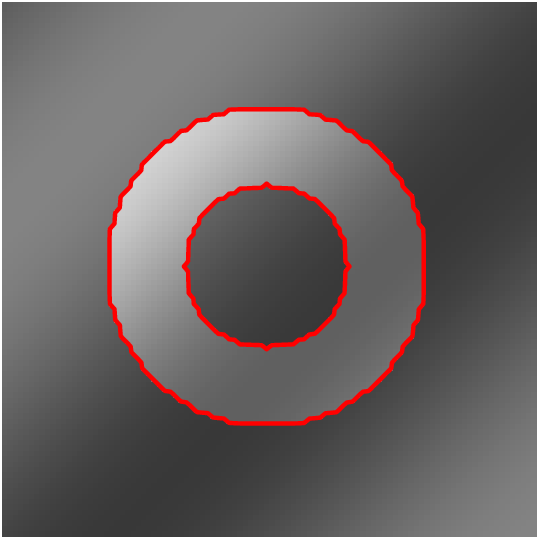}\label{fig:length_9-6}}\quad
	\subfigure[$\alpha=10$,$\nu=10$]{\includegraphics[height=0.22\textwidth,width=0.22\textwidth]{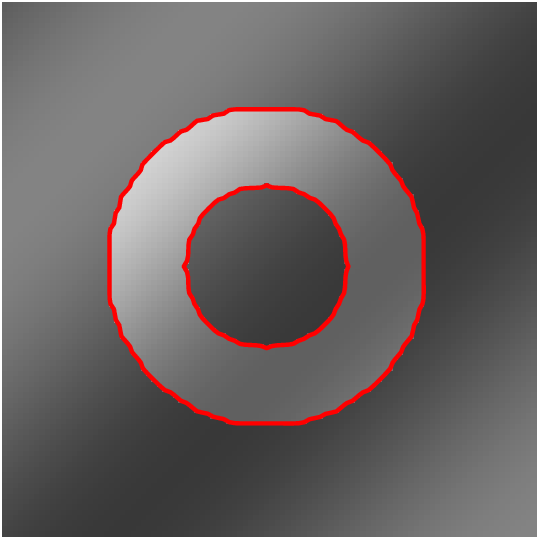}\label{fig:length_9-2}}
}\vspace{0.1cm}
\centerline{
	\subfigcapskip=-3pt

	\subfigure[$\alpha=100$,$\nu=0$]{\includegraphics[height=0.22\textwidth,width=0.22\textwidth]{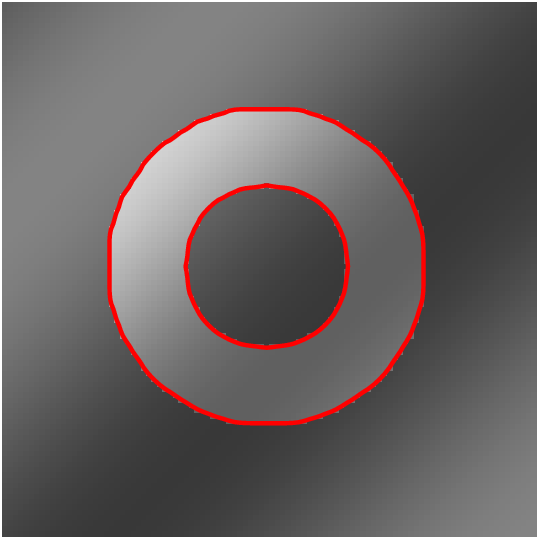}\label{fig:length_9-3}}\quad
	\subfigure[$\alpha=100$,$\nu=10$]{\includegraphics[height=0.22\textwidth,width=0.22\textwidth]{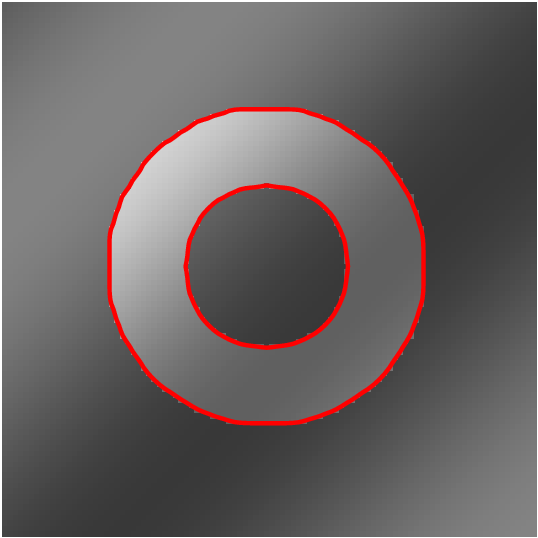}\label{fig:length_9-4}}\quad
	\subfigure[$\alpha=1$,$\nu=500$]{\includegraphics[height=0.22\textwidth,width=0.22\textwidth]{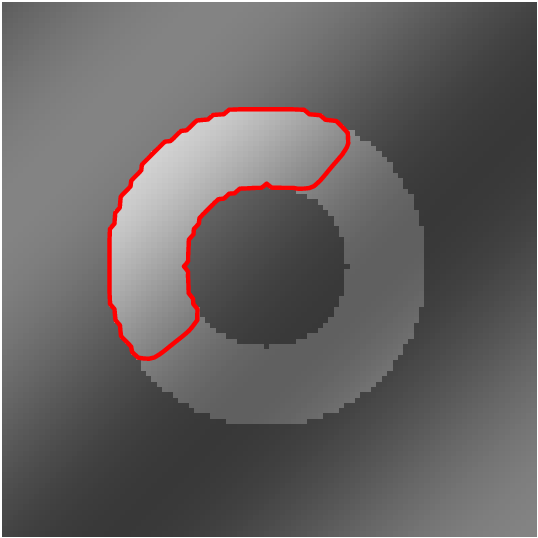}\label{fig:length_9-7}}
}
	\caption{Segmentation results of MBE-RSF model.}\label{fig:length-9-contour}
\end{figure}
\begin{table}[htbp]
	\centering
    \setlength{\tabcolsep}{6mm}
	\caption{Parameters of Fig. \ref{fig:length-9-contour}. ($\lambda_1=0.33$, $\lambda_2=0.67$, $\epsilon=1$, $\sigma=3$, $\tau = 0.01$)}
	\begin{tabular}{rccccccc}\hline
		 & $\mu$ &$\alpha$ & $\nu$ &$\tau$ & iter \\\hline
		Fig. \ref{fig:length_9-1}   & 1     & 1     & 0     & 0.01  & 2000 \\

		Fig. \ref{fig:length_9-6}   & 1 &1 & 10 & 0.01 & 2000\\

		Fig. \ref{fig:length_9-2}   & 1     & 10    & 10    & 0.01  & 2000 \\

		Fig. \ref{fig:length_9-3}  & 1     & 100   & 0     & 0.01  & 2000 \\

		Fig. \ref{fig:length_9-4}   & 1     & 100   & 10    & 0.01  & 2000 \\

		Fig. \ref{fig:length_9-7} & 1 & 1 & 500 & 0.01 & 2000\\\hline
	\end{tabular}
	\label{tab:test_length-9}%
\end{table}%

We tested an image with sharp corners, as shown in Figure \ref{fig:length-602-contour}. The initial contour is depicted in Figure \ref{fig:length-602-init}, and the parameters are listed in Table \ref{tab:test-602}. The results were similar. When $\alpha$ is small, the arc length term can only control the macro-scale curve length and cannot force the curve to evolve completely to the vertices of the object to be segmented (Fig. \ref{fig:length-602-0} and Fig. \ref{fig:length-602-5}); however, as the coefficient of the fourth-order term gradually increases, the segmentation curve further evolves towards the edges of the object (Fig. \ref{fig:length-602-0} and Fig. \ref{fig:length-602-2}); when the arc length term is consistent, the larger the coefficient of the fourth-order term, the higher the local smoothness of the segmentation curve (Fig. \ref{fig:length-602-1}, Fig. \ref{fig:length-602-4}, and Fig. \ref{fig:length-602-3}), but overall, when both the arc length term and higher-order terms are appropriately valued, it can ensure that the segmentation result is fine and can better maintain corner information (Fig. \ref{fig:length-602-1}).

\begin{figure}[htbp]
	\centerline{
		\subfigcapskip=-3pt
		\subfigure[]{\includegraphics[height=0.25\textwidth,width=0.25\textwidth]{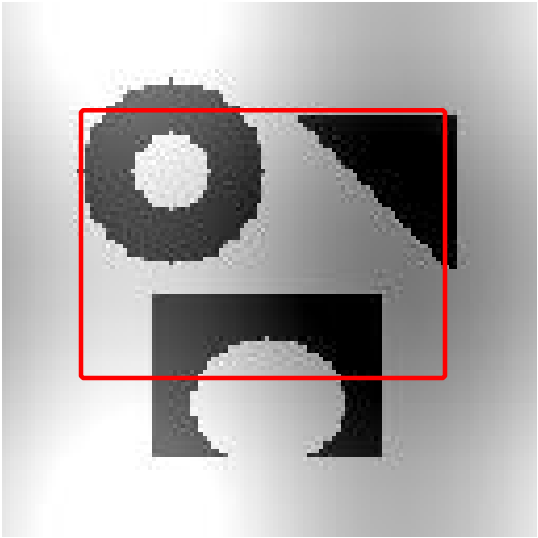}}\quad
		\subfigure[]{\includegraphics[height=0.25\textwidth,width=0.25\textwidth]{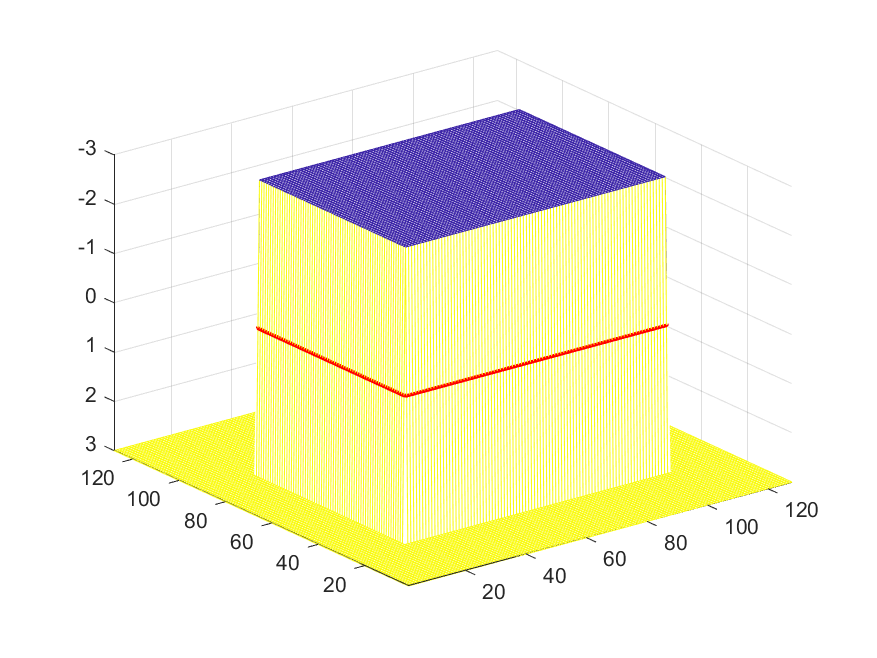}}
	}
	\caption{The initial curve and level set function.}\label{fig:length-602-init}
\end{figure}

\begin{figure}[htbp]
	\centerline{
		\subfigcapskip=-3pt
\subfigure[$\alpha=10$,$\nu=0$]{\includegraphics[height=0.22\textwidth,width=0.22\textwidth]{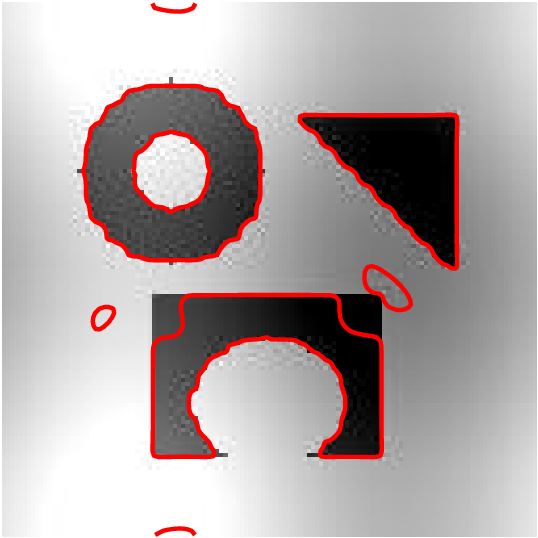}\label{fig:length-602-0}}\quad
\subfigure[$\alpha=30$,$\nu=100$]{\includegraphics[height=0.22\textwidth,width=0.22\textwidth]{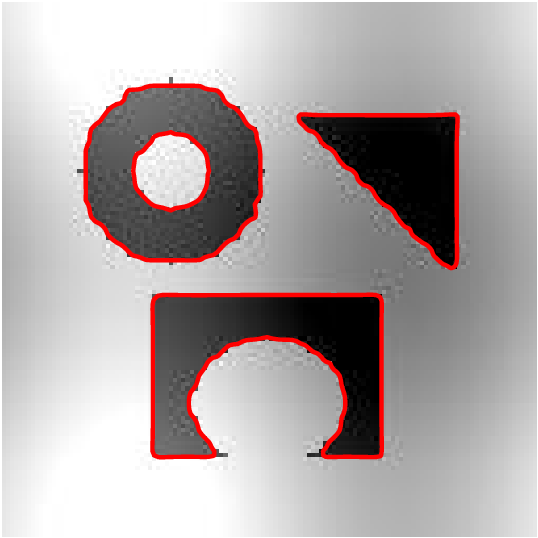}\label{fig:length-602-1}}\quad
\subfigure[$\alpha=200$,$\nu=0$]{\includegraphics[height=0.22\textwidth,width=0.22\textwidth]{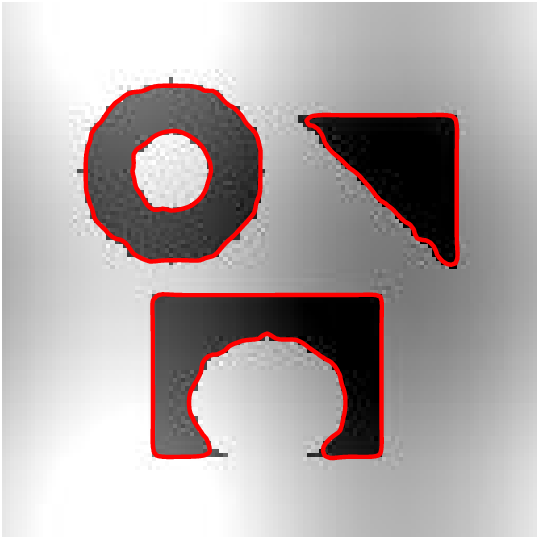}\label{fig:length-602-2}}
}\vspace{0.1cm}
\centerline{
	\subfigcapskip=-3pt
\subfigure[$\alpha=10$,$\nu=1000$]{\includegraphics[height=0.22\textwidth,width=0.22\textwidth]{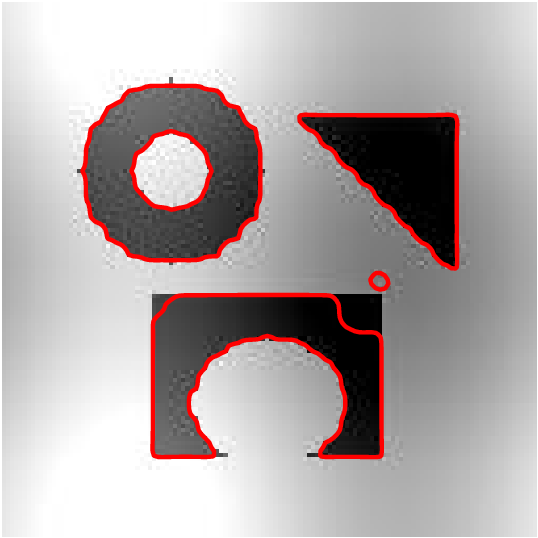}\label{fig:length-602-5}}\quad
\subfigure[$\alpha=100$,$\nu=100$]{\includegraphics[height=0.22\textwidth,width=0.22\textwidth]{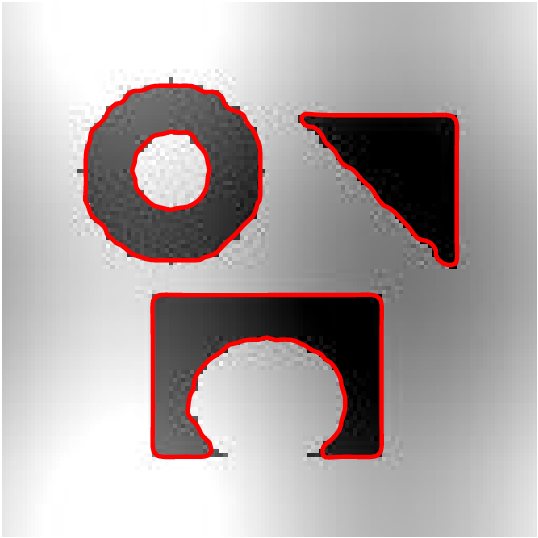}\label{fig:length-602-4}}\quad
\subfigure[$\alpha=200$,$\nu=100$]{\includegraphics[height=0.22\textwidth,width=0.22\textwidth]{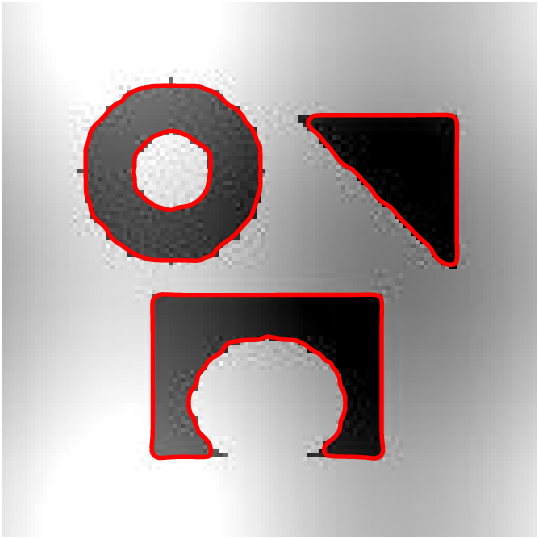}\label{fig:length-602-3}}
}
	\caption{Segmentation results of MBE-RSF model.}\label{fig:length-602-contour}
\end{figure}
\begin{table}[htbp]
	\centering
    \setlength{\tabcolsep}{6mm}
	\caption{Parameters of Fig. \ref{fig:length-602-contour}. ($\lambda_1=1$, $\lambda_2=3.5$, $\epsilon=1$, $\sigma=3$, $\tau = 0.01$)}
	\begin{tabular}{rccccccc}\hline
		 & $\mu$ &$\alpha$ & $\nu$ &$\tau$ & iter \\\hline
		Fig. \ref{fig:length-602-0}   & 1     & 10     & 0     & 0.01  & 600 \\

		Fig. \ref{fig:length-602-1}   & 1     &30 & 100 & 0.01 & 600\\

		Fig. \ref{fig:length-602-2}   & 1     & 200    & 0   & 0.01  & 600 \\

		Fig. \ref{fig:length-602-5}  & 1     & 10   & 1000     & 0.01  & 600 \\

		Fig. \ref{fig:length-602-4}   & 1     & 100   & 100   & 0.01  & 600 \\

		Fig. \ref{fig:length-602-3} & 1 & 200 & 100 & 0.01 & 600\\\hline
	\end{tabular}
	\label{tab:test-602}%
\end{table}%

In summary, the MBE regularization term can effectively control the smoothness of the contour by adjusting the biharmonic coefficient $\alpha$, while the arc-length term only controls the macroscopic length of the curve with little influence on local smoothness. Based on these findings, we plan to further explore the application of the MBE-RSF model in noise image segmentation.

\subsection{Property of anti-noise}\label{sec:5-3}
Image noise is a common problem that can negatively impact the performance of image segmentation models. In order to evaluate the the robustness of the model with respect to noises, we contaminate the original images with Gaussian noise. We apply MBE and DR2 regularization methods to the RSF model and evaluate their segmentation performance on a same noisy image. Results can be compared to determine which regularization method performs better for the specific noise image segmentation problem at hand.

The images to be segmented are shown in Fig. \ref{fig:length-10-contour} and Fig. \ref{fig:length-101-contour}, with noise levels of $10$ and $15$. The initial level and function selection are the same as in Fig. \ref{fig:length-9-init}. In noisy situations, segmentation of non-uniform intensity images is more difficult than that of uniform intensity images.
Both models require two fitting function $f_1$ and $f_2$ to approximate the image intensities, and some of the noisy speckle would be enlarged during the fitting procedure due to the influence of noise and inhomogeneous background, as shown in Fig. \ref{fig:inhomo-noisy}. The above analysis shows that noise has a great influence on the segmentation results of RSF fitting model, so it is necessary to add a suitable regular term.
\begin{figure}[htbp]
	\centerline{
			\subfigcapskip=-2pt
		\subfigure[]{\includegraphics[width=3cm,height=3cm]{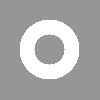}\label{fig:homo}}\quad
\subfigure[]{\includegraphics[width=3cm,height=3cm]{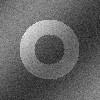}\label{fig:noisy}}\quad
 \subfigure[]{\includegraphics[width=3cm,height=3cm]{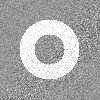}\label{fig:fitting}}
	}
	\caption{(a)Original homogeneous image; (b)Noisy inhomogeneous image ($std = 10$); (c)Ideally fitted image.}\label{fig:inhomo-noisy}
\end{figure}
\begin{figure}[htbp]
	\centerline{
						\subfigcapskip=-2pt
		\subfigure[]{\includegraphics[height=0.22\textwidth, width=0.22\textwidth]{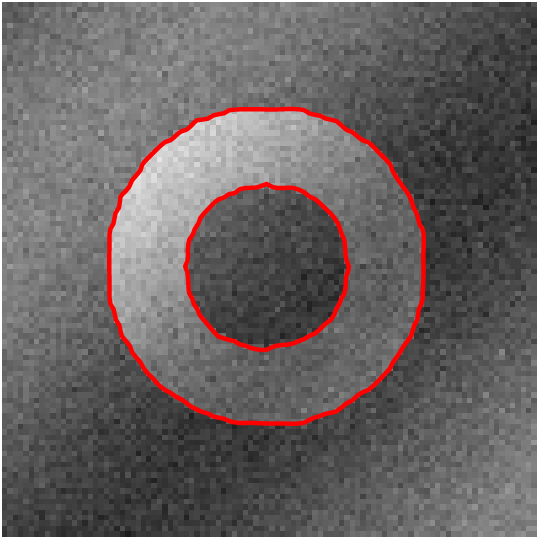}\label{fig:length-10-24}}\quad 
		\subfigure[]{\includegraphics[height=0.22\textwidth, width=0.22\textwidth]{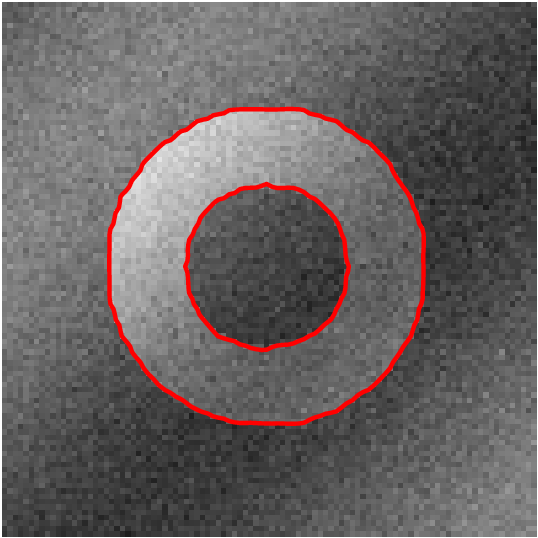}\label{fig:length-10-33}}\quad 
		\subfigure[]{\includegraphics[height=0.22\textwidth, width=0.22\textwidth]{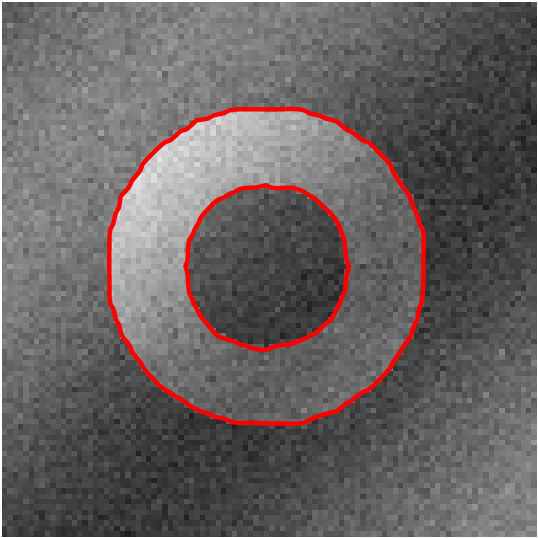}\label{fig:length-10-26}} 
	}\vspace{0.15cm}
\centerline{
					\subfigcapskip=-2pt
		\subfigure[]{\includegraphics[height=0.22\textwidth, width=0.22\textwidth]{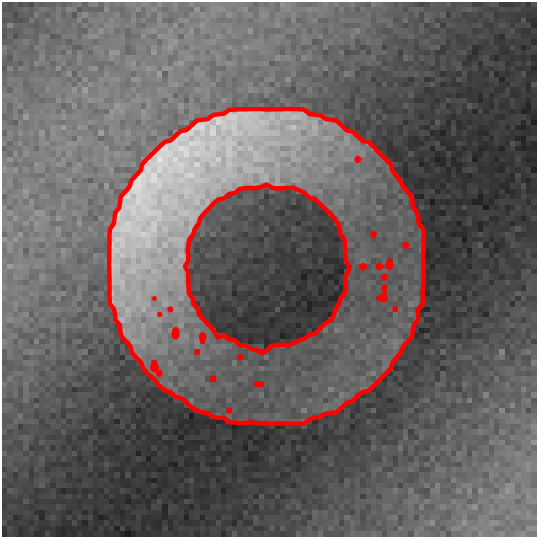}\label{fig:length-10-DR-8}}\quad
		\subfigure[]{\includegraphics[height=0.22\textwidth,  width=0.22\textwidth]{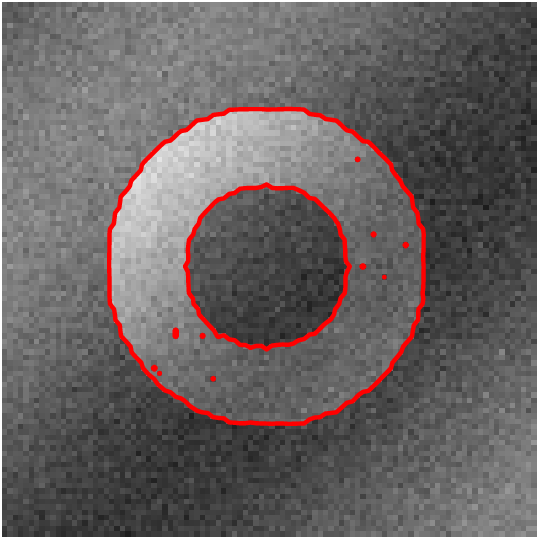}\label{fig:length-10-DR-10}}\quad 
		\subfigure[]{\includegraphics[height=0.22\textwidth,  width=0.22\textwidth]{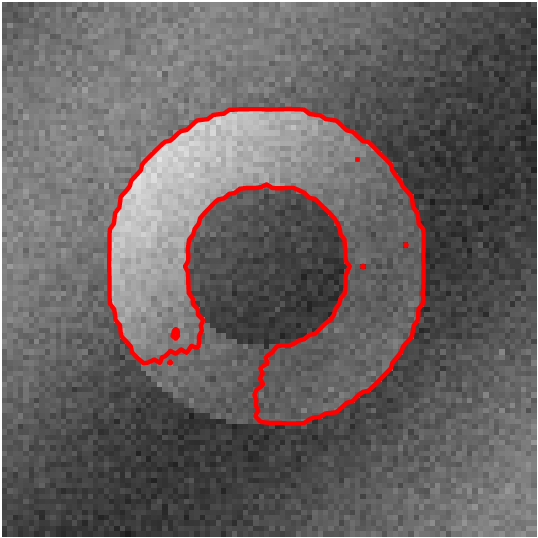}\label{fig:length-10-DR-12}}
	}
	\caption{Segmentation results of MBE-RSF model (a)-(c) and DR2-RSF model (d)-(f) for the noisy image ($std = 10$).}\label{fig:length-10-contour}
\end{figure}

\begin{figure}[htbp]
\centerline{
						\subfigcapskip=-2pt
	\subfigure[Figure \ref{fig:length-10-24}]{\includegraphics[height=0.22\textwidth, width=0.24\textwidth]{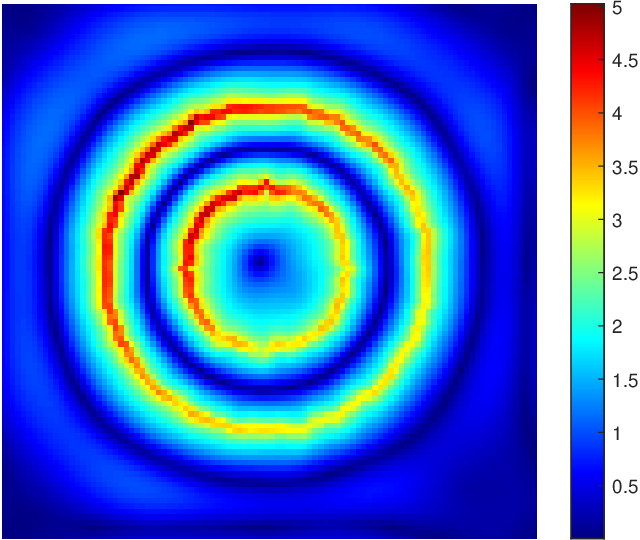}}\quad 
	\subfigure[Figure \ref{fig:length-10-33}]{\includegraphics[height=0.22\textwidth, width=0.24\textwidth]{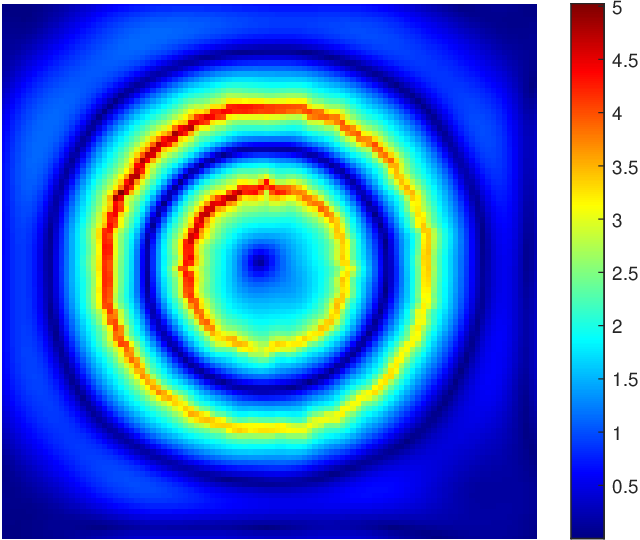}}\quad 
	\subfigure[Figure \ref{fig:length-10-26}]{\includegraphics[height=0.22\textwidth,  width=0.24\textwidth]{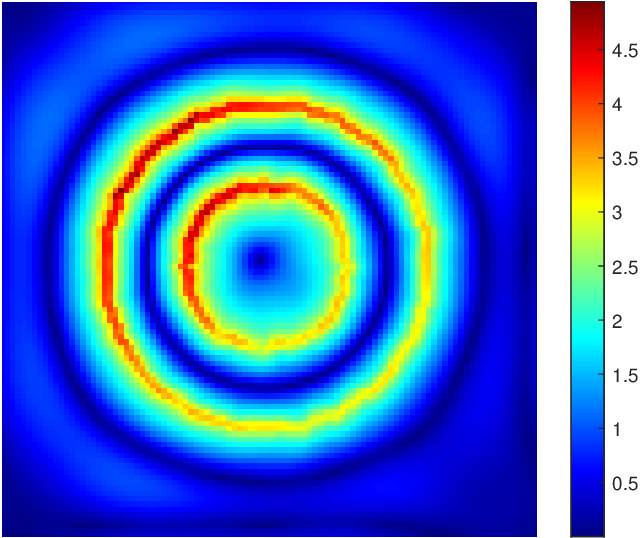}}
	}\vspace{0.1cm}
\centerline{
					\subfigcapskip=-2pt
		\subfigure[Figure \ref{fig:length-10-DR-8}]{\includegraphics[height=0.22\textwidth, width=0.24\textwidth]{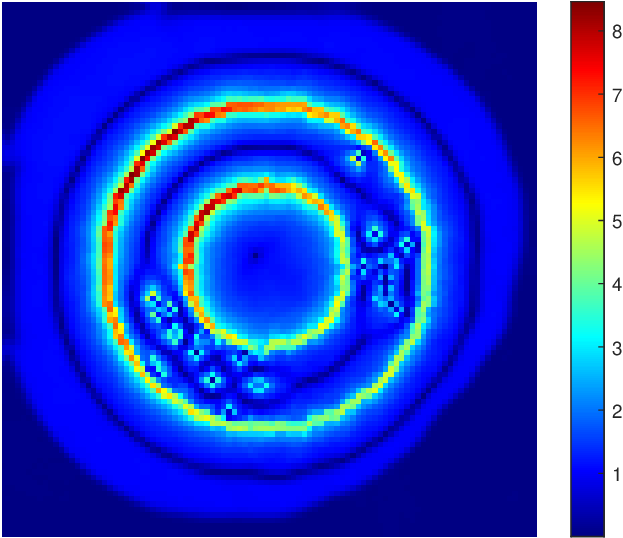}}\quad 
		\subfigure[Figure \ref{fig:length-10-DR-10}]{\includegraphics[height=0.22\textwidth, width=0.24\textwidth]{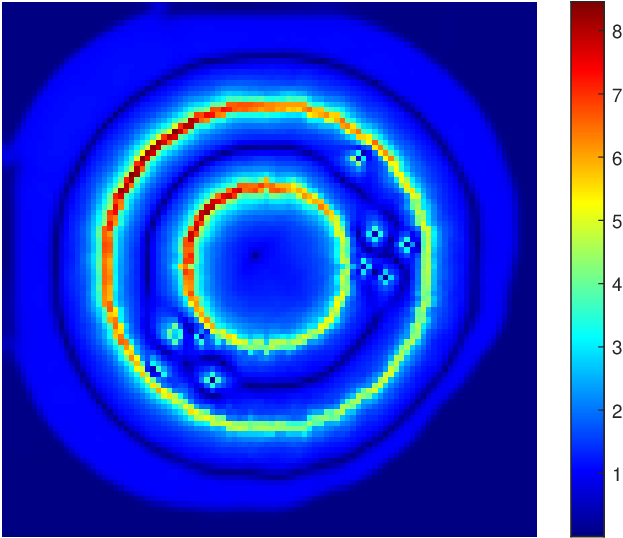}}\quad 
		\subfigure[Figure \ref{fig:length-10-DR-12}]{\includegraphics[height=0.22\textwidth, width=0.24\textwidth]{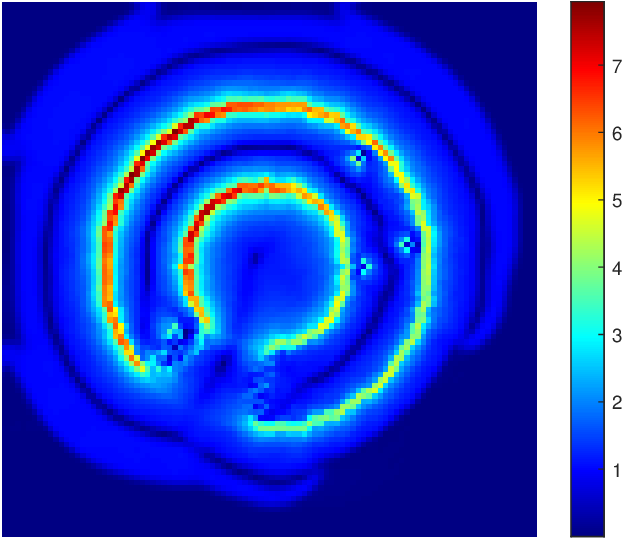}}
}
\caption{The map of  $|\nabla \phi|$ corresponding to Figure \ref{fig:length-10-contour}.}\label{fig:length-10-nablau}
\end{figure}

\begin{table}[htbp]
  \centering
  \setlength{\tabcolsep}{6mm}
  \caption{Parameters of Fig. \ref{fig:length-10-contour}. ($\lambda_1=0.33$, $\lambda_2=0.67$, $\epsilon=1$, $\sigma=5$) }
    \begin{tabular}{cccccccccc}\hline
    std\_n = 10  & $\mu$ & $\alpha$ & $\nu$ & $\tau$ & iter \\\hline
   Fig. \ref{fig:length-10-24}   & 1     & 15    & 10 & 0.01   & 4000 \\
    Fig. \ref{fig:length-10-33}   & 1     & 15    & 20   & 0.01 & 4000 \\
  Fig. \ref{fig:length-10-26}     & 1     & 20    & 0    & 0.01 & 4000 \\
    \hline
    Fig. \ref{fig:length-10-DR-8}     & 6.6  & -  & 20    & 0.015 & 2000 \\
    Fig. \ref{fig:length-10-DR-10}    & 6.6  & -  & 100   & 0.015 & 2000 \\
    Fig. \ref{fig:length-10-DR-12}   & 8   & -   & 150   & 0.013 & 4000 \\\hline
    \end{tabular}%
  \label{tab:test_length_10}%
\end{table}%

The effectiveness of DR2 regularization term in combating noise can be seen in the segmentation results of DR2-RSF and MBE-RSF when the noise level is set to 10, as shown in Fig. \ref{fig:length-10-contour}.  We set $\lambda_1=0.33$, $\lambda_2=0.67$ and $\sigma=5$, the other parameters are shown in Table \ref{tab:test_length_10}. It is found that the DR2 regularization term is ineffective in combating noise for the segmentation results of the DR2-RSF model always contain spots. The MBE-RSF model can perfectly segment the banded ring, thanks to its ability to control local smoothness and anti-noise of the curve using the MBE regularization term.
Moreover, the segmentation results of MBE-RSF are robust to different parameter selection conditions, as demonstrated in Fig. \ref{fig:length-10-24}, Fig. \ref{fig:length-10-33}, Fig. \ref{fig:length-10-26}. Even when the arc length parameter $\nu=0$, the model can achieve good segmentation results, further highlighting the advantages of the MBE regularization term.

\begin{figure}[htbp]
	\centerline{
						\subfigcapskip=-2pt
		\subfigure[]{\includegraphics[height=0.22\textwidth, width=0.22\textwidth]{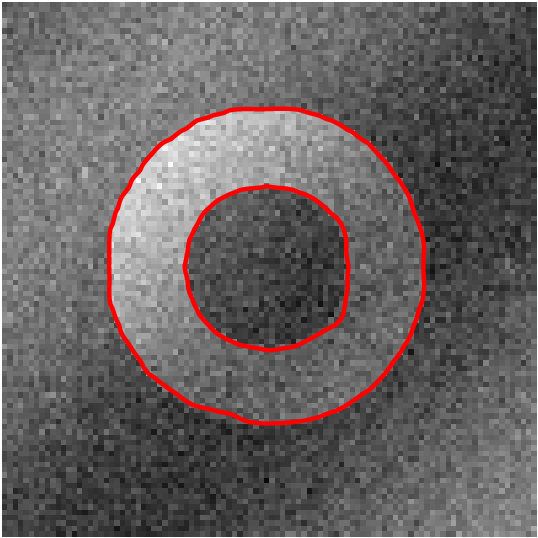}\label{fig:length-101-1}}\quad 
		\subfigure[]{\includegraphics[height=0.22\textwidth, width=0.22\textwidth]{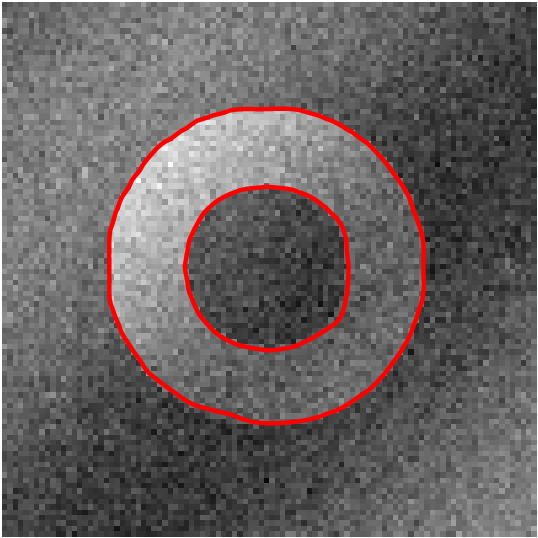}\label{fig:length-101-2}}\quad 
		\subfigure[]{\includegraphics[height=0.22\textwidth, width=0.22\textwidth]{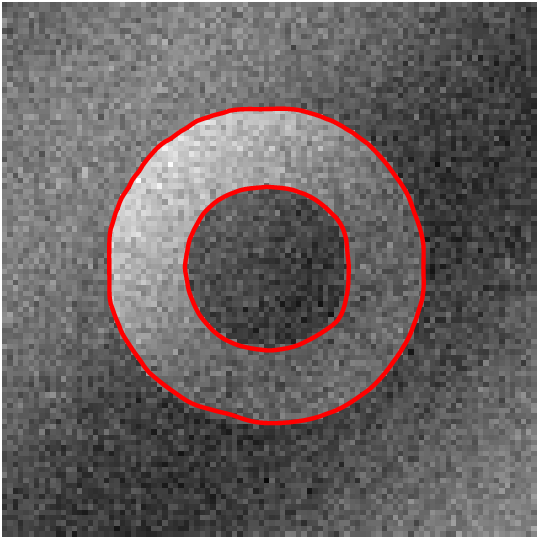}\label{fig:length-101-3}}
	}\vspace{0.15cm}
\centerline{
					\subfigcapskip=-2pt
		\subfigure[]{\includegraphics[height=0.22\textwidth, width=0.22\textwidth]{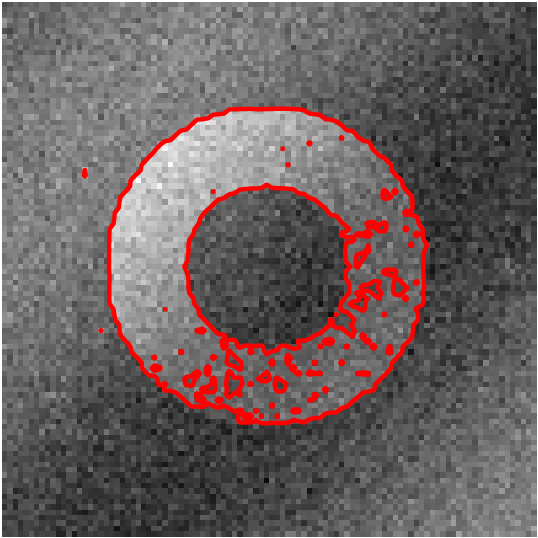}\label{fig:length-101-DR-2}}\quad
		\subfigure[]{\includegraphics[height=0.22\textwidth,  width=0.22\textwidth]{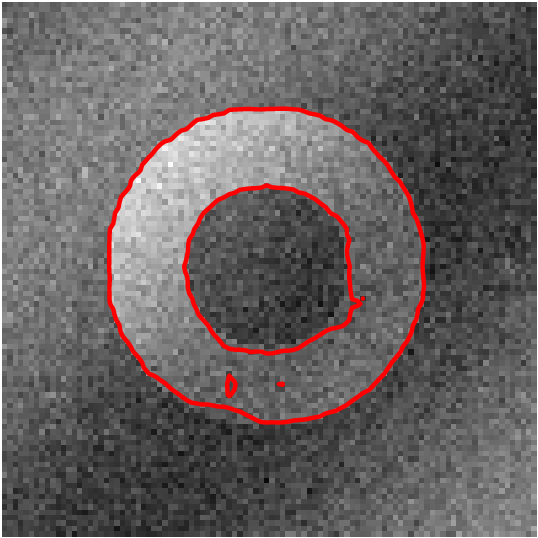}\label{fig:length-101-DR-4}}\quad 
		\subfigure[]{\includegraphics[height=0.22\textwidth,  width=0.22\textwidth]{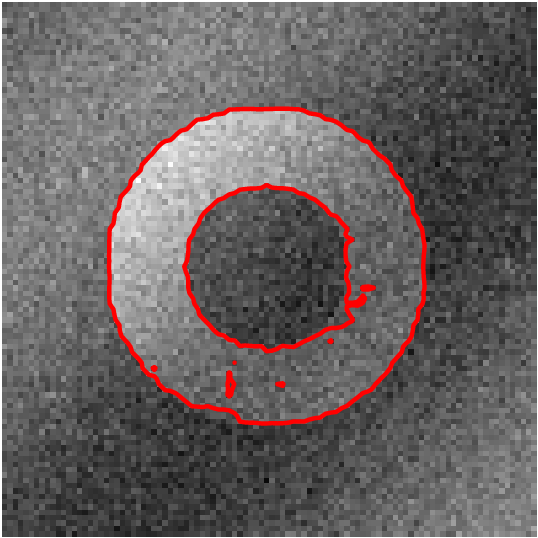}\label{fig:length-101-DR-5}}
	}
	\caption{Segmentation results of MBE-RSF model (a)-(c) and DR2-RSF model (d)-(f) for the noisy image ($std = 15$).}\label{fig:length-101-contour}
\end{figure}

The segmentation results of the DR2-RSF and MBE-RSF models with a noise level of 15 and parameters set to $\lambda_1=0.4$, $\lambda_2=0.6$, and $\sigma=6$ are shown in Fig. \ref{fig:length-101-contour}, with the other parameters selected as listed in Table \ref{tab:addlabel}. Similar to the case with a noise level of 10, the DR2-RSF model fails to produce smooth segmentation results. As the noise level increases, more artifacts appear at the segmentation boundaries in the DR2-RSF model, indicating that noise has a greater negative impact on its segmentation performance. In contrast, in the absence of arc length constraint, the MBE-RSF model adjusts the coefficients between $\mu$ and $\alpha$ to smoothly separate the ring region from the background.

\begin{figure}[htbp]
\centerline{
	\subfigure[Fig. \ref{fig:length-101-1}]{\includegraphics[height=0.22\textwidth, width=0.24\textwidth]{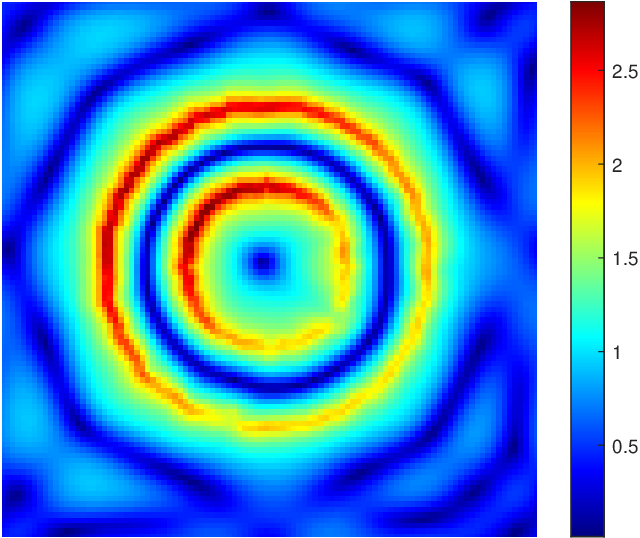}}\quad 
	\subfigure[Fig. \ref{fig:length-101-2}]{\includegraphics[height=0.22\textwidth, width=0.24\textwidth]{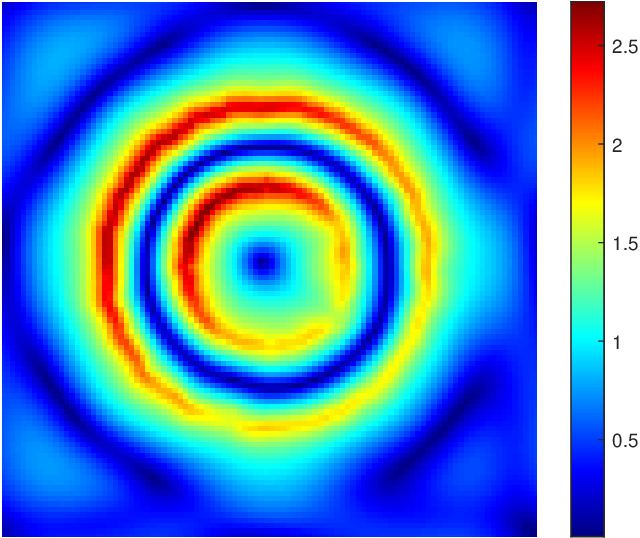}}\quad 
	\subfigure[Fig. \ref{fig:length-101-3}]{\includegraphics[height=0.22\textwidth, width=0.24\textwidth]{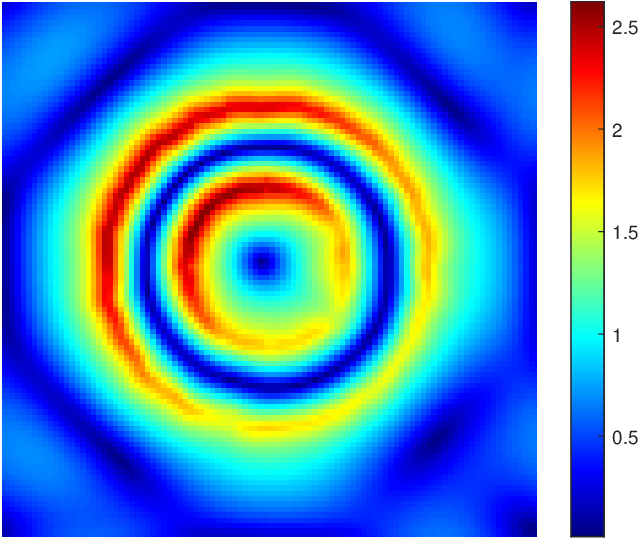}}
\vspace{0.1cm}
	}\centerline{
	\subfigcapskip=-2pt
		\subfigure[Fig. \ref{fig:length-101-DR-2}]{\includegraphics[height=0.22\textwidth, width=0.24\textwidth]{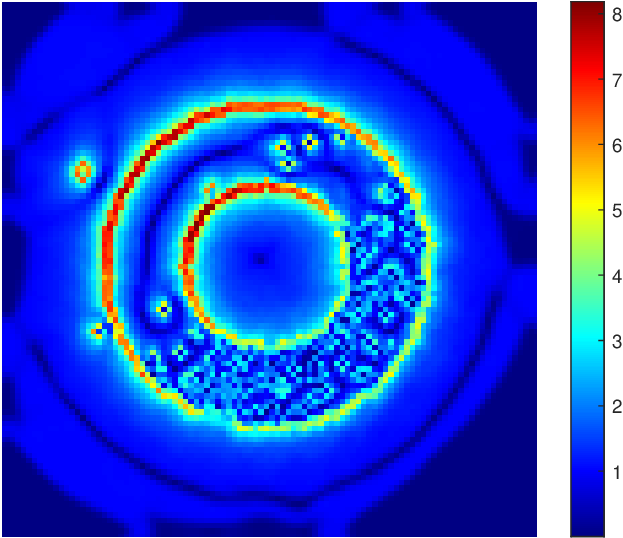}}\quad 
		\subfigure[Fig. \ref{fig:length-101-DR-4}]{\includegraphics[height=0.22\textwidth, width=0.24\textwidth]{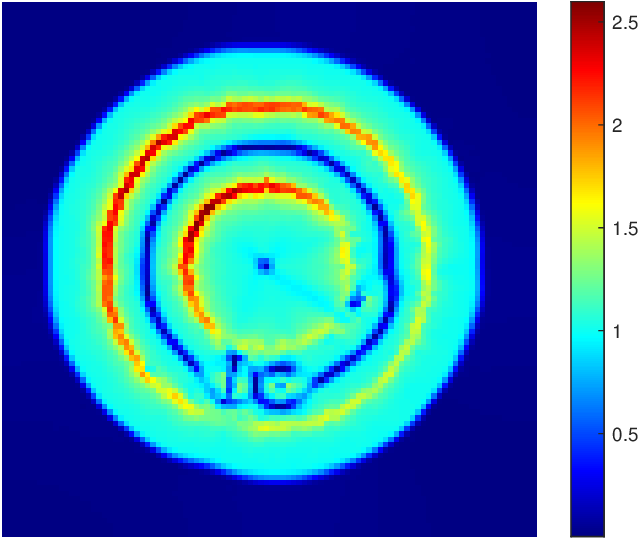}}\quad 
		\subfigure[Fig. \ref{fig:length-101-DR-5}]{\includegraphics[height=0.22\textwidth, width=0.24\textwidth]{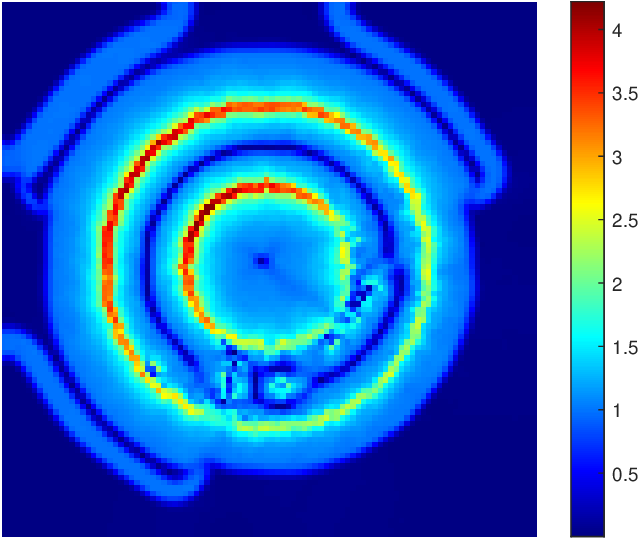}}
}
\caption{The map of  $|\nabla \phi|$ corresponding to Fig. \ref{fig:length-101-contour}.}\label{fig:length-101-nablau}
\end{figure}
The segmentation results of the DR2-RSF and MBE-RSF models with a noise level of 15 and parameters set to $\lambda_1=0.4$, $\lambda_2=0.6$, and $\sigma=6$ are shown in Fig. \ref{fig:length-101-contour}, with the other parameters selected as listed in Table \ref{tab:addlabel}. Similar to the case with a noise level of 10, the DR2-RSF model fails to produce smooth segmentation results. As the noise level increases, more artifacts appear at the segmentation boundaries in the DR2-RSF model, indicating that noise has a greater negative impact on its segmentation performance. In contrast, in the absence of arc length constraint, the MBE-RSF model adjusts the coefficients between $\mu$ and $\alpha$ to smoothly separate the ring region from the background.

\begin{table}[htbp]
	\centering
    \setlength{\tabcolsep}{6mm}
	\caption{Parameters of Fig. \ref{fig:length-101-contour}.  ($\lambda_1=0.4$, $\lambda_2=0.6$, $\epsilon=1$, $\sigma=6$)}
	\begin{tabular}{cccccccccc}\hline
		std\_n = 15  & $\mu$ & $\delta$ & $\nu$ &$\tau$ & iter \\\hline
		Fig. \ref{fig:length-101-1}   & 10    & 10    & 0     & 0.001 & 20000 \\
		Fig. \ref{fig:length-101-2}     & 10    & 20    & 0     & 0.001 & 20000 \\
		Fig. \ref{fig:length-101-3}   & 10    & 30    & 0     & 0.001 & 20000 \\
		Fig. \ref{fig:length-101-DR-2}       & 7     &    -   & 1     & 0.012 & 8000 \\
		Fig.\ref{fig:length-101-DR-4}      & 200   &   -    & 1     & 0.0005 & 30000 \\
		Fig. \ref{fig:length-101-DR-5}     & 50    &   -    & 100   & 0.002 & 10000 \\\hline
	\end{tabular}%
	\label{tab:addlabel}%
\end{table}%

Furthermore, both methods require more iterations to achieve a stable state due to the increased noise level. The experimental results in Fig. \ref{fig:length-101-contour} further highlight the advantages of the MBE regularization method in combating noise.
\subsection{MBE-RSF model for real image segmentation}\label{sec:5-4}
Real images often contain complex structures, various backgrounds, and different lighting conditions, which pose significant challenges to image segmentation. Numerical experiments on real images can also be time-consuming. Therefore, in this section, we use both the semi-implicit semi-explicit finite-difference method(FDM) and the SAV scheme to discretize the images. Our experimental results demonstrate that the MBE-RSF model produces highly accurate segmentation results, even on complex real-world images.

The MBE-RSF model has shown promising results in real image segmentation.
In natural image segmentation, as shown in Fig. \ref{fig:bird} and Fig. \ref{fig:bear}, the MBE-RSF model effectively segments objects in complex scenes with varying backgrounds and lighting conditions. The model also has demonstrated its ability to segment objects in challenging scenarios with occlusions and overlapping objects.

We have also applied the MBE-RSF model to medical image segmentation and compared it with the DR1/DR2-RSF models. To ensure the fairness of the comparison, both models are fine-tuned to achieve optimal segmentation results. Our experimental results show that the MBE-RSF model outperforms the DR1/DR2-RSF models in segmenting low-dose CT images, X-ray images, microscopic cell images, and MRI images.

\begin{figure}[htbp]
	\centerline{
		\subfigcapskip=-2pt
		\subfigure[Original]{\includegraphics[trim={0cm 0cm 0cm 0cm},clip, width=0.23\textwidth,height=4.5cm]{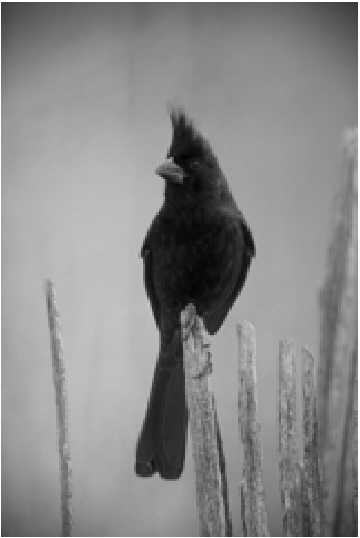}}

        \quad
        \subfigure[{DR1-RSF}]{\includegraphics[trim={0cm 0cm 0cm 0cm},clip, width=0.23\textwidth,height=4.5cm]{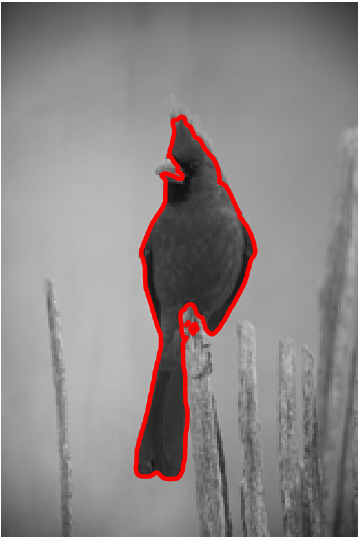}\label{fig:birddr1}}
        \quad
		\subfigure[{DR2-RSF}]{\includegraphics[trim={0cm 0cm 0cm 0cm},clip, width=0.23\textwidth,height=4.5cm]{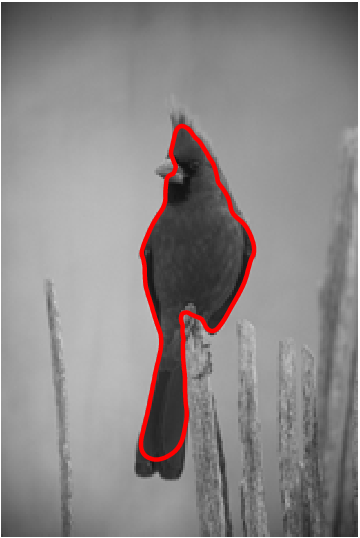}\label{fig:birddr2}}
        \vspace{0.1cm}
	}\centerline{
	\subfigcapskip=-2pt
   \subfigure[Init-contour]{\includegraphics[trim={0cm 0cm 0cm 0cm},clip, width=0.23\textwidth,height=4.5cm]{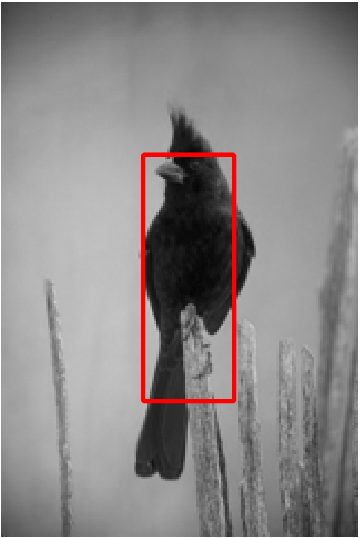}}
   \quad
       \subfigure[{MBE-RSF-FDM}]{\includegraphics[trim={0cm 0cm 0cm 0cm},clip, width=0.23\textwidth,height=4.5cm]{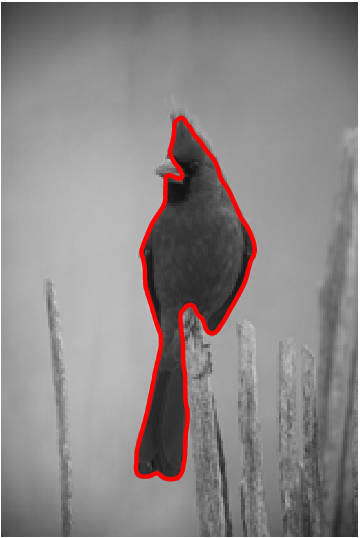}\label{fig:birdmbe}}
       \quad
        
           \subfigure[{MBE-RSF-SAV}]{\includegraphics[trim={0cm 0cm 0cm 0cm},clip, width=0.23\textwidth,height=4.5cm]{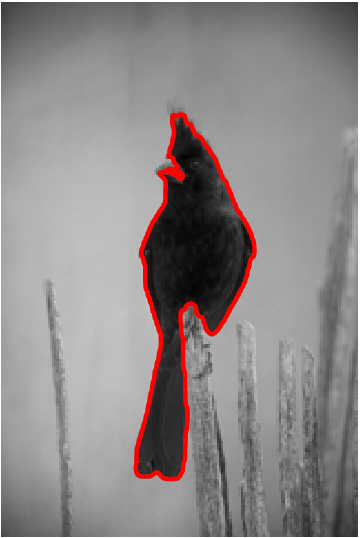}\label{fig:birdsav}}
         
}
	\caption{Segmentation results of DR1-RSF model, DR2-RSF model and MBE-RSF model.}\label{fig:bird}
\end{figure}
\begin{figure}[htbp]
	\centerline{
		\subfigcapskip=-2pt
		\subfigure[Original]{\includegraphics[trim={0cm 0cm 0cm 0cm},clip, width=0.24\textwidth,height=2.5cm]{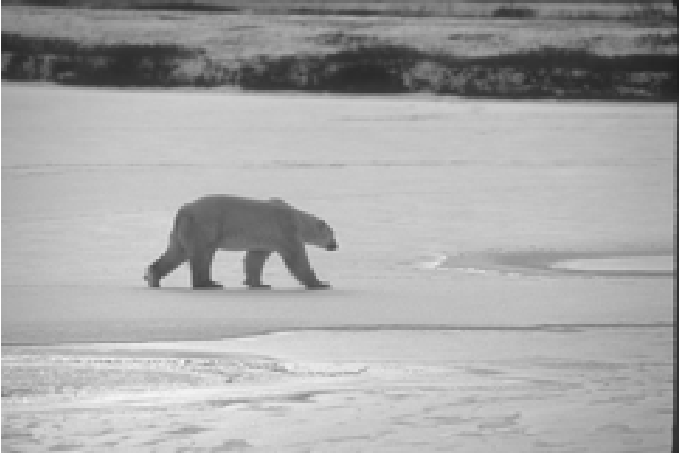}}
        \quad

		\subfigure[DR1-RSF]{\includegraphics[trim={0cm 0cm 0cm 0cm},clip, width=0.24\textwidth,height=2.5cm]{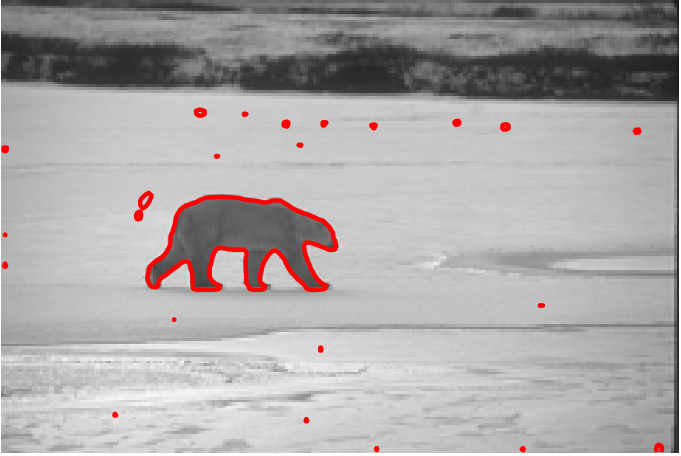}\label{fig:beardr1}}
		
        \quad
        \subfigure[DR2-RSF]{\includegraphics[trim={0cm 0cm 0cm 0cm},clip, width=0.24\textwidth,height=2.5cm]{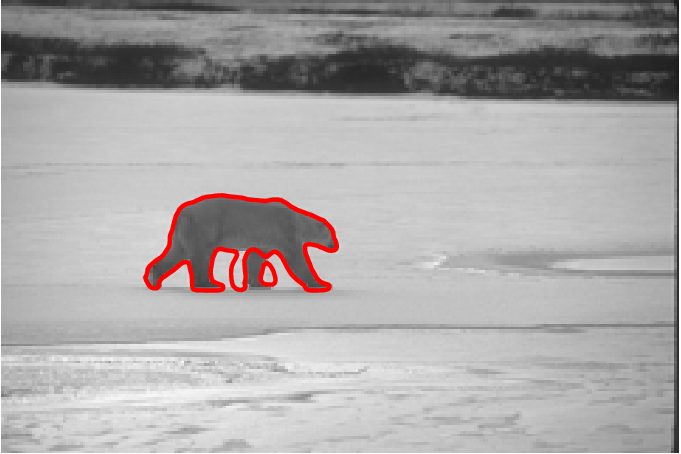}\label{fig:beardr2}}
     
        \vspace{0.1cm}
        }\centerline{
	\subfigcapskip=-2pt
\subfigure[Init-contour]{\includegraphics[trim={0cm 0cm 0cm 0cm},clip, width=0.24\textwidth,height=2.5cm]{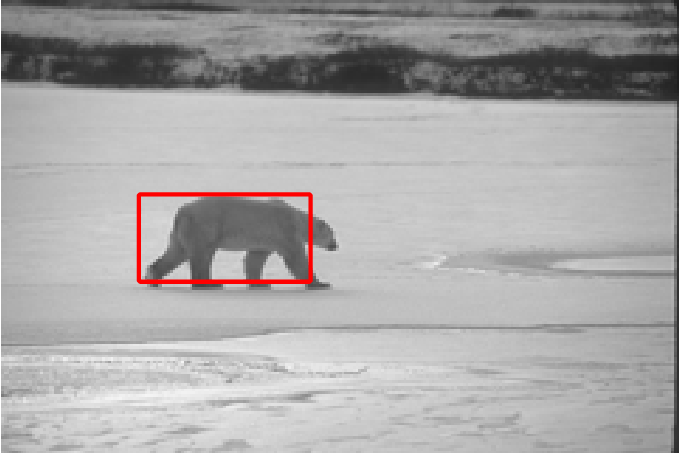}}
 \quad		
\subfigure[MBE-RSF-FDM]{\includegraphics[trim={0cm 0cm 0cm 0cm},clip, width=0.24\textwidth,height=2.5cm]{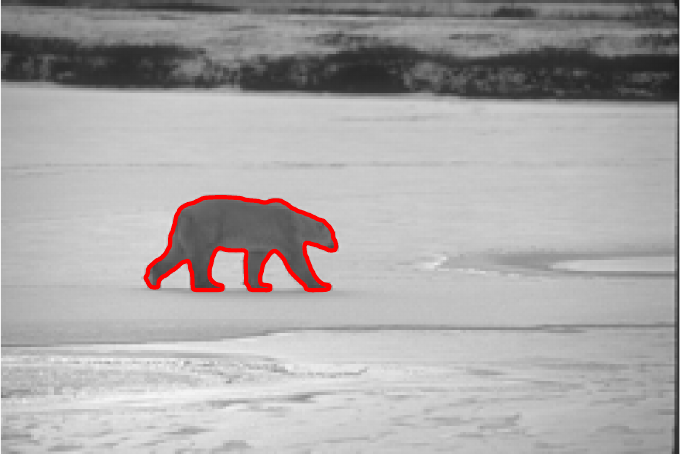}\label{fig:bearmbe}}
		
 \quad		
\subfigure[MBE-RSF-SAV]{\includegraphics[trim={0cm 0cm 0cm 0cm},clip, width=0.24\textwidth,height=2.5cm]{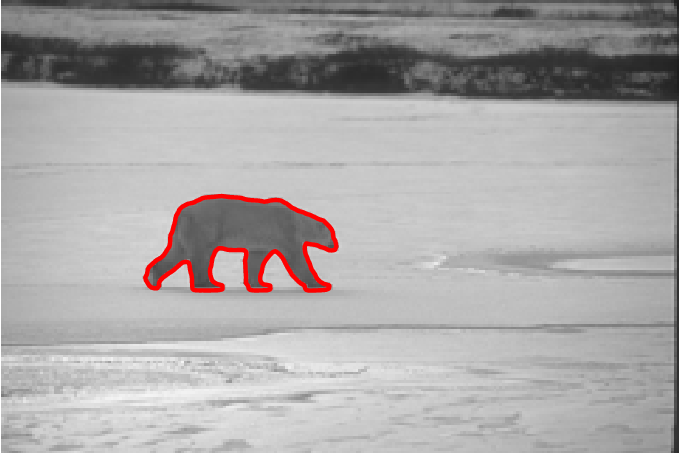}\label{fig:bearsav}}
        
	}
	\caption{Segmentation results of DR1-RSF model, DR2-RSF model and MBE-RSF model.}\label{fig:bear}
\end{figure}

\begin{figure}[htbp]
	\centerline{
				\subfigcapskip=-2pt
\subfigure[Original]{\includegraphics[width=0.23\textwidth,height=1.7cm]{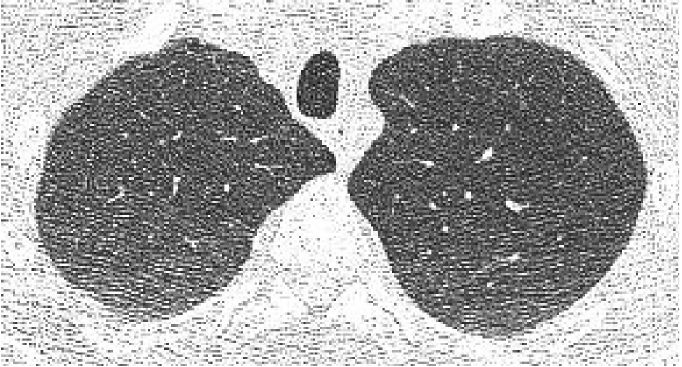}}\quad

	    \subfigure[DR1-RSF]{\includegraphics[width=0.23\textwidth,height=1.7cm]{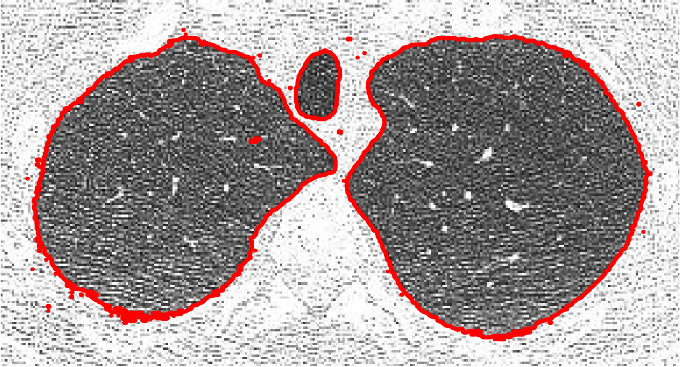}\label{fig:ctdr1}}\quad
        \subfigure[DR2-RSF]{\includegraphics[width=0.23\textwidth,height=1.7cm]{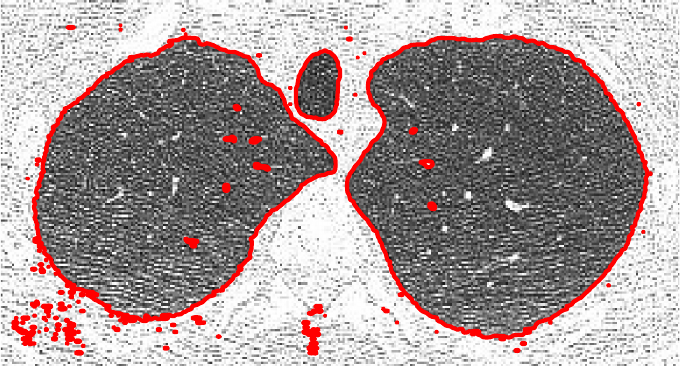}\label{fig:ctdr2}}
\vspace{0.1cm}
	}\centerline{
	\subfigcapskip=-2pt
  \subfigure[Init-contour]{\includegraphics[ width=0.23\textwidth,height=1.7cm]{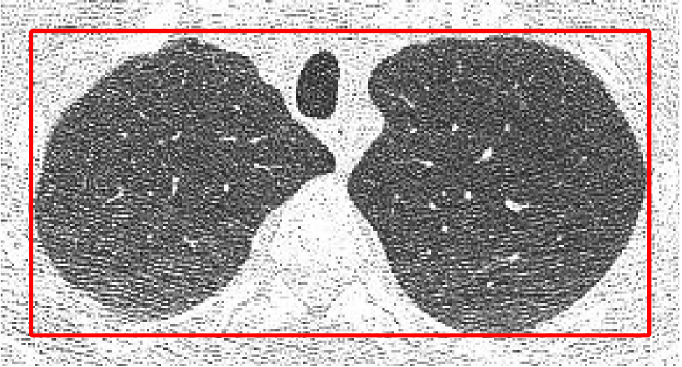}}\quad
		\subfigure[MBE-RSF-FDM]{\includegraphics[width=0.23\textwidth,height=1.7cm]{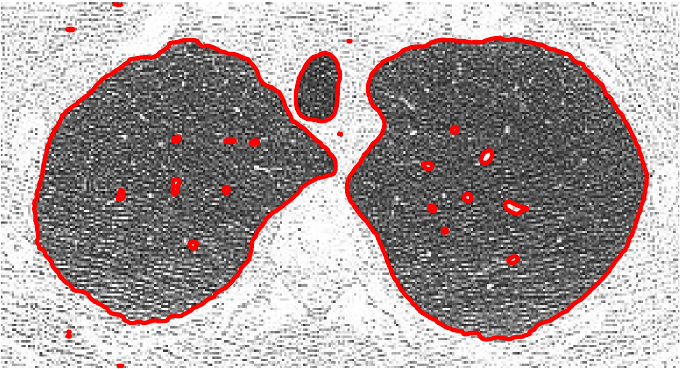}\label{fig:ctmbe}}\quad
        \subfigure[MBE-RSF-SAV]{\includegraphics[width=0.23\textwidth,height=1.7cm]{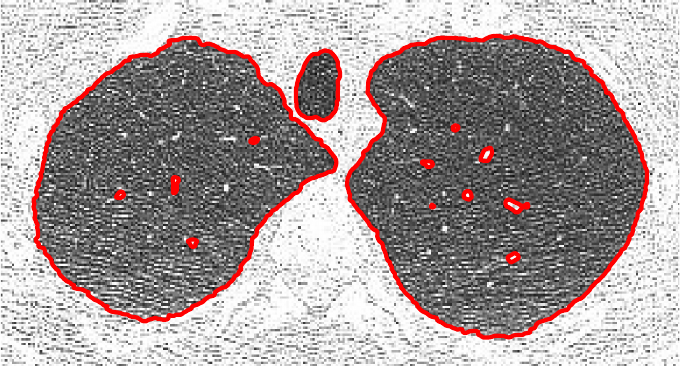}\label{fig:ctsav}}
        }
	\caption{Segmentation results of DR1-RSF model, DR2-RSF model and MBE-RSF model.}\label{fig:CT}
\end{figure}

 Fig. \ref{fig:CT} shows the segmentation results of low-dose lung CT images. Due to the poor imaging quality of the image, there are a lot of artifacts and noise in the image, which will increase the difficulty of segmentation. Neither DR1 nor DR2 model can represent the smooth curve of lung, which indicates that DR1 and DR2 regularization terms cannot overcome the influence of artifacts and noise. Due to the function of MBE regularization term, the RSF-MBE model can obtain smooth segmentation curve and successfully separate the alveoli. Consistent with the conclusions of the previous section, this experiment demonstrates the advantage of MBE regularization terms in controlling the smoothness of curves and overcoming the influence of noise. Due to the function of MBE regularization term, the model can control the smoothness of curves without losing small targets.

 \begin{figure}[htbp]
	\centerline{
		\subfigcapskip=-2pt
\subfigure[Original]{\includegraphics[width=0.23\textwidth,height=2.5cm]{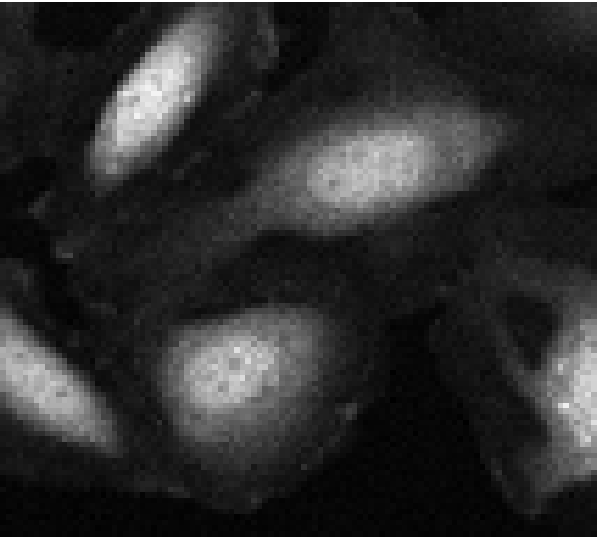}}\quad

	    \subfigure[DR1-RSF]{\includegraphics[width=0.23\textwidth,height=2.5cm]{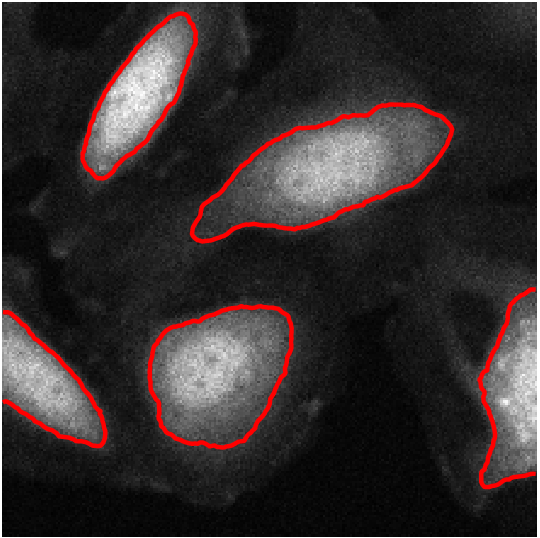}\label{fig:celldr1}}\quad
        \subfigure[DR2-RSF]{\includegraphics[width=0.23\textwidth,height=2.5cm]{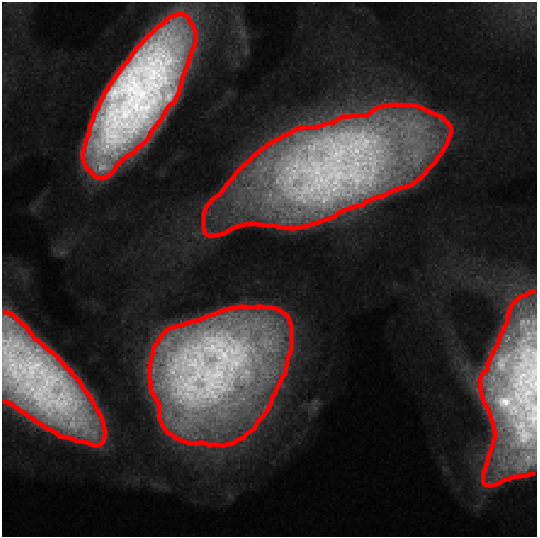}\label{fig:celldr2}}
\vspace{0.1cm}
	}\centerline{
	\subfigcapskip=-2pt
 \subfigure[Init-contour]{\includegraphics[ width=0.23\textwidth,height=2.5cm]{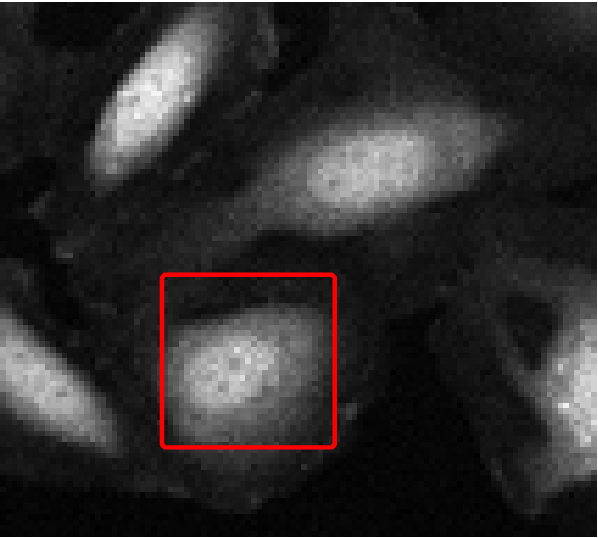}}\quad
		\subfigure[MBE-RSF-FDM]{\includegraphics[width=0.23\textwidth,height=2.5cm]{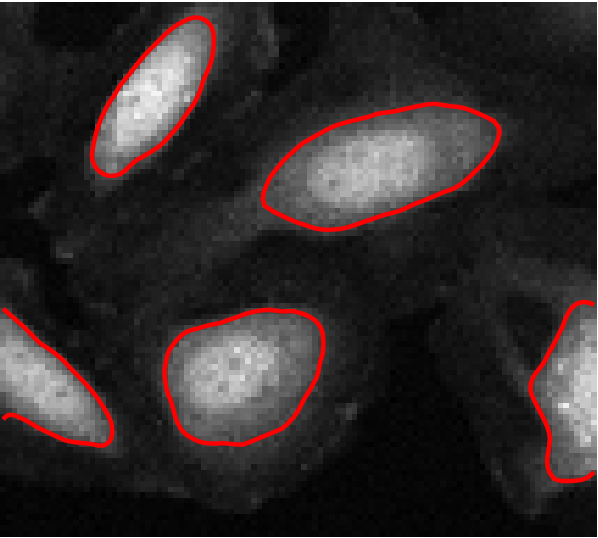}\label{fig:cellmbe}}\quad
        \subfigure[MBE-RSF-SAV]{\includegraphics[width=0.23\textwidth,height=2.5cm]{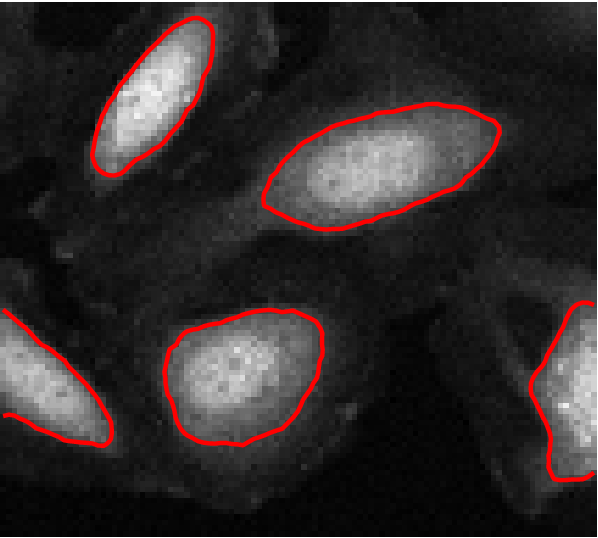}\label{fig:cellsav}}
	}
	\caption{Segmentation results of DR1-RSF model, DR2-RSF model and MBE-RSF model.}\label{fig:CELL}
\end{figure}

\begin{figure}[htbp]
	\centerline{
		\subfigcapskip=-2pt
		\subfigure[Original]{\includegraphics[trim={0cm 0cm 0cm 0cm},clip, width=0.23\textwidth,height=2cm]{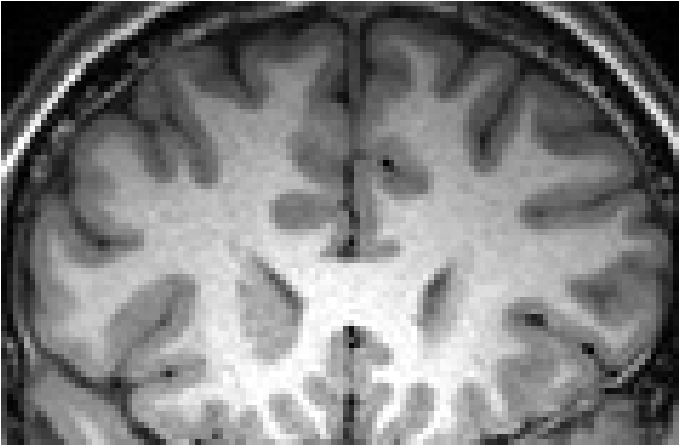}}\quad


		\subfigure[DR1-RSF]{\includegraphics[trim={0cm 0cm 0cm 0cm},clip, width=0.23\textwidth,height=2cm]{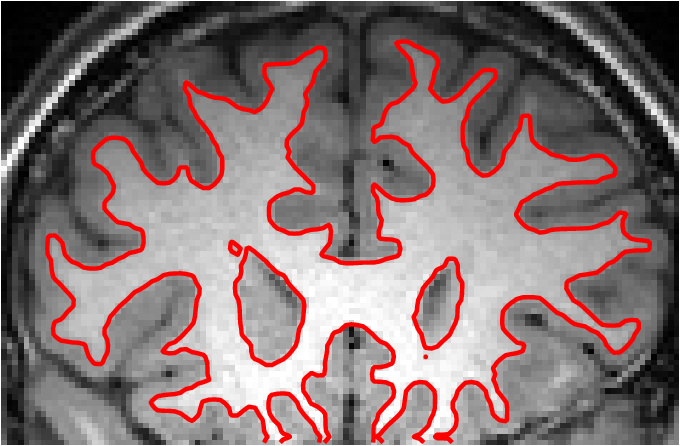}\label{fig:mridr1}}\quad
		\subfigure[DR2-RSF]{\includegraphics[trim={0cm 0cm 0cm 0cm},clip, width=0.23\textwidth,height=2cm]{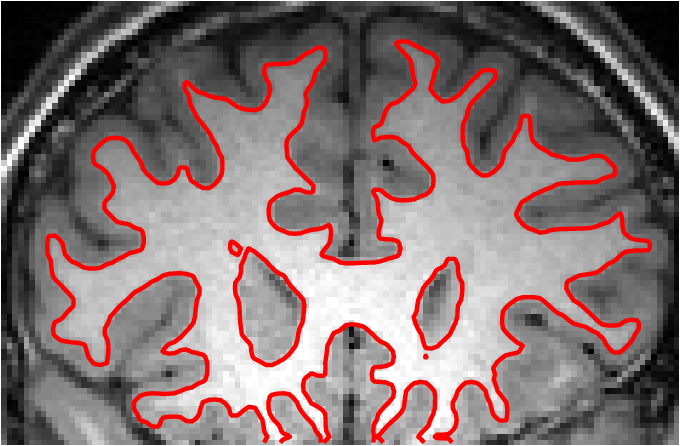}\label{fig:mridr2}}
        \vspace{0.1cm}
	}\centerline{
	\subfigcapskip=-2pt
\subfigure[Init-contour]{\includegraphics[trim={0cm 0cm 0cm 0cm},clip, width=0.23\textwidth,height=2cm]{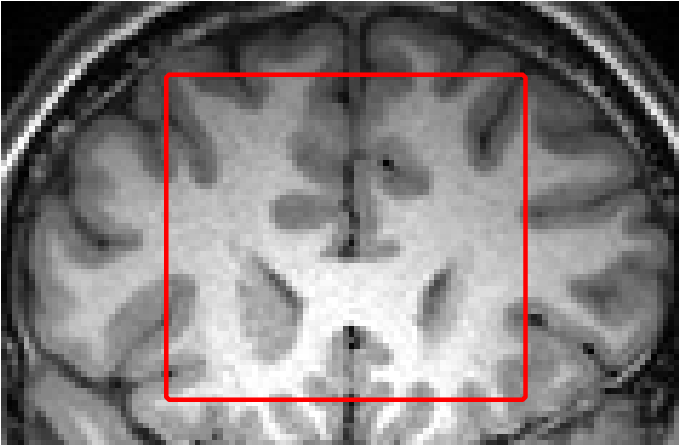}}\quad
		\subfigure[MBE-RSF-FDM]{\includegraphics[trim={0cm 0cm 0cm 0cm},clip, width=0.23\textwidth,height=2cm]{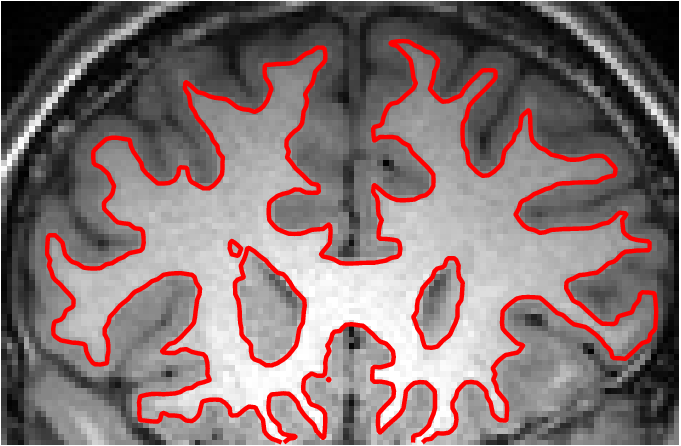}\label{fig:mrimbe}}\quad
		\subfigure[MBE-RSF-SAV]{\includegraphics[trim={0cm 0cm 0cm 0cm},clip, width=0.23\textwidth,height=2cm]{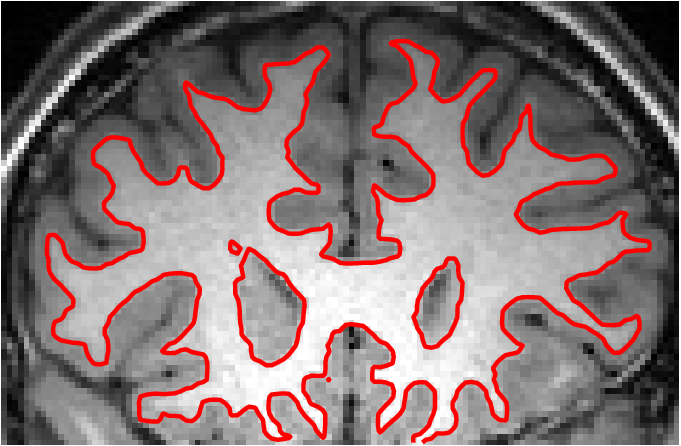}\label{fig:mrisav}}
	}
  \caption{Segmentation results of DR1-RSF model, DR2-RSF model and MBE-RSF model.}\label{fig:MRI}
 \end{figure}

 Fig. \ref{fig:CELL} and Fig. \ref{fig:MRI} show the results of microcellular images and brain MRI images, respectively. Both proposed model and compared models can achieve excellent segmentation results, which verifies the effectiveness of the model. Meanwhile, for microscopic images with blurred edges and brain MRI with rugged edges, the MBE regularization term can achieve fine segmentation while controlling the local smoothness of the curve.

  \begin{figure}[htbp]
	\centerline{
		\subfigcapskip=-2pt
		\subfigure[Original]{\includegraphics[trim={1cm 0cm 0cm 0cm},clip, width=0.23\textwidth,height=4.5 cm]{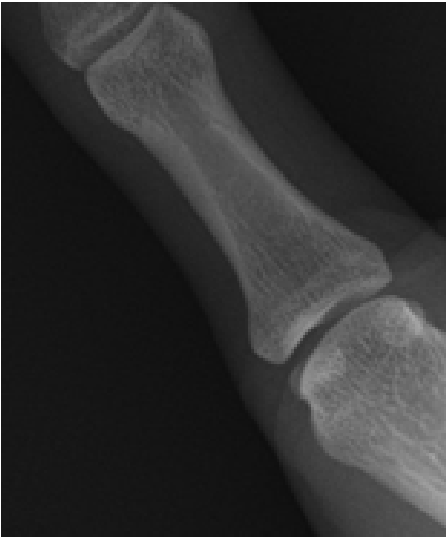}}\quad

        \subfigure[DR1-RSF]{\includegraphics[trim={1cm 0cm 0cm 0cm},clip, width=0.23\textwidth,height=4.5cm]{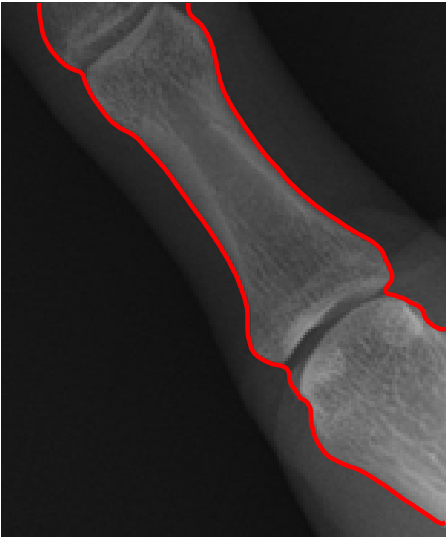}\label{fig:xraydr1}}\quad
		
		\subfigure[DR2-RSF]{\includegraphics[trim={1cm 0cm 0cm 0cm},clip, width=0.23\textwidth,height=4.5cm]{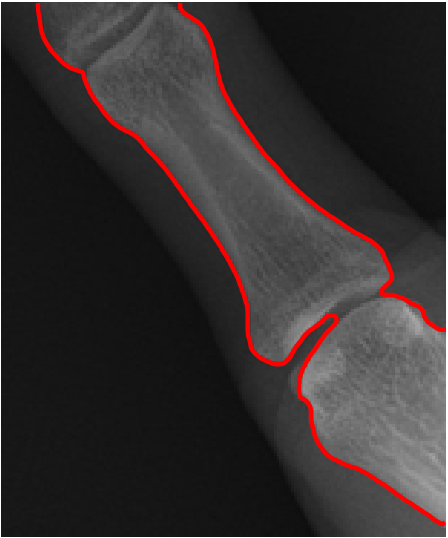}\label{fig:xraydr2}}
        
        \vspace{0.1cm}
	}\centerline{
	\subfigcapskip=-2pt
\subfigure[Init-contour]{\includegraphics[trim={1cm 0cm 0cm 0cm},clip, width=0.23\textwidth,height=4.5cm]{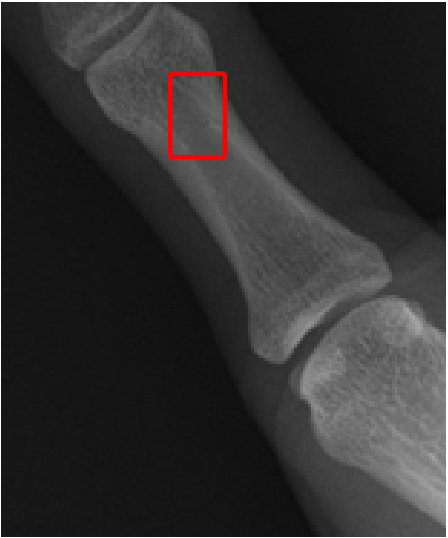}}\quad
       \subfigure[MBE-RSF-FDM]{\includegraphics[trim={1cm 0cm 0cm 0cm},clip, width=0.23\textwidth,height=4.5cm]{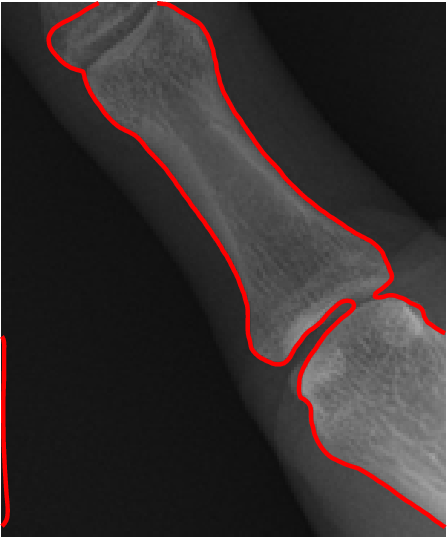}\label{fig:xraymbe}}\quad
        
           \subfigure[MBE-RSF-SAV]{\includegraphics[trim={1cm 0cm 0cm 0cm},clip, width=0.23\textwidth,height=4.5cm]{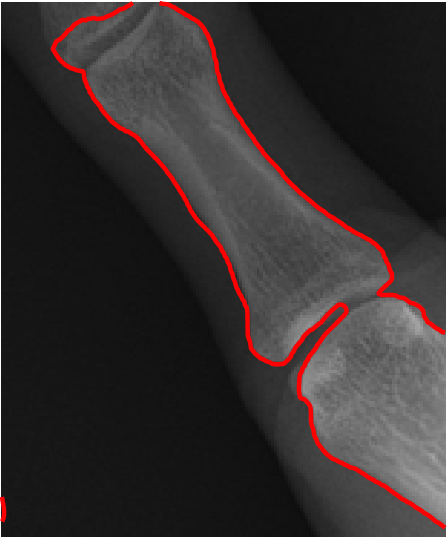}\label{fig:xraysav}}
           
}
	\caption{The segmentation results of DR1-RSF model, DR2-RSF model and MBE-RSF model.}\label{fig:X-ray}
\end{figure}
Fig. \ref{fig:X-ray} compares the segmentation results of the X-ray images joints of the three models. The segmentation curve obtained by the MBE-RSF model is more suitable for the edges in the image, the model can distinguish the joints with weak contrast in the image. This shows that the proposed model can also obtain more precise segmentation results for images with lower contrast.
\section{Conclusions}
\label{sec:conclusions}
In this paper, we have proposed a new high-order variational level set method for image segmentation problems, using  molecular beam epitaxial film manufacturing. This approach has eliminated the need for re-initialization and has enhanced the stability of the evolution. Two segmentation models, MBE-GAC and MBE-RSF, have been presented, demonstrating the flexibility of the MBE regularization term. We have designed the SAV scheme coupled with FFT for the proposed MBE segmentation models, resulting in a significant improvement in efficiency. Numerical experiments have demonstrated that our approach has effectively controlled the local smoothness of the segmentation curve, proving robust against noise. In addition, the model has exhibited remarkable segmentation results in handling fuzzy edges, intensity inhomogeneity, and small and complex objects. Overall, the MBE regularization method has shown superiority and potential in image segmentation, both in terms of accuracy and efficiency.

\bmhead{Acknowledgements}
This work is partially supported by the National Natural Science Foundation of China (12171123, U21B2075, 12371419, 12271130),
the Natural Science Foundation of Heilongjiang Province (ZD2022A001), the Fundamental Research Funds for the Central Universities (HIT.NSRIF202202, 2022FRFK060020,  HIT.NSRIF. 2020081, 2022FRFK060014).
 
  \bibliography{references}

\end{document}